\documentclass[review,11pt]{ReportTemplate}
\usepackage{bm}
\usepackage{graphicx}
\usepackage{paralist,algorithmic,algorithm}
\usepackage{amsmath,amssymb,mathrsfs}
\usepackage{multirow}
\usepackage{comment}
\usepackage{xcolor}

\newtheorem{thm}{Theorem}

\newtheorem{lemma}{Lemma}

\newtheorem{definition}[thm]{Definition}
\newenvironment{proof}{\textbf{Proof:}\ }{\hspace{\stretch{1}}$\square$\\}

\def \R {\mathbb{R}}

\def \E {\mathrm{E}}
\def \x {\mathbf{x}}
\def \a {\mathbf{a}}

\def \v {\mathbf{v}}

\def \X {\mathcal{X}}

\def \Mh {\widehat{M}}

\def \u {\mathbf{u}}

\def \v {\mathbf{v}}
\def \w {\mathbf{w}}

\def \R {\mathbb{R}}

\def \b {\mathbf{b}}

\def \vt {\widetilde{\v}}

\def \Hh {\widehat{H}}

\def \e {\mathbf{e}}

\def \Uh {\widehat{U}}
\def \uh {\widehat{\u}}
\def \Vh {\widehat{V}}
\def \vh {\widehat{\v}}

\def \Rt {\mathcal{R}}

\def \ut {\widetilde{\u}}
\def \muh {\widehat{\mu}}

\def \1 {\mathbf{1}}

\begin{document}
\begin{frontmatter}
\title{CUR Algorithm for Partially Observed Matrices}

\author{Miao Xu$^1$}
\author{Rong Jin$^2$}
\author{Zhi-Hua Zhou$^1$\corref{cor1}}
\address{$^1$National Key Laboratory for Novel Software Technology\\
Nanjing University, Nanjing 210093, China\\
Department of Computer Science and Engineering\\
$^2$Michigan State University, East Lansing, MI, 48824} \cortext[cor1]{\small Corresponding author.
Email: zhouzh@nju.edu.cn}

\begin{abstract}
CUR matrix decomposition computes the low rank approximation of a given matrix by using the actual rows and columns of the matrix. It has been a very useful tool for handling large matrices. One limitation with the existing algorithms for CUR matrix decomposition is that they need an access to the {\it full} matrix, a requirement that can be difficult to fulfill in many real world applications. In this work, we alleviate this limitation by developing a CUR decomposition algorithm for partially observed matrices. In particular, the proposed algorithm computes the low rank approximation of the target matrix based on (i) the randomly sampled rows and columns, and (ii) a subset of observed entries that are randomly sampled from the matrix. Our analysis shows the relative error bound, measured by spectral norm, for the proposed algorithm when the target matrix is of full rank. We also show that only $O(n r\ln r)$ observed entries are needed by the proposed algorithm to perfectly recover a rank $r$ matrix of size $n\times n$, which improves the sample complexity of the existing algorithms for matrix completion. Empirical studies on both synthetic and real-world datasets verify our theoretical claims and demonstrate the effectiveness of the proposed algorithm.
\end{abstract}

\end{frontmatter}

\section{Introduction}
In many machine learning applications, it is convenient to represent data by matrix. Examples include user-item rating matrix in recommender system~\cite{DBLP:conf/nips/SrebroRJ04}, gene expression matrix in bioinformatics~\cite{drineas-2008-relative}, kernel matrix in kernel learning~\cite{DBLP:conf/nips/WilliamsS00}, document-term matrix in document retrieval~\cite{drineas-2008-relative}, and instance-label matrix in multi-label learning~\cite{DBLP:conf/nips/GoldbergZRXN10}. An effective approach for handling big matrices is to approximate them by their low rank counterparts which can be computed and stored efficiently. Various methods have been developed for low rank matrix approximation, including truncated singular value decomposition, matrix factorization~\cite{DBLP:conf/nips/SrebroRJ04}, matrix regression~\cite{Koltchinskii11}, column subset selection~\cite{DBLP:conf/focs/BoutsidisDM11}, the Nystr{\"o}m method~\cite{DBLP:conf/nips/WilliamsS00}.

In this work, we will focus on the CUR algorithm for low rank matrix approximation~\cite{mahoney-2009-cur}. It is a randomized algorithm that computes the low rank approximation for a given rectangle matrix by randomly sampled columns and rows of the matrix. Compared to other low rank approximation algorithms, CUR is advantageous in that it has (i) an easy interpretation of the approximation result because the subspace is constructed by the actual columns and rows of the target matrix~\cite{mahoney-2009-cur}, and (ii) strong (near-optimal) theoretical guarantee~\cite{DBLP:conf/nips/BienXM10,drineas-2006-fast,drineas-2008-relative,mahoney-2009-cur,DBLP:conf/nips/WangZ12,DBLP:journals/jmlr/WangZ13}.
The CUR matrix decomposition algorithm has been successfully applied to many domains, including bioinformatics~\cite{mahoney-2009-cur}, collaborative filtering~\cite{DBLP:conf/nips/MackeyTJ11}, video background modeling~\cite{DBLP:conf/nips/MackeyTJ11}, hyperspectral medical image analysis~\cite{DBLP:conf/kdd/MahoneyMD06}, text data analysis~\cite{drineas-2008-relative}. In the past decade, many variants of the CUR algorithm have been developed ~\cite{DBLP:conf/nips/BienXM10,drineas-2006-fast,DBLP:conf/nips/MackeyTJ11,DBLP:conf/kdd/MahoneyMD06, drineas-2008-relative,mahoney-2009-cur,DBLP:conf/nips/WangZ12,DBLP:journals/jmlr/WangZ13}.

Despite the success, one limitation with the existing CUR algorithms is that to effectively compute the low rank approximation, they require an access to the {\it full} matrix, a requirement that can be difficult to fulfill. For instance, in bioinformatics, it is usually too expensive to acquire the full expression information for hundreds of genes and thousands of individuals; in crowdsourcing, when both the number of workers and instances are large, it becomes impractical to request every worker to label all the instances in study ; in social network analysis, it is often the case that only part of the links between individuals can be accurately detected. In all the above cases, due to the physical or financial constraints, we only have a partial observation of the target matrix, making it difficult to apply the existing CUR algorithm.

%

One way to deal with the missing entries is to first compute an unbiased estimation of the target matrix based on the observed entries, and then apply the CUR algorithm to the estimated matrix. The main shortcoming of this simple method is that the unbiased estimate can be far from the target matrix when the number of observation is small, as we will show in the empirical study. Another approach is to recover the target matrix from the observed entries by using the matrix completion technique~\cite{DBLP:journals/siamjo/CaiCS10,candes-2012-exact}. Since most matrix completion algorithms are developed only for matrices of exactly low rank, they usually work poorly for matrices of full rank~\cite{DBLP:journals/corr/abs-1112-5629}. We note that although an adaptive sampling approach is developed in~\cite{krishnamurthy-20130-low} that does apply to matrices of full rank, they use a different sampling strategy,
 and their bound has a poor dependence on failure probability $\delta$ (i.e. $O(1/\delta)$),
 which significantly limit its application when both rows and columns are randomly sampled.


In this work, we address the challenge by developing a novel CUR algorithm, named \textbf{CUR$+$}, for partially observed matrix. More specifically, the proposed algorithm computes a low rank approximation of matrix $M$ based on (i) randomly sampled rows and columns from $M$, and (ii) randomly sampled entries from $M$. Unlike most matrix completion algorithms that require solving an optimization problem involving trace norm regularization ~\cite{DBLP:journals/jmlr/Bach08,DBLP:journals/siamjo/CaiCS10,DBLP:conf/icml/JiY09,DBLP:journals/jmlr/MazumderHT10,Toh:2010vx}, the proposed algorithm only needs to solve a standard regression problem and therefore is computationally efficient.
In addition, we develop a relative error bound for the proposed CUR$+$ that works for both low-rank and full-rank matrices. In particular, to perfectly recover a rank-$r$ matrix of size $n\times n$, only $O(nr\ln r)$ observed entries are needed, significantly lower than $O(nr\ln^2n)$ for standard matrix completion theory~\cite{candes-2012-exact,candes-2010-power,DBLP:journals/tit/Gross11,DBLP:journals/tit/KeshavanMO10, recht-2011-simple} and lower than $O(nr^{3/2}\ln r)$ for adaptive algorithm for matrix recovery~\cite{krishnamurthy-20130-low}. We verify our theoretical claims by empirical studies of low rank matrix approximation.

The rest of the paper is organized as follows: Section~\ref{sec:related} briefly reviews the related work on the CUR algorithms and matrix completion; Section~\ref{sec:main} presents the proposed algorithm and its theoretical properties. Section~\ref{sec:experiments} gives our empirical study. Section~\ref{sec:conclusion} concludes our work with future directions.

\section{Related Work}\label{sec:related}

\paragraph{CUR matrix decomposition} CUR algorithms compute a low rank approximation of the target matrix using the actual rows and columns of the matrix~\cite{DBLP:conf/nips/BienXM10, drineas-2006-fast,Goreinov1997b,Goreinov1997a,drineas-2008-relative,mahoney-2009-cur,Stewart1999,Tyrtyshnikov00, DBLP:conf/nips/WangZ12,DBLP:journals/jmlr/WangZ13}. More specially, let $M\in \R^{n\times m}$ be the given matrix and $r$ be the target rank for approximation. A classical CUR decomposition algorithm~\cite{drineas-2008-relative,mahoney-2009-cur} randomly samples $d_1$ columns and $d_2$ rows from $M$, according to their leverage scores, to form matrices $C$ and $R$, respectively. The approximated matrix $\widehat{M}$ is then computed as $\widehat{M}=C(C^{\dagger} M R^{\dagger})R$, where $^\dagger$ is the pseudoinverse. \cite{drineas-2006-fast} gives an additive error bound for the CUR decomposition, and a relative error bound, a significantly stronger result, is given in \cite{drineas-2008-relative}. It stated that, with a high probability,
\begin{eqnarray}
\|M - \widehat{M}\|_F \leq (1 + \epsilon) \|M - M_r\|_F \label{eqn:bound-1}
\end{eqnarray}
where $M_r$ is the best rank-$r$ approximation to $M$, and $\|\cdot\|_F$ is the Frobenius norm of a matrix.

Various improved versions of CUR have been developed. \cite{DBLP:conf/nips/MackeyTJ11} proposes a divide-and-conquer method to compute the CUR decomposition in parallel. \cite{DBLP:journals/jmlr/WangZ13} proposes an adaptive CUR algorithm with much tighter error bound and much lower time complexity. In~\cite{drineas-2006-fast}, the authors suggest a simple uniform sampling of columns and rows for the CUR decomposition when the maximum statistical leverage scores, also referred to as incoherence measure~\citep{candes-2012-exact,candes-2010-power,recht-2011-simple}, is limited. In~\cite{DBLP:conf/icml/MahoneyDMW12}, algorithms have been developed to efficiently compute the approximated values of statistical leverage scores without having to calculate the SVD decomposition of a large matrix. As we claimed in the introduction section, all the existing CUR algorithms require the knowledge of {\it every} entry in the target matrix and therefore cannot be applied directly to partially observed matrices. More complete list of related work on CUR can be found in~\cite{drineas-2008-relative,DBLP:journals/jmlr/WangZ13}.



CUR decomposition is closely related to column subset selection problem~\cite{DBLP:conf/focs/BoutsidisDM11,DBLP:conf/focs/DeshpandeR10,drineas-2008-relative}, which has been studied extensively in theoretical computer science and numerical analysis communities~\cite{drineas-2008-relative,mahoney-2009-cur,DBLP:journals/jmlr/WangZ13}. It samples multiple columns from the target matrix $M$ and use them as the basis to approximate $M$, and is often viewed as special case of the CUR algorithm. A special case of column subset selection is Nystr{\"o}m methods, which is usually used to approximate Positive Semi-Definitive (PSD) matrix in kernel learning~\cite{DBLP:conf/nips/WilliamsS00}. A more complete list of related Nystr{\"o}m methods can be found in~\cite{DBLP:journals/tit/JinYMLZ13}.

\paragraph{Matrix Completion}
The objective of matrix completion is to fill out the missing entries of a low-rank matrix
based on the observed ones. In the standard matrix completion theory, when entries are missing uniformly at random, it requires $O(nr\ln^2 n)$ observed entries to perfectly recover the target matrix under the incoherence condition~\cite{candes-2012-exact,candes-2010-power,DBLP:journals/tit/Gross11,DBLP:journals/tit/KeshavanMO10,recht-2011-simple}. Multiple improvements have been developed for matrix completion, either to deal with nonuniform missing entries or to develop tighter bounds under more strict coherence conditions.
\citep{krishnamurthy-20130-low} developed an adaptive sensing strategy for matrix completion that removes an $\ln n$ factor from the sample complexity. In~\cite{DBLP:conf/icml/BhojanapalliJ14,DBLP:conf/icml/ChenBSW14}, the authors study matrix completion when observed entries are not sampled uniformly at random. \cite{DBLP:conf/icml/NegahbanW10,rhode-2011-estimation} generalize matrix completion to matrix regression. In~\cite{DBLP:conf/nips/XuJZ13}, the authors show that the sample complexity for perfect matrix recovery can be reduced dramatically with appropriate side information.

Although it is appealing to directly combine the CUR algorithm with matrix completion to estimate a low rank approximation of a partially observed matrix, it may not work well in practice. One issue is that most matrix completion algorithms are developed for matrix of exactly low rank, significantly limiting its application to low rank matrix application. Although a few studies develop recovery bounds for matrix of full rank~\cite{DBLP:journals/corr/abs-1112-5629,krishnamurthy-20130-low}, recovery errors usually deteriorate dramatically when applied to a matrix with a long tail spectrum. In addition, most matrix completion algorithms are computationally expensive, especially for large matrices, since they require, at each iteration of optimization, computing the SVD decomposition of the approximate matrix~\cite{DBLP:journals/jmlr/Bach08,DBLP:journals/siamjo/CaiCS10,DBLP:conf/icml/JiY09,DBLP:journals/jmlr/MazumderHT10,Toh:2010vx}. In contrast, the proposed CUR algorithm scales to large matrix and works well for matrix of full rank.

\section{CUR+ for Partially Observed Matrices}\label{sec:main}

We describe the proposed CUR+ algorithm, and then present the key theoretical results for it. Due to space limitation, we postpone all the detailed analysis to the supplementary document.

\subsection{CUR+ Algorithm}
Let $M \in \R^{n\times m}$ be the matrix to be approximated, where $n \geq m$.
To approximate $M$, we first sample uniformly at random $d_1$ columns and $d_2$ rows from $M$, denoted by $A = (\a_1, \ldots, \a_{d_1}) \in \R^{n\times d_1}$, and $B = (\b_1, \ldots, \b_{d_2}) \in \R^{m\times {d_2}}$, respectively, where each $\a_i\in\R^n$ and $\b_j\in\R^m$ is the $i$th row and the $j$th column of $M$ respectively. We noticed that uniform sampling of rows and columns may not be the best strategy as it does not take into account the difference between individual rows and columns. Other sampling strategies, such as sampling rows/columns based on their statistical leverage scores~\cite{drineas-2008-relative} and adaptive sampling~\cite{krishnamurthy-20130-low,DBLP:conf/nips/WangZ12}, can be more effective. We do not choose these sampling methods because they either require an access to the full matrix~\cite{drineas-2008-relative}, introduce serious overhead in computation~\cite{DBLP:conf/nips/WangZ12}, or result in significantly worse bound when matrix is of full rank~\cite{krishnamurthy-20130-low}. Finally, for simplicity of discussion, we will assume $d_1=d_2=d$ throughout the draft even though our algorithm and analysis can easily be extended to the case when $d_1 \neq d_2$.

Let $r$ be the target rank for approximation, with $r\leq d$. $\Uh = (\uh_1, \ldots, \uh_r) \in \R^{n\times r}$ and
$\Vh = (\vh_1, \ldots, \vh_r) \in \R^{m\times r}$ are the first $r$
eigenvectors of $AA^{\top}$ and $BB^{\top}$, respectively. Besides $A$
and $B$, we furthermore sample, uniformly at random, entries from
matrix $M$.
Let $\Omega$ include the indices of randomly sampled entries. Our goal is to estimate a low rank approximation of matrix $M$ using $A$, $B$, and randomly sampled entries in $\Omega$. To this end, we need to solve the following optimization \begin{eqnarray}
\min\limits_{Z \in \R^{r\times r}}\frac{1}{2} \|\Rt_{\Omega}(M) - \Rt_{\Omega}(\Uh Z \Vh^{\top})\|_F^2 \label{eqn:opt}
\end{eqnarray}
where given $\Omega$, we define a linear operator $\Rt_{\Omega}(M): \R^{n\times m} \mapsto \R^{n\times m}$ as
\[
[\Rt_{\Omega}(M)]_{i,j} = \left\{
\begin{array}{cc}
M_{i,j} & (i,j) \in \Omega \\
0 & (i,j) \notin \Omega
\end{array}
\right.
\]

Let $Z_*$ be an optimal solution to (\ref{eqn:opt}). The estimated low rank approximation is given by $\Mh = \Uh Z_*\Vh^{\top}$. $\Mh$ can also be expressed using standard $C\times U\times R$ formulation by solving a group of linear equations. We note that (\ref{eqn:opt}) is a standard regression problem and therefore can be solved efficiently using the standard regression method (e.g. accelerated gradient descent~\cite{book:nesterov}). We refer to the proposed algorithm as \textbf{CUR}$\mathbf{+}$.

\subsection{Guarantee for CUR+}

Before presenting the theoretical results, we first describe the notations that will be used throughout the analysis. Let $\sigma_i, i=1, \ldots, m$ be the singular values of $M$ ranked in descending order, and let $\u_i$ and $\v_i$ be the corresponding left and right singular vectors. Define $U = (\u_1, \ldots, \u_m)$ and $V = (\v_1, \ldots, \v_m)$. Given $r \in [m]$, partitioning the SVD decomposition of $M$ as
\begin{eqnarray}
    M = U\Sigma V^{\top} = \begin{array}{cc}r & m - r \\ \mbox{$[$}U_1 & U_2\mbox{$]$} \end{array}\left[ \begin{array}{cc} \Sigma_1 & \\ & \Sigma_2\end{array}\right]\left[\begin{array}{c}V_1^{\top} \\ V_2^{\top} \end{array} \right] \label{eqn:partition}
\end{eqnarray}
Let $\ut_i, i \in [n]$ be the $i$th row of $U_1$ and $\vt_i, i \in [m]$ be the $i$th row of $V_1$. The \emph{incoherence} measure for $U_1$ and $V_1$ is defined as
\begin{eqnarray}
\mu(r) = \max\left(\max\limits_{i \in [n]} \frac{n}{r}|\ut_i|^2, \max\limits_{i \in [m]} \frac{m}{r}|\vt_i|^2 \right) \label{eqn:mu-1}
\end{eqnarray}
Similarly, we can have the \emph{incoherence} measure for matrices $\Uh$ and $\Vh$ that include the first $r$ eigenvectors of $AA^{\top}$ and $BB^{\top}$, respectively. Let $\uh'_i, i \in [n]$ be the $i$th row of $\Uh$ and $\vh'_i, i \in [m]$ be the $i$th row of $\Vh$. Define the incoherence measure for $\Uh$ and $\Vh$ as
\begin{eqnarray}
\muh(r)& =& \max\left(\max\limits_{i \in [n]} \frac{n}{r}|\uh'_i|^2, \max\limits_{i \in [m]} \frac{m}{r}|\vh'_i|^2 \right) \label{eqn:mu-2} 
\end{eqnarray}
Define projection operators $P_{U} = UU^{\top}$, $P_V = VV^{\top}$, $P_{\Uh} = \Uh\Uh^{\top}$, and $P_{\Vh} = \Vh\Vh^{\top}$. We will use $\|\cdot\|_2$ and $\|\cdot\|_F$ respectively for the spectral norm and Frobenius norm of a matrix.

We first present the theoretical guarantee for the CUR+ algorithm when the rank of the target matrix $M$ is no greater than $r$.
\begin{thm} \label{thm:low-rank} (\textbf{Low-Rank Matrix Approximation})
Assume $\mbox{rank}(M) \leq r$, $d \geq 7\mu(r) r (t+\ln r)$, and $|\Omega| \geq 7\mu^2(r) r^2 (t+2\ln r)$. Then, with a probability at least $1 - 5e^{-t}$, we have $M = \Mh$, where $\Mh$ is a low rank approximation estimated by the CUR+ algorithm.
\end{thm}

\begin{table}[t]
\caption{Current results of sample complexity for matrix completion (including matrix regression).}\label{tbl:samplec}
\begin{center}
\begin{tabular}{l|cccccc}
\hline
Method  &CUR$+$& \cite{krishnamurthy-20130-low}& \cite{DBLP:conf/icml/BhojanapalliJ14}&  \cite{DBLP:conf/stoc/JainNS13}&\cite{candes-2012-exact,candes-2010-power,DBLP:conf/icml/ChenBSW14,DBLP:journals/tit/KeshavanMO10,recht-2011-simple} \\
\# Observation &$nr\ln r$&$nr^{3/2}\ln r$& $nr^2$& $nr^{4.5}\ln n$& $nr\ln^2n$\\
\hline
\end{tabular}
\end{center}
\end{table}

\paragraph{Remark} Theorem~\ref{thm:low-rank} shows that a rank-$r$ matrix can be perfectly recovered from $2dn + |\Omega| = O(nr\ln r)$ observed entries if we set $t = \Omega(\ln r)$. In Table~\ref{tbl:samplec}, we compare the sample complexity of the CUR+ algorithm with the sample complexity of the other matrix completion algorithms. We observe that our result significantly improves the sample complexity from previous work.
We should note that unlike~\citep{krishnamurthy-20130-low} where the incoherence measure is only assumed for column vectors, we assume a small incoherence measure for both row and column vectors here. It is this stronger assumption that allows us to sample both rows and columns, leading to the improvement in the sample complexity from $O(nr^{3/2}\ln r)$ in \citep{krishnamurthy-20130-low} to $O(nr\ln r)$.

We now consider a more general case where matrix $M$ is of full rank. Theorem~\ref{thm:combine} bounds the difference between $M$ and $\Mh$, measured in spectral norm,
\begin{thm} \label{thm:combine}
Let $r \leq m$ be an integer that is no larger than $m$. Assume (i) $d \geq {7\mu(r) r (t+\ln r)}$ , and (ii) $|\Omega| \geq 7 \muh^2(r) r^2(t+2\ln r)$. Then with a probability at least $1 - 3e^{-t}$
\[
\|M - \Mh\|^2_2 \leq 8\sigma^2_{r+1}\left(1+2mn\right)\left(1 + \frac{m + n}{d}\right).
\]
\end{thm}

As indicated by Theorem~\ref{thm:combine}, when both $\mu(r)$ and $\muh(r)$, the incoherence measure for the first $r$ singular/eigen vectors of $M$ and the sampled columns/rows, are small, we have
\[
\|M - \Mh\|_2 \leq O\left(\sqrt{mn}\sqrt{\frac{n}{d}}\|M - M_r\|_2\right)
\]
provided that $d \geq O(r\ln r)$ and $|\Omega| \geq O(r^2 \ln r)$. 

One limitation with Theorem~\ref{thm:combine} is that $\muh(r)$ is a random variable depending on the sampled columns and rows. Since $\muh(r)$ can be as high as $n/r$, $|\Omega|$, the number of observed entries required by Theorem~\ref{thm:combine}, can be as large as $O(n^2)$, making it practically meaningless. Below, we develop a result that explicitly bounds $\muh$ with a high probability. Using the high probability bound for $\muh$, we are able to show that under appropriate conditions, we need at most $O(n^2/d^2)$ observed entries in order to establish a relative error bound for $\|M - \Mh\|$.

To make our analysis simple, we focus on the case when $M$ is of full rank but with skewed singular value distribution. In particular, we assume $\sigma_r\ge\sqrt{2}\sigma_{r+1}$. In order to effectively capture the skewed singular value distribution, we introduce the concept of \emph{numerical rank} $r(M, \eta)$~\cite{book:golub} with respect to non-negative constant $\eta > 0$
\[
r(M, \eta) = \sum_{i=1}^m \frac{\sigma_i^2}{\sigma_i^2 + mn\eta}
\]
Note that when $\eta=0$, the numerical rank is equivalent to the true rank of the matrix. The larger $\eta$ is , the smaller it compared to the true rank. In the following analysis, we will replace rank $r$ with numerical rank $r(M, \eta)$.

We furthermore generalize the definition of \emph{incoherence} measure to matrix with \emph{numerical rank}, that is, we further define incoherence measure $\mu(\eta)$ as
\begin{eqnarray}
\mu(\eta) = \max\left(\max\limits_{1 \leq i \leq m} \frac{m}{r(M, \eta)}|V_{i,*} \Sigma|^2, \max\limits_{1 \leq i \leq n} \frac{n}{r(M, \eta)}|U_{i,*}\Sigma|^2 \right) \label{eqn:mu-3}
\end{eqnarray}
It is easy to verify that $\mu(\eta) \geq 1$. Compared to the standard incoherence measure defined in (\ref{eqn:mu-1}), the key difference is that (\ref{eqn:mu-3}) introduces singular values $\Sigma$ into the definition of incoherence measure, making it appropriate for matrix of full rank.

The following two lemmas relate $r\mu(r)$ and $r\muh(r)$, respectively, with $r(M, \eta)\mu(\eta)$,
\begin{lemma}\label{lem:murmuzeta}
If we choose $\eta = \sigma_r^2/mn$, we have
\[
{r}\mu(r) \leq 2r(M, \eta)\mu(\eta)
\]
\end{lemma}

\begin{lemma} \label{thm:muh-1}
Assume that $d \geq 16(\mu(\eta)r(M, \eta) + 1) (t + \ln n)$, and $\sigma_r \geq \sqrt{2}\sigma_{r+1}$. Set $\eta = \sigma_r^2/mn$. With a probability $1 - 4e^{-t}$, we have
\[
r\muh(r) \leq 2r(M,\eta)\mu(\eta) + 18 n \delta^2/r\;\;\;\;\emph{ where }\;\;\;\;\delta^2 = \frac{4}{d} (\mu(\eta)r(M, \eta) + 1)(t + \ln n)
\]
\end{lemma}

Using Theorem~\ref{thm:combine}, Lemma~\ref{lem:murmuzeta} and ~\ref{thm:muh-1}, we have the result for full-rank matrix with skewed singular value distribution,
\begin{thm} \label{thm:high-rank} (\textbf{Full Rank Matrix Approximation})
Assume $d \geq 16(\mu(\eta) r(M, \eta) + 1)(t + \ln n)$ and $\sigma_r \geq \sqrt{2} \sigma_{r+1}$. Set $\eta = \sigma_r^2/mn$. We have, with a probability $1 - 7e^{-t}$,
\[
\|M - \Mh\|_2^2 \leq 8\sigma^2_{r+1}\left(1+2mn\right)\left(1 + \frac{m + n}{d}\right).
\]
\[\text{ if }
|\Omega| \geq 7\left(2\mu(\eta)r(M, \eta) + 72\frac{n}{d}(\mu(\eta) r(M, \eta) + 1)(t + \ln n)\right)^2(t+2\ln r)=O\left(\frac{n^2}{d^2}\right)
\]
\end{thm}
As indicated by Theorem~\ref{thm:high-rank}, we will have a bound similar to that of Theorem~\ref{thm:combine} if $|\Omega| \geq O(n^2/d^2)$. The key difference between Theorem \ref{thm:combine} and~\ref{thm:high-rank} is that in Theorem~\ref{thm:combine}, the requirement for $|\Omega|$ depends on $\muh(r)$, a random variable depending on the sampled rows and columns. In contrast, in Theorem~\ref{thm:high-rank}, we remove $\muh$ and bound $|\Omega|$ directly. We finally note that the result $|\Omega| \geq O(n^2/d^2)$ requires a large number of sampled entries for accurately estimating the low rank approximation of the target matrix. This is mostly due to the potentially loose bound for $\widehat{\mu}$. It remains an open question whether it is possible to reduce the number of observed entries for CUR-type low rank approximation.

\section{Experiments}\label{sec:experiments}

We first verify the theoretical result in Theorem~\ref{thm:low-rank}, i.e. the dependence of sample complexity on $r$ and $n$, using synthetic data. We then evaluate the performance of the proposed CUR+ algorithm by comparing it to the state-of-the-art algorithms for low rank matrix approximation. We implement the proposed algorithm using Matlab, and all the experiments were run on a Linux server with CPU 2.53GHz and 48GB memory.

\begin{figure*}[!t]
\centering
\begin{minipage}[h]{1.3in}
\centering
\includegraphics[width= 1.3in]{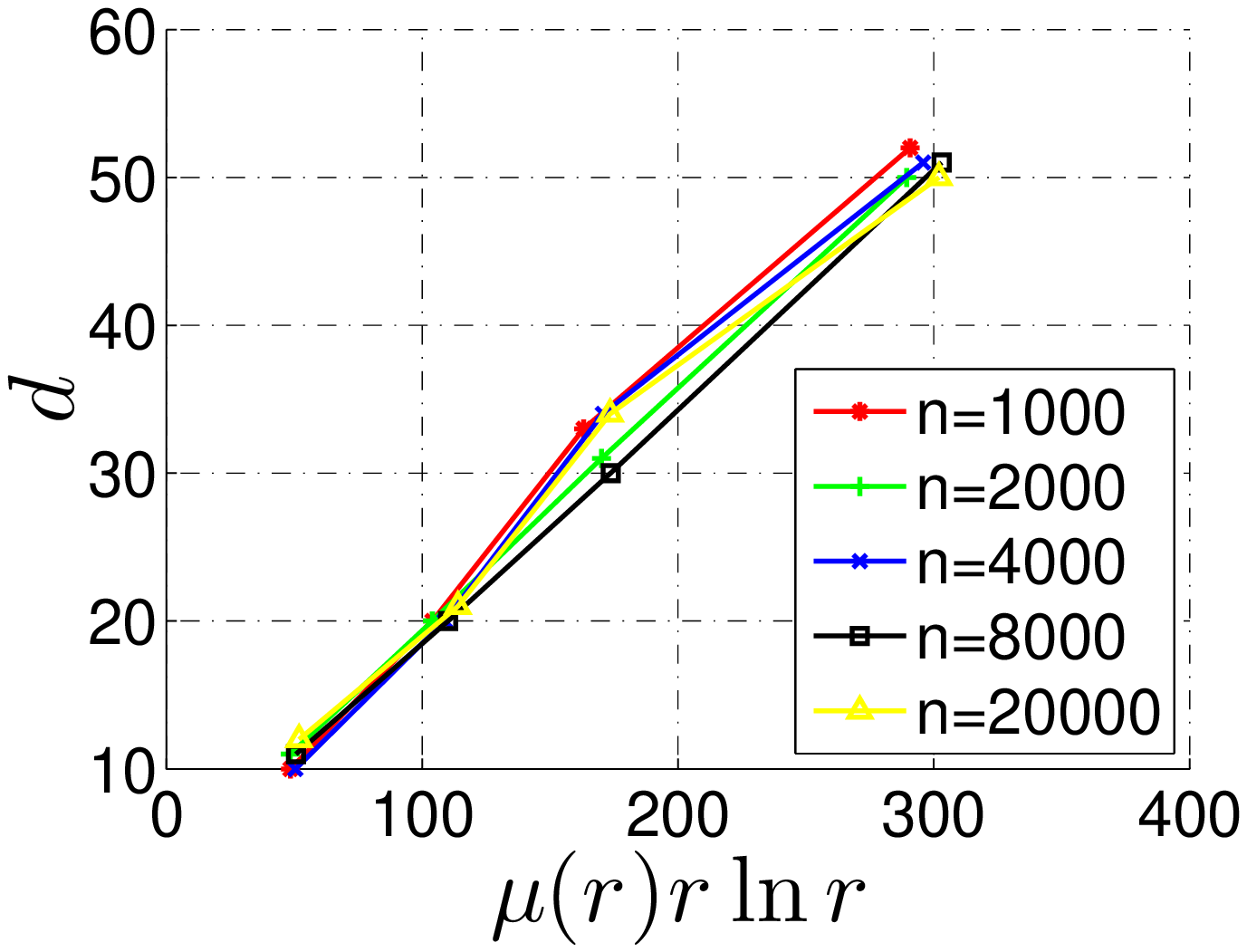}\\
\mbox{(a)}
\end{minipage}
\begin{minipage}[h]{1.3in}
\centering
\includegraphics[width= 1.3in]{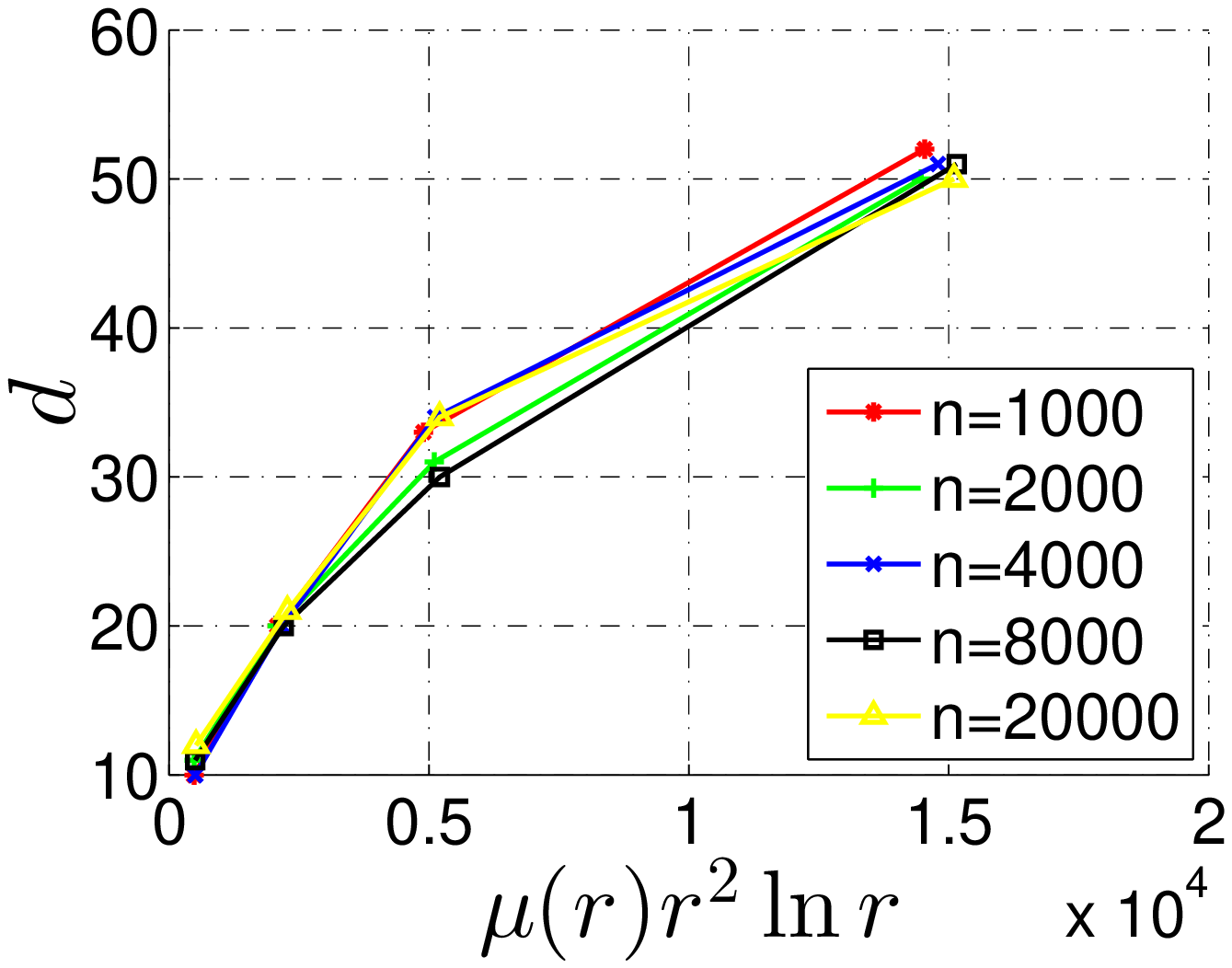}\\
\mbox{(b)}
\end{minipage}
\begin{minipage}[h]{1.3in}
\centering
\includegraphics[width= 1.3in]{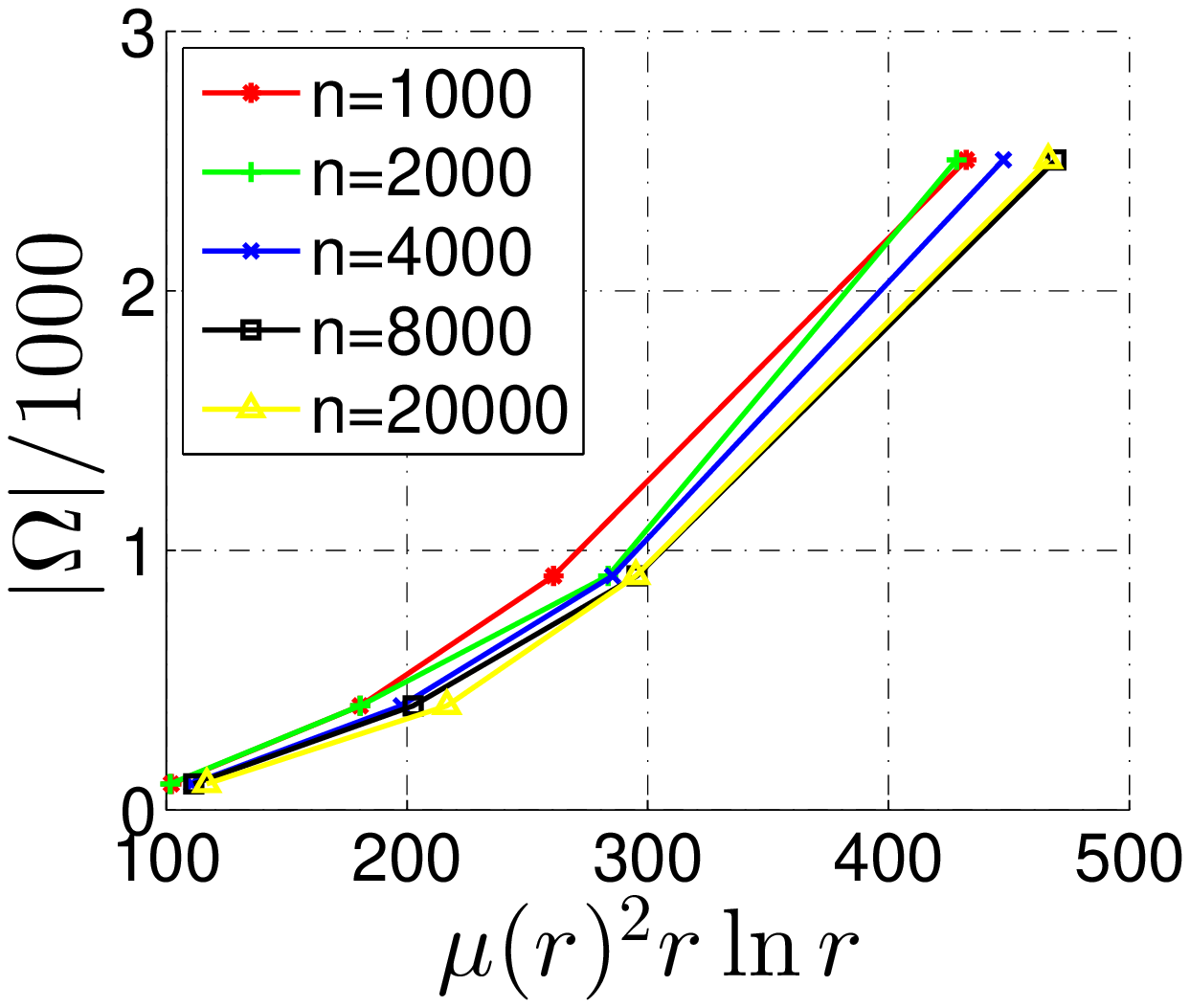}\\
\mbox{(c)}
\end{minipage}
\begin{minipage}[h]{1.3in}
\centering
\includegraphics[width= 1.3in]{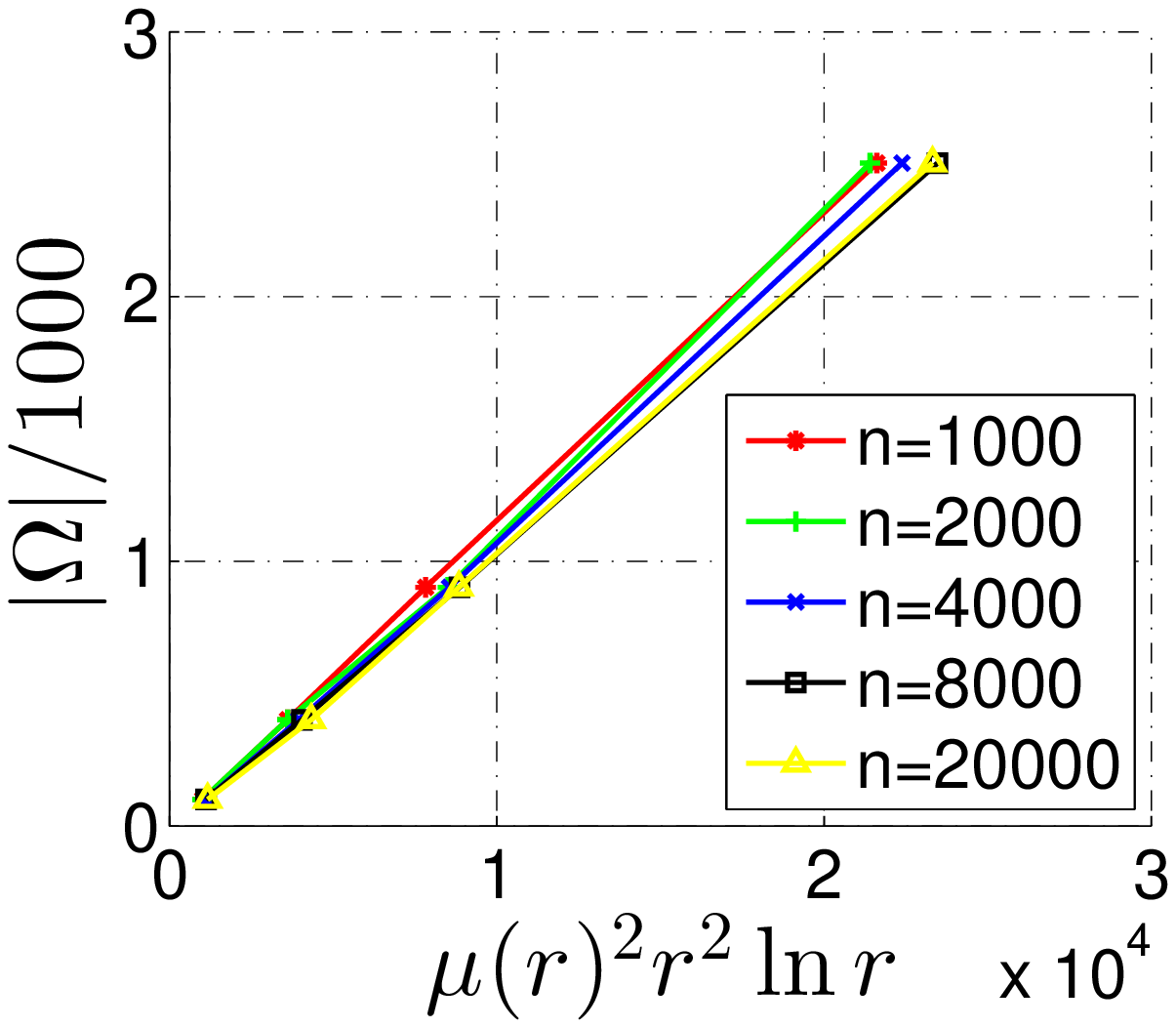}\\
\mbox{(d)}
\end{minipage}
\caption{Experiment results on the synthetic data. (a)(b) plot the minimum $d$ for perfect matrix recovery against $r\ln r$ and $r^2 \ln r$ respectively, and (c)(d) plot the minimum $|\Omega|$ for perfect matrix recovery against $r\ln r$ and $r^2\ln r$. The results confirm the theoretical finding in Theorem~\ref{thm:low-rank}, i.e. $d = O(r\ln r)$ and $|\Omega| = O(r^2\ln r)$.}\label{fig:simulation}
\end{figure*}

\subsection{Experiment (I): Verifying the Dependence on $r$ and $n$}
We will verify the sample complexity result in Theorem~\ref{thm:low-rank}, i.e. $d \geq O(r\ln r)$ and $|\Omega| \geq O(r^2 \ln r)$. We note both the requirements on $d$ and $|\Omega|$ are independent from matrix size.

\paragraph{Settings} Here we study square matrices of different sizes and ranks, with $n$ varied in $\{1,000; 2,000;\\4,000; 8,000; 10,000\}$, and $r$ varied in $\{10,20,30,50\}$. For each special $n$ and $r$, we search for the smallest $d$ and $|\Omega|$ that can lead to almost perfect recovery of  the target matrix (i.e. $\|M-\widehat M\|_F/\|M\|_F\le 2\times 10^{-4}$) in all $10$ independent trials.
To create the rank-$r$ matrix $M\in\R^{n\times n}$, we first randomly generate matrix $M_L\in\R^{n\times r}$ and $M_R\in \R^{r\times n}$ with each entry of $M_L$ and $M_R$ drawn independently at random from $\mathcal{N}(0,1)$, and $M$ is given by $M=M_L\times M_R$. To create $A$ and $B$, we sample uniformly at random $d$ rows and columns. We further sample $|\Omega|$ entries from $M$ to be partially observed. Under this construction scheme, the difference between the incoherence $\mu(r)$ for different sized matrices are relatively small (from minimum $1.4127$ to maximum $2.4885$). Although we will plot $d$ and $|\Omega|$'s dependence on $\mu(r)$, we will ignore their impact in discussion of the results.

\paragraph{Results} The dependence of minimal $d$ on $r$ and $n$ is given in Figure~\ref{fig:simulation}(a) and (b), where (a) plots $d$ against $r\ln r$ and (b) shows $d$ versus $r^2\ln r$. We can see clearly that $d$ has a linear dependence on $r\ln r$. We also observed from Figure~\ref{fig:simulation}(a) that $d$ is almost independent from $n$, the matrix size. Figure~\ref{fig:simulation}(c) and (d) plot the $|\Omega|$, the minimum number of observed entries, against $r\ln r$ and $r^2 \ln r$. The result in Figure~\ref{fig:simulation} (d) confirms our theoretical finding, i.e. $|\Omega| \propto r^2\ln r$.
%
%


\subsection{Experiment(II): Comparison with Baseline Methods for Low Rank Approximation}\label{sec:application_cur}

We evaluate the performance of the proposed CUR+ algorithm on several benchmark data sets that have been used in the recent studies of the CUR matrix decomposition algorithm, including Enron emails ($39,861\times 28,102$), Dexter ($20,000\times 2,600$), Farm Ads ($54,877\times 4,143$) and Gisette ($13,500\times 5,000$), where each row of the matrix corresponds to a document and each column corresponds to a term/word. Detailed information of these data sets can be found in~\cite{DBLP:journals/jmlr/WangZ13}. All four matrices are of full rank and have skewed singular value distribution, as shown in Figure~\ref{fig:singular_values} 

\begin{figure*}[!t]
\centering
\begin{minipage}[h]{1.3in}
\centering
\includegraphics[width= 1.3in]{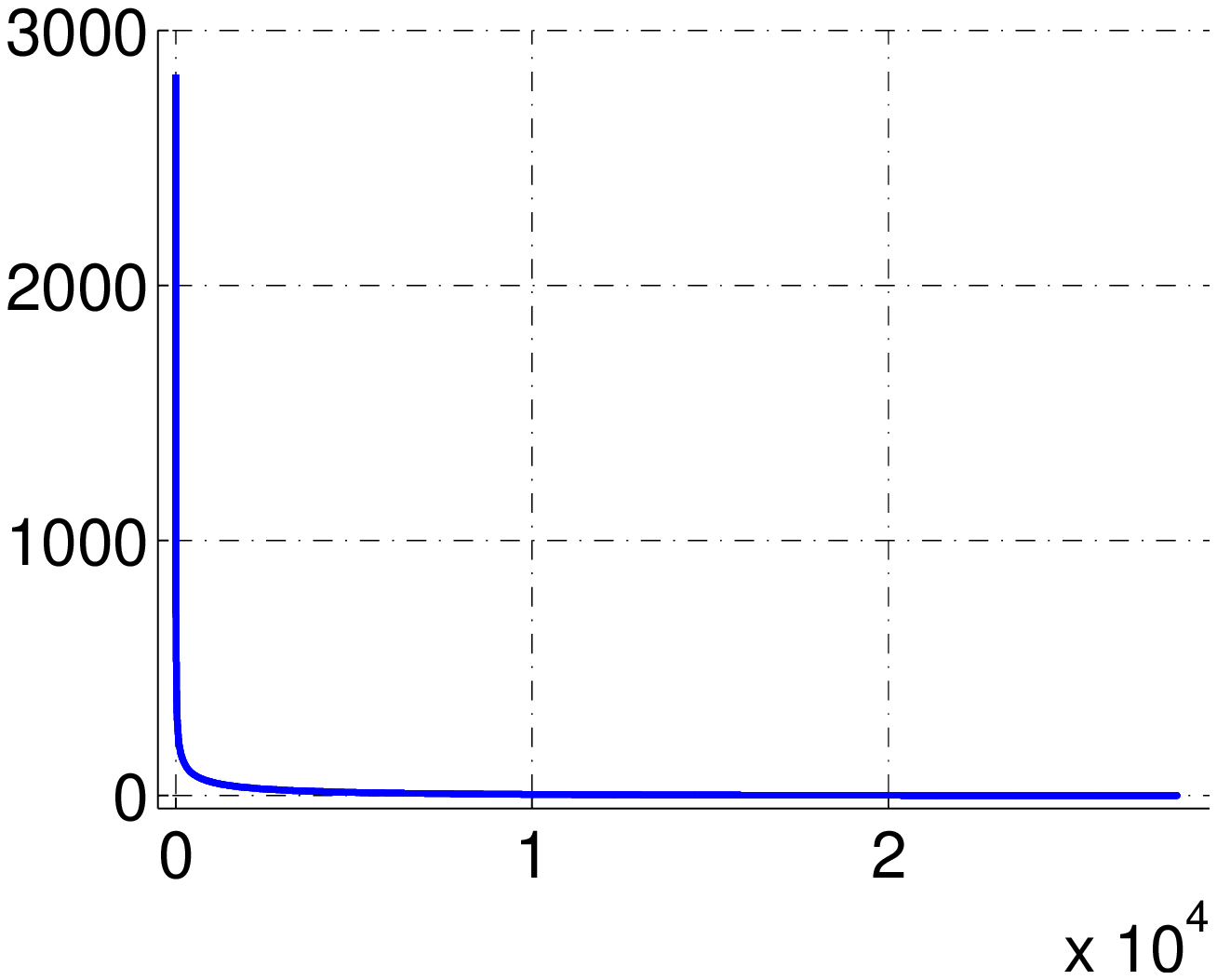}\\
\mbox{Enron}
\end{minipage}
\begin{minipage}[h]{1.3in}
\centering
\includegraphics[width= 1.3in]{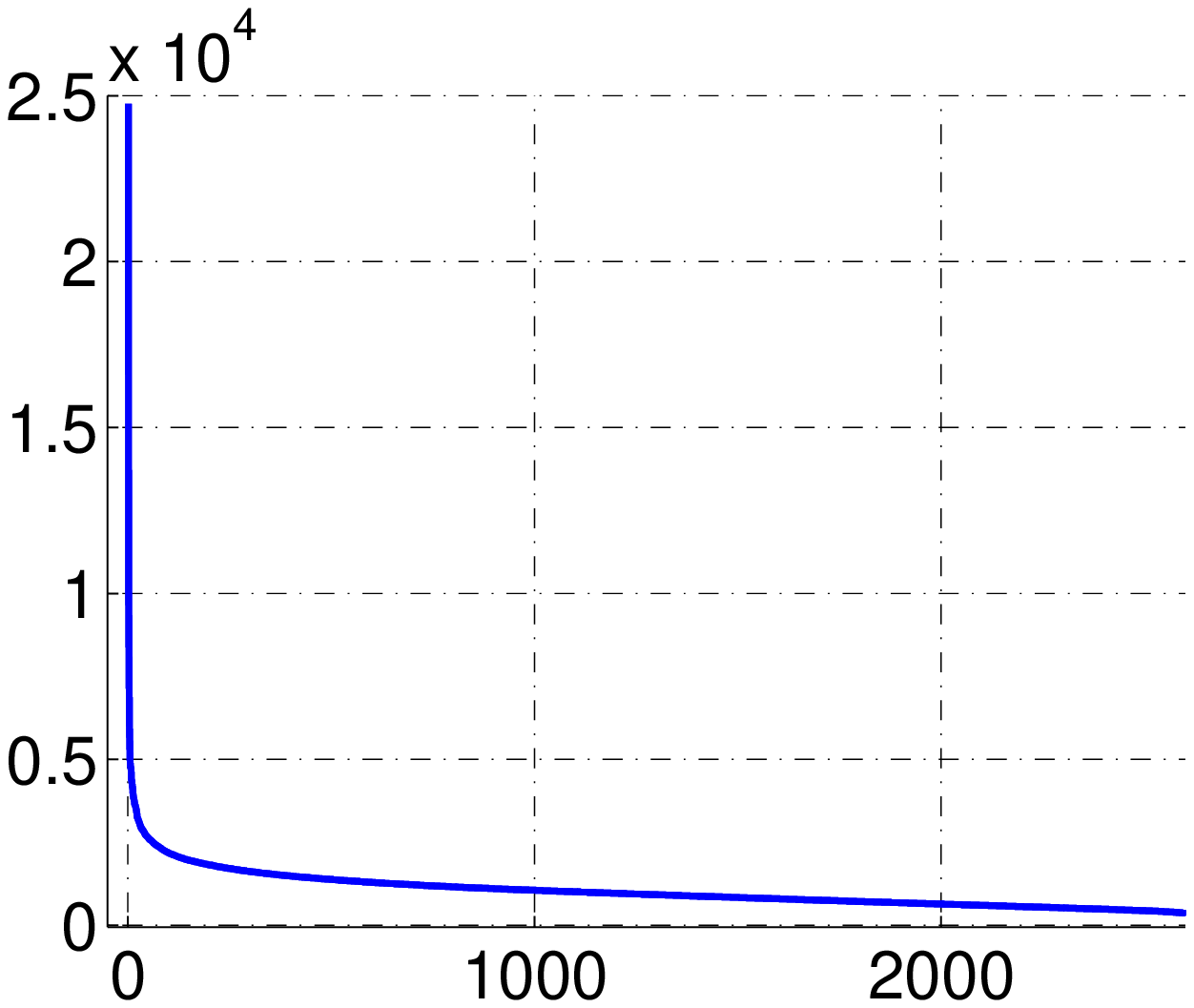}\\
\mbox{Dexter}
\end{minipage}
\begin{minipage}[h]{1.3in}
\centering
\includegraphics[width= 1.3in]{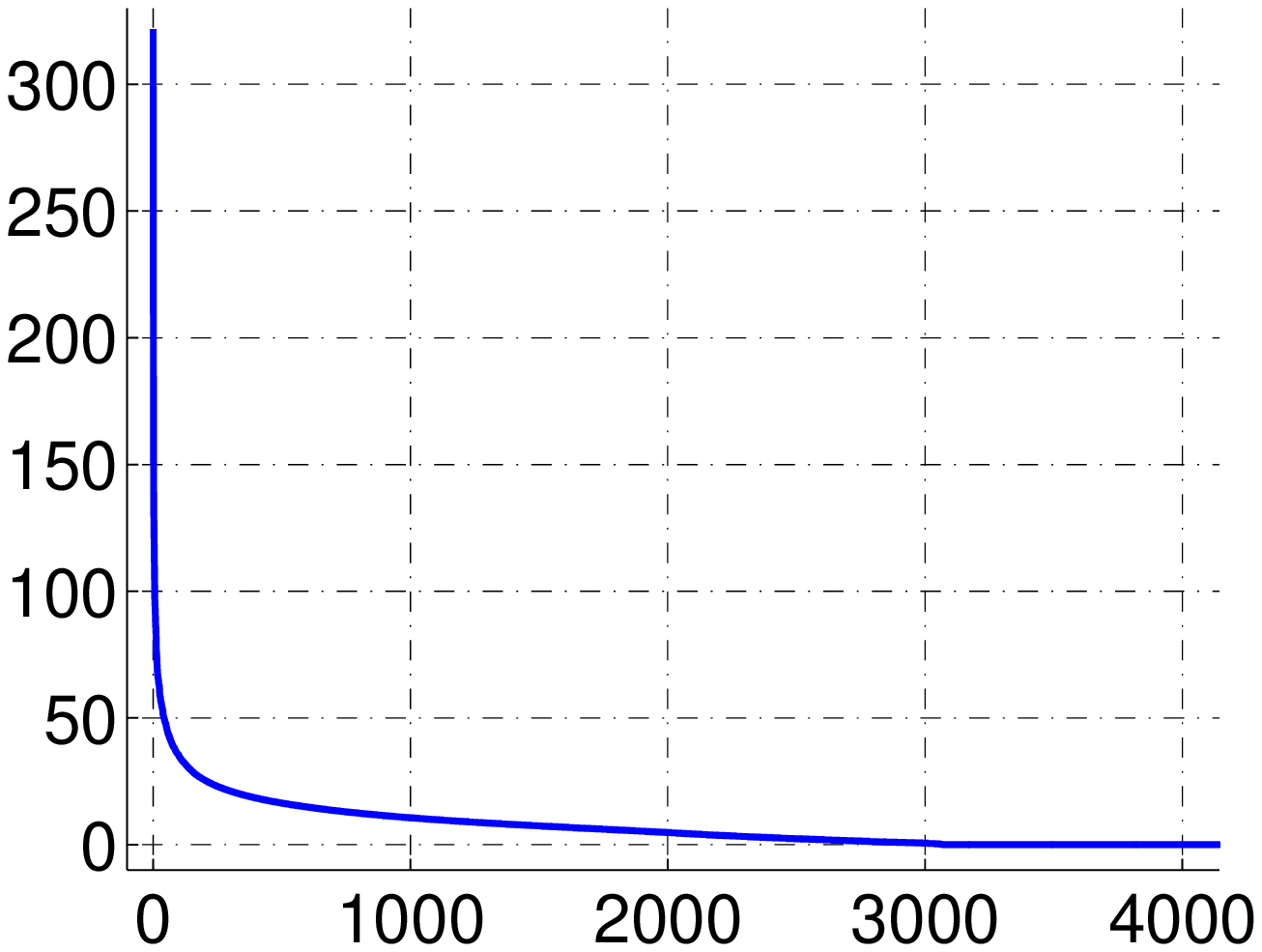}\\
\mbox{Farm Ads}
\end{minipage}
\begin{minipage}[h]{1.3in}
\centering
\includegraphics[width= 1.3in]{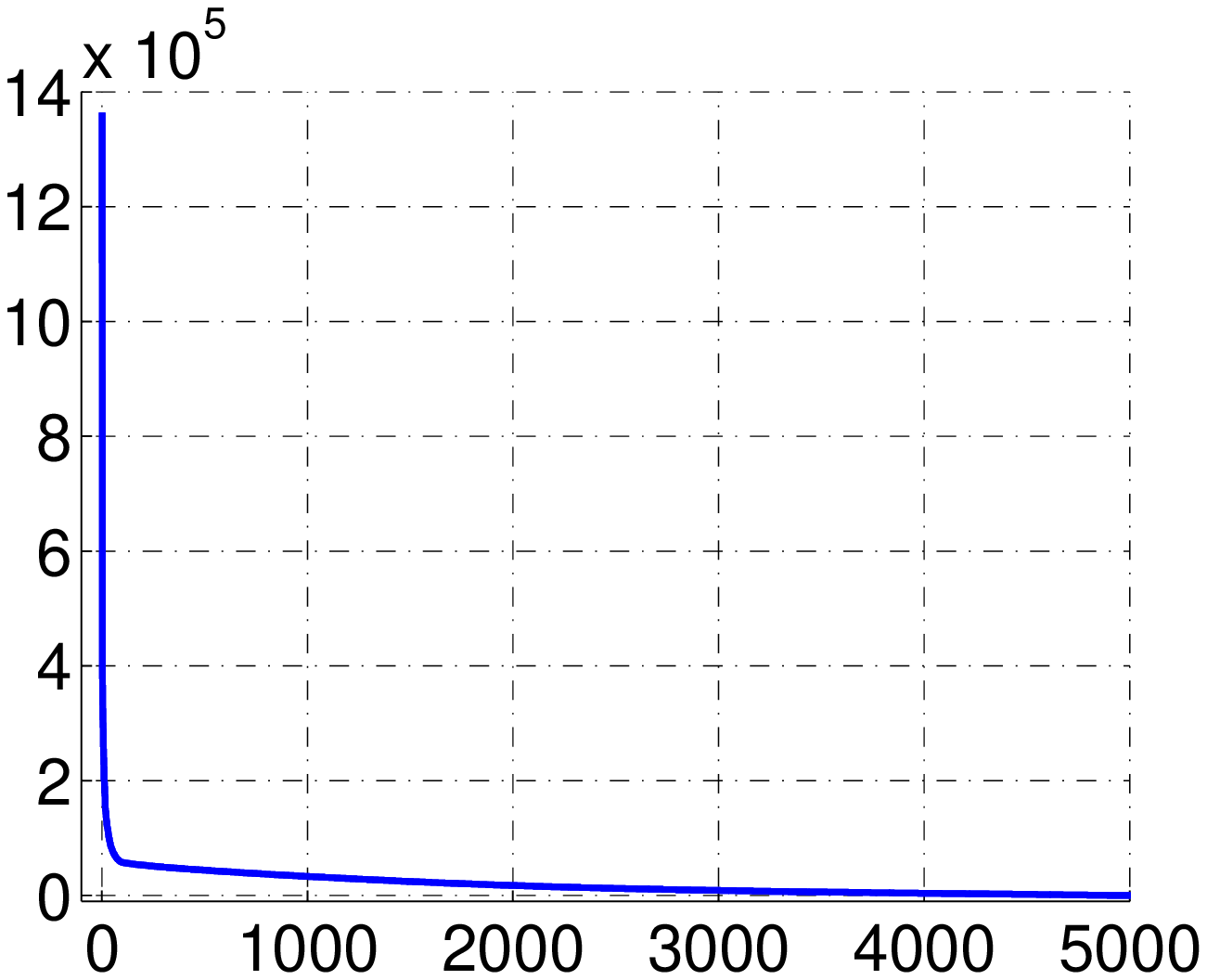}\\
\mbox{Gisette}
\end{minipage}
\caption{Singular values of real data ranked in descending order.
All these four data sets are full-rank and have skewed singular value distribution.
}\label{fig:singular_values}
\end{figure*}


\paragraph{Baselines} Since both the rows/columns and entries observed in the proposed algorithm are sampled uniformly at random, we only compare our approach to the standard CUR algorithm using uniformly sampled rows and columns. Although the adaptive sampling based approaches~\cite{krishnamurthy-20130-low} usually yield lower errors than the standard CUR algorithm, they do not choose observed entries randomly and therefore are not included in the comparison. Let $C$ be a set of $d_1$ sampled columns and $R$ be the set of $d_2$ sampled rows. The low rank approximation by the CUR algorithm is given by $\widehat{M} = CZR$, where $Z \in \R^{d_1\times d_2}$. Two methods are adopted to estimate $Z$. We first estimated $Z$ by $Z = C^{\dagger} M R^{\dagger}$. Since this estimation requires an access to the full matrix, we refer to it as \textbf{CUR-F}. In the second method, we first construct an unbiased estimator $M_e$ by using the randomly observed entries in $\Omega$, and then estimate matrix $Z$ by $Z = C^{\dagger}M_eR^{\dagger}$. Here, the unbiased estimation $M_e$ is given by
\[
[M_e]_{i,j} = \left\{
\begin{array}{cc}
\frac{mn}{|\Omega|}M_{i,j} & (i,j) \in \Omega \\
0 & (i,j) \notin \Omega
\end{array}
\right.
\]
We call this algorithm \textbf{CUR-E}. Evidently, CUR-F is expected to work better than our proposal and will provide a lower bound for the CUR algorithm for partially observed matrices.



\paragraph{Settings}

To make our result comparable to the previous studies, we adapted the same experiment strategy as in~\cite{DBLP:conf/nips/WangZ12,DBLP:journals/jmlr/WangZ13}. More specially, for each data set, we set $d_1=\alpha r$ and $d_2=\alpha d_1$, with rank $r$ varied in the range of $(10, 20, 50)$ and $\alpha$ varied from $1$ to $5$. To create partial observations, we randomly sample $|\Omega| = \Omega_0 = nmr^2/nnz(M)$ entries from the target matrix $M$, where $nnz(M)$ is the number of non-zero entries of $M$. We measure the performance of low rank matrix approximation by the relative spectral-norm difference $\ell_s=\|M - \Mh\|/\|M - M_r\|$
which has solid theoretical guarantee according to
Theorem~\ref{thm:high-rank}. We noticed that most previous work report their results in the form of relative Frobenius norm, thus we will also show the results compared to the state-of-the-art algorithms for low rank matrix approximation measured by the relative Frobenius norm $\ell_F=\|M-\hat M\|_F/\|M-M_r\|_F$. Finally, we follow the experimental protocol specified in~\cite{DBLP:conf/nips/WangZ12} by repeating every experiment $10$ times and reporting the mean value.

\paragraph{Results} Figure~\ref{fig:vard} shows the results of low rank matrix approximation for $r = 10,20,50$. We observe that the CUR$+$ works significantly better than the CUR-E method, and yields a similar performance as the CUR-F that has an access to the full target matrix $M$.

\begin{figure*}[!t]
\centering
\begin{minipage}[h]{1.3in}
\centering
\includegraphics[width=1.3in]{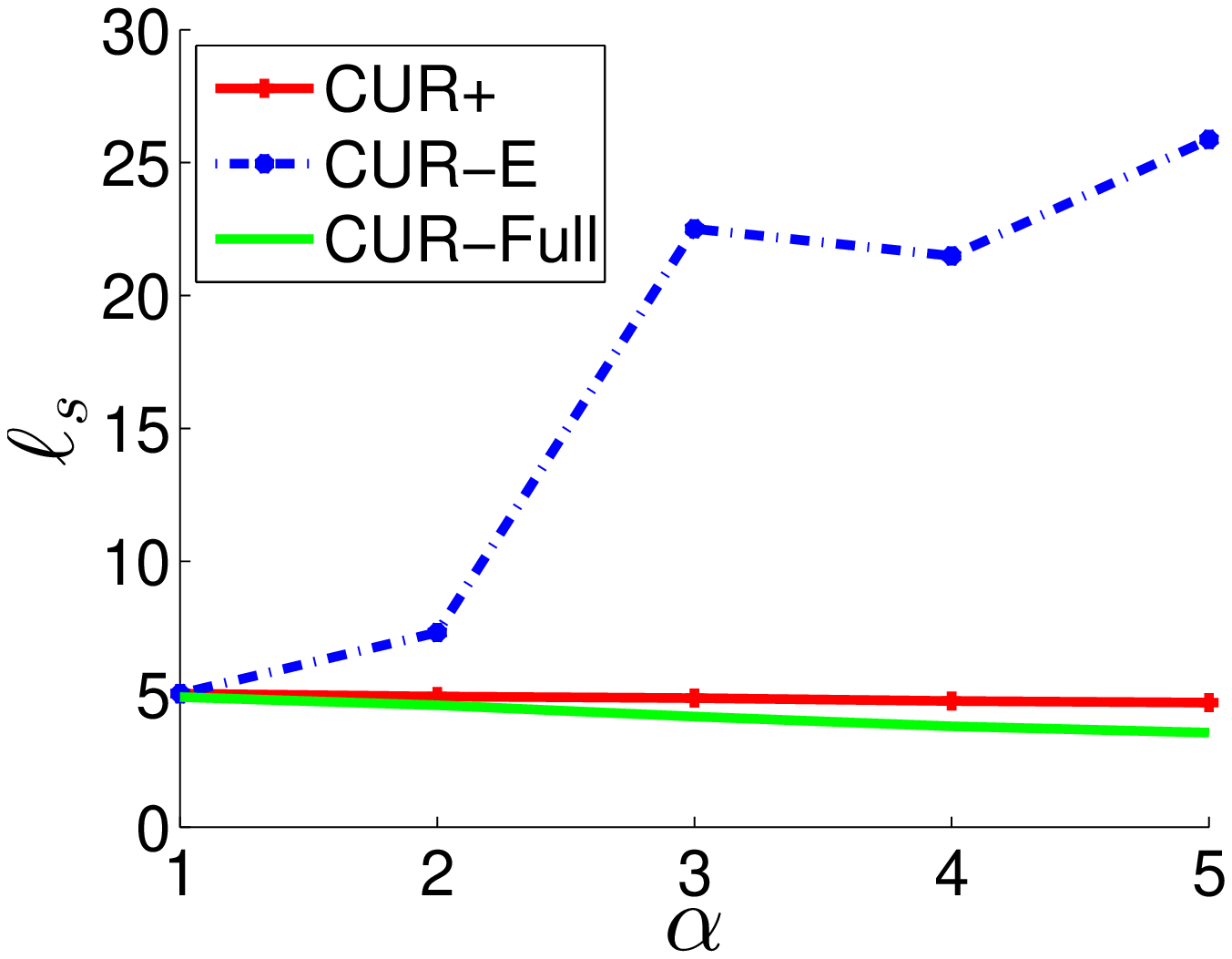}\\
\mbox{Enron $r=10$}
\end{minipage}
\begin{minipage}[h]{1.3in}
\centering
\includegraphics[width= 1.3in]{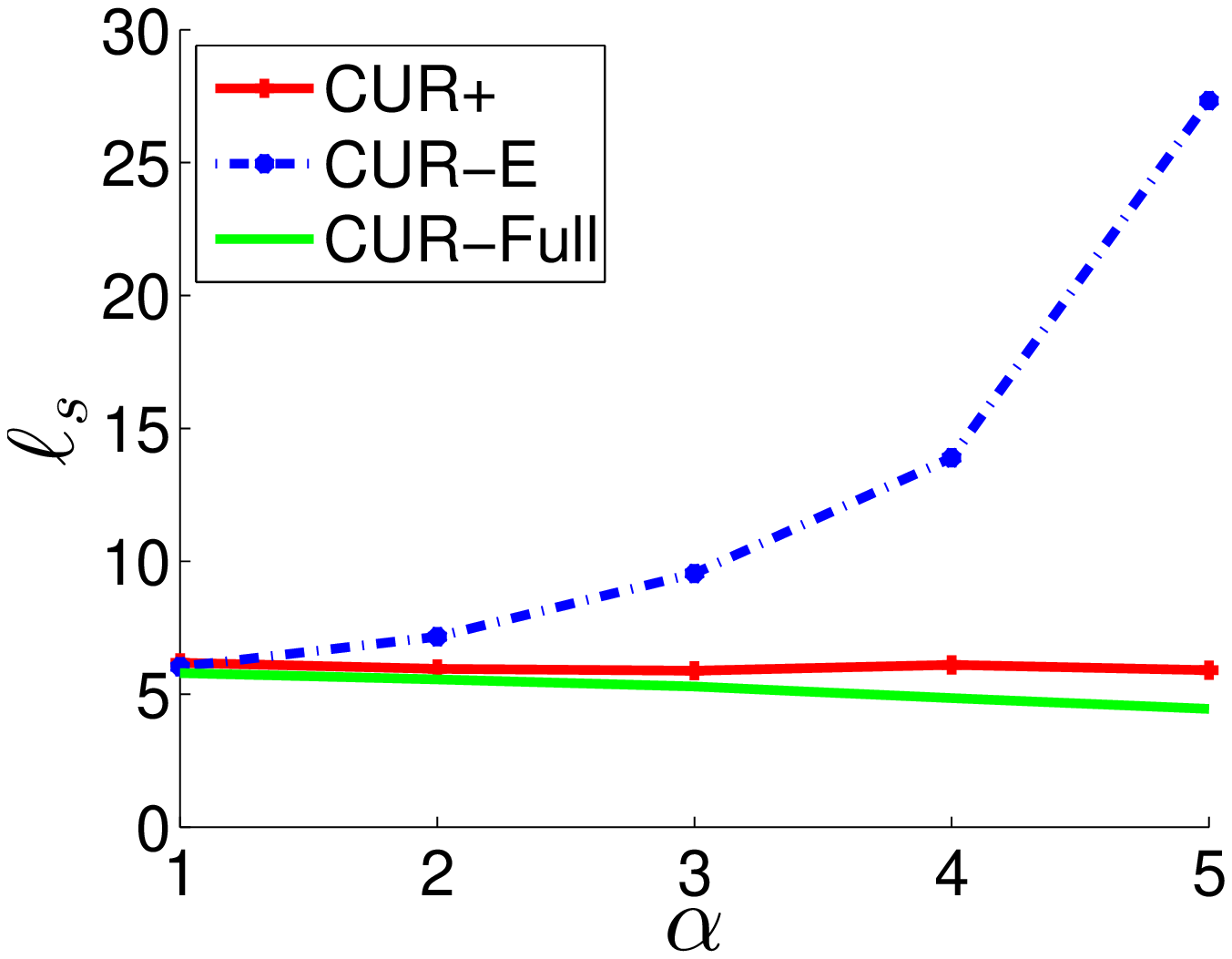}\\
\mbox{Dexter $r=10$}
\end{minipage}
\begin{minipage}[h]{1.3in}
\centering
\includegraphics[width= 1.3in]{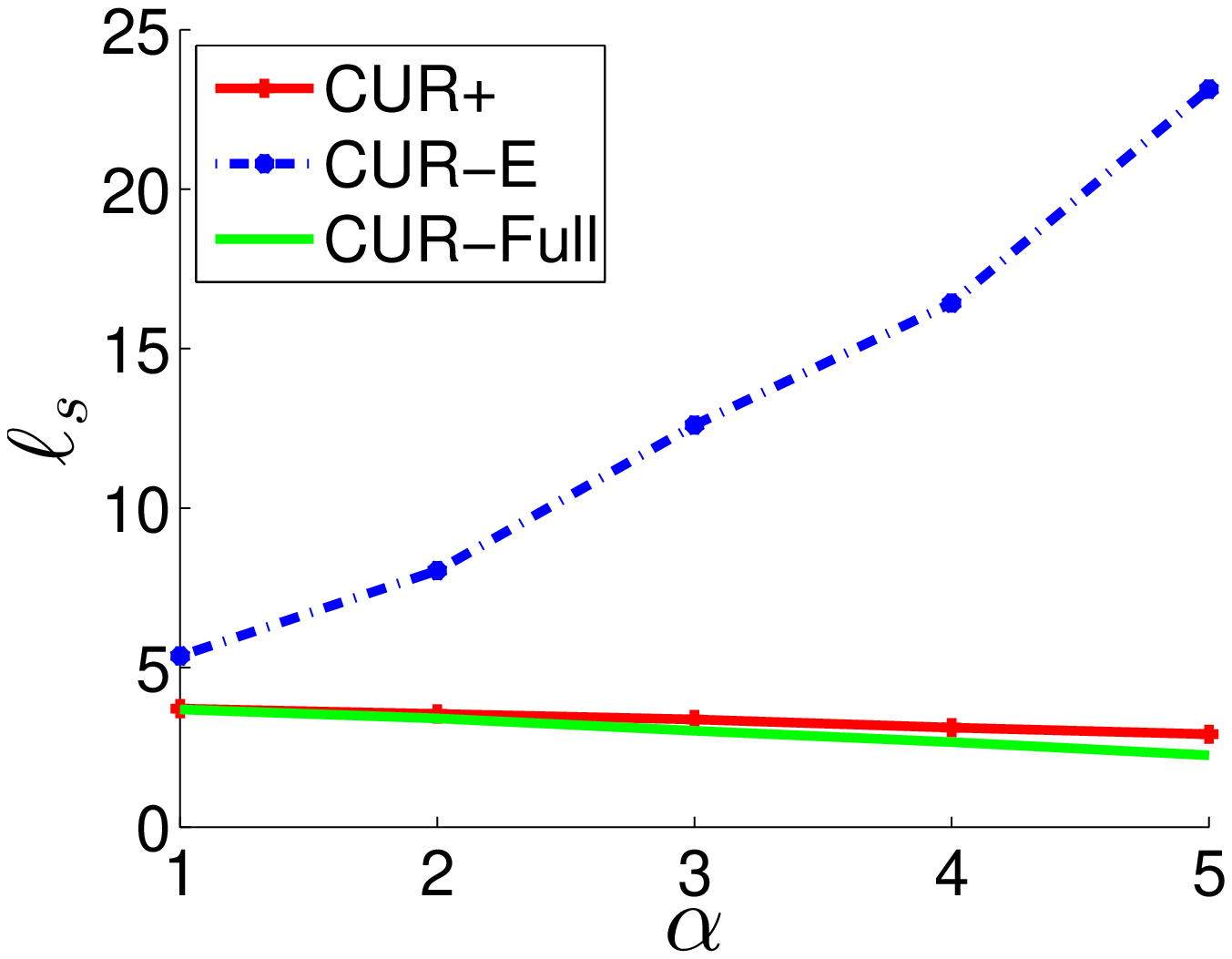}\\
\mbox{Farm Ads $r=10$}
\end{minipage}
\begin{minipage}[h]{1.3in}
\centering
\includegraphics[width= 1.3in]{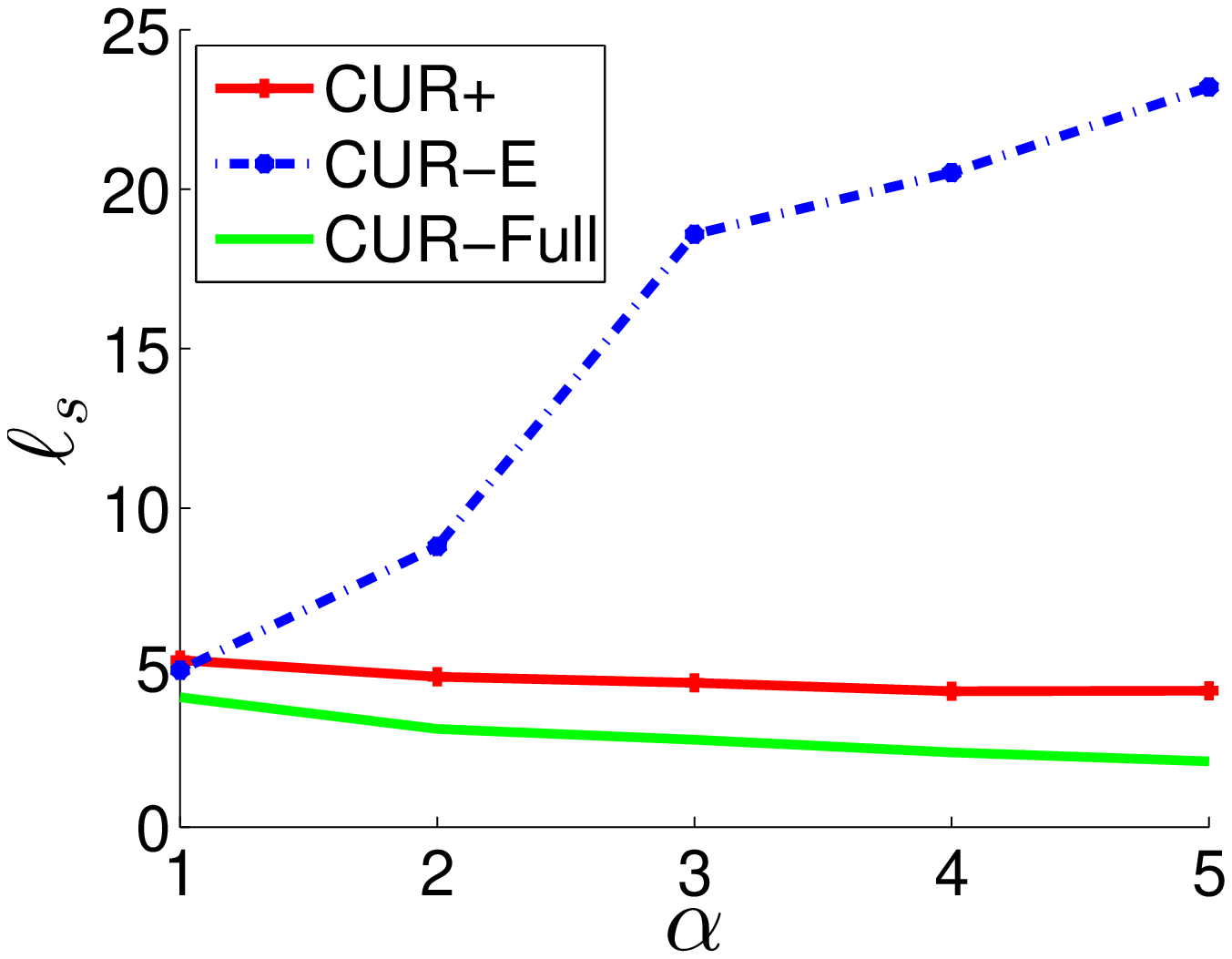}\\
\mbox{Gisette $r=10$}
\end{minipage}

\begin{minipage}[h]{1.3in}
\centering
\includegraphics[width= 1.3in]{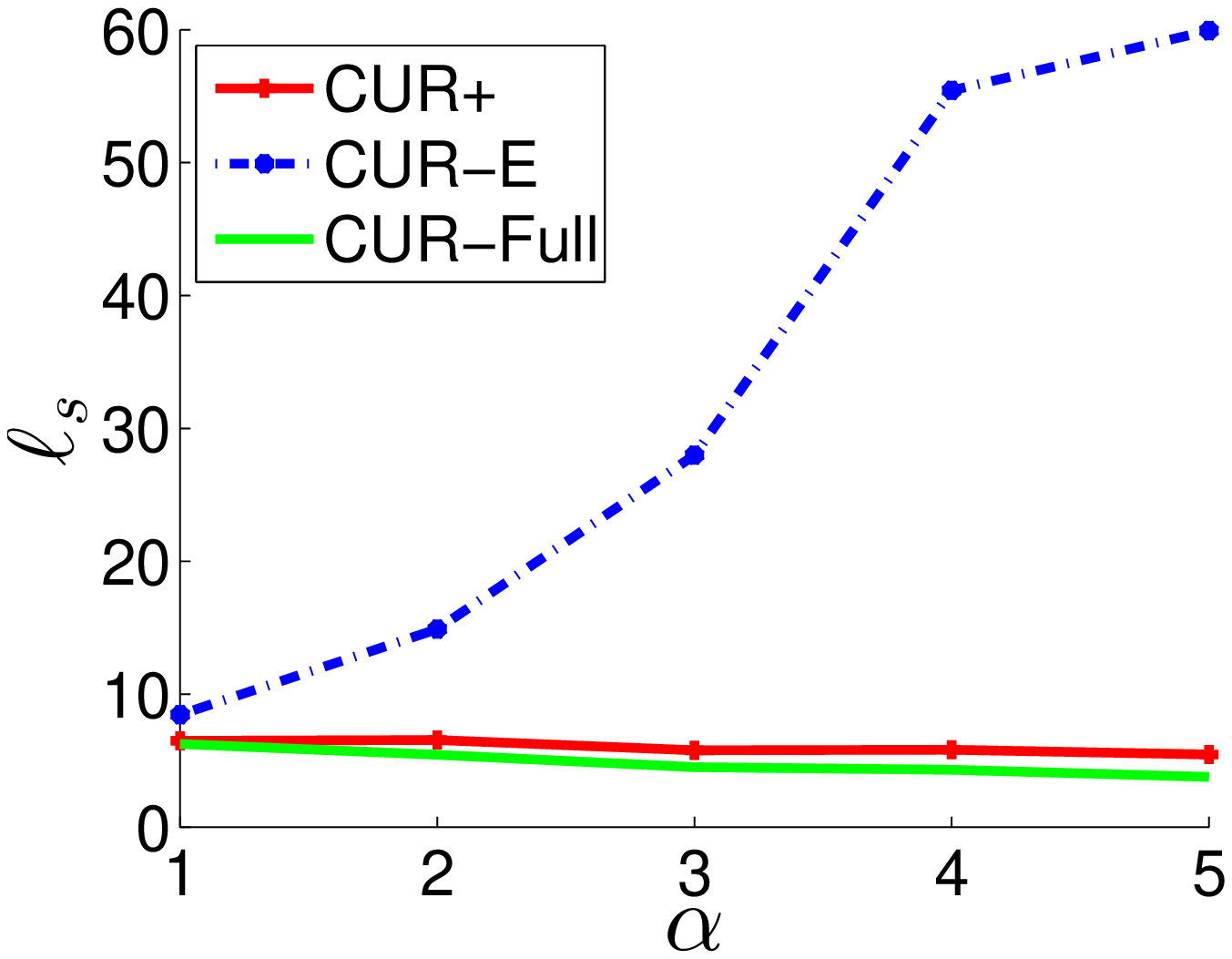}\\
\mbox{Enron $r=20$}
\end{minipage}
\begin{minipage}[h]{1.3in}
\centering
\includegraphics[width= 1.3in]{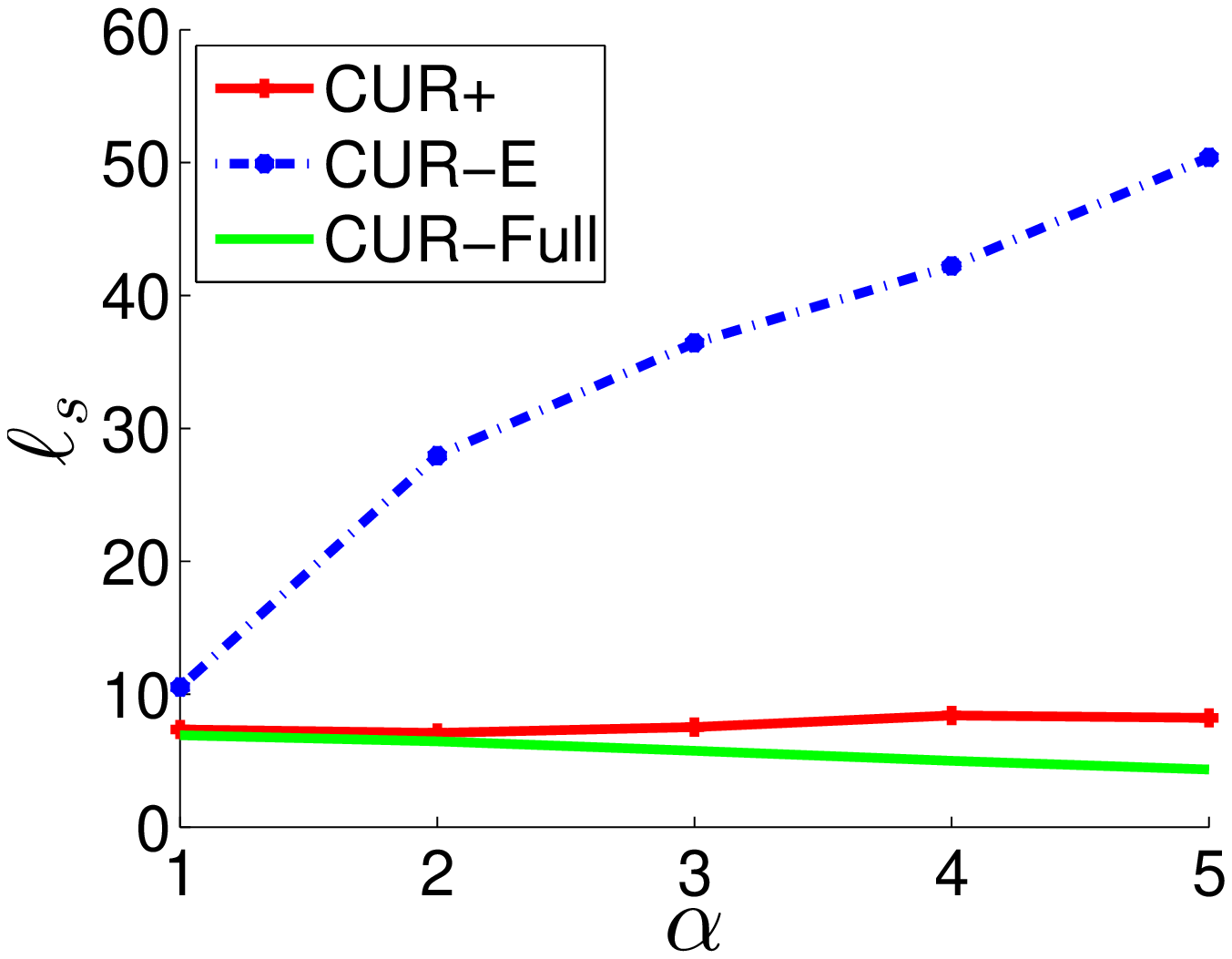}\\
\mbox{Dexter $r=20$}
\end{minipage}
\begin{minipage}[h]{1.3in}
\centering
\includegraphics[width= 1.3in]{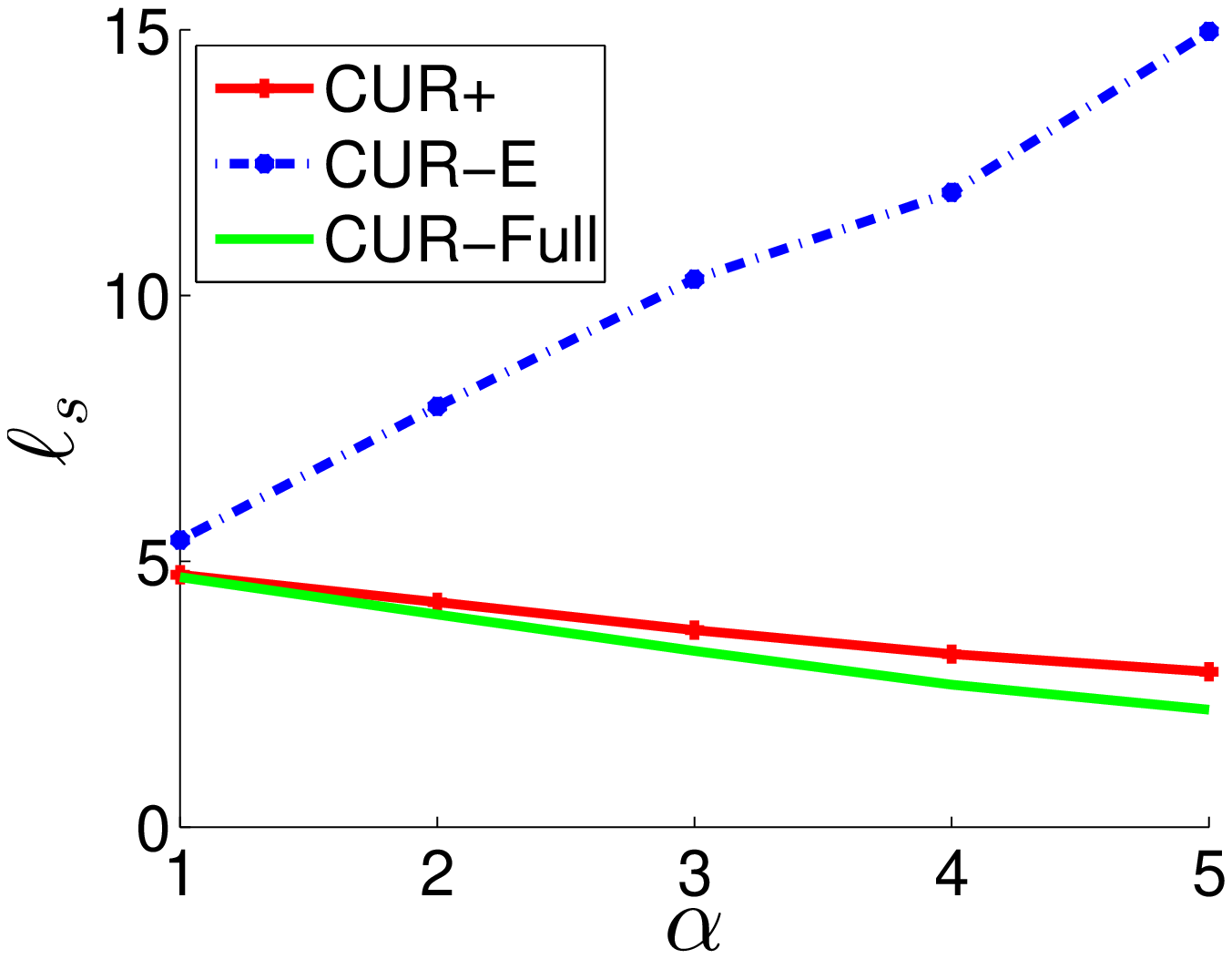}\\
\mbox{Farm Ads $r=20$}
\end{minipage}
\begin{minipage}[h]{1.3in}
\centering
\includegraphics[width= 1.3in]{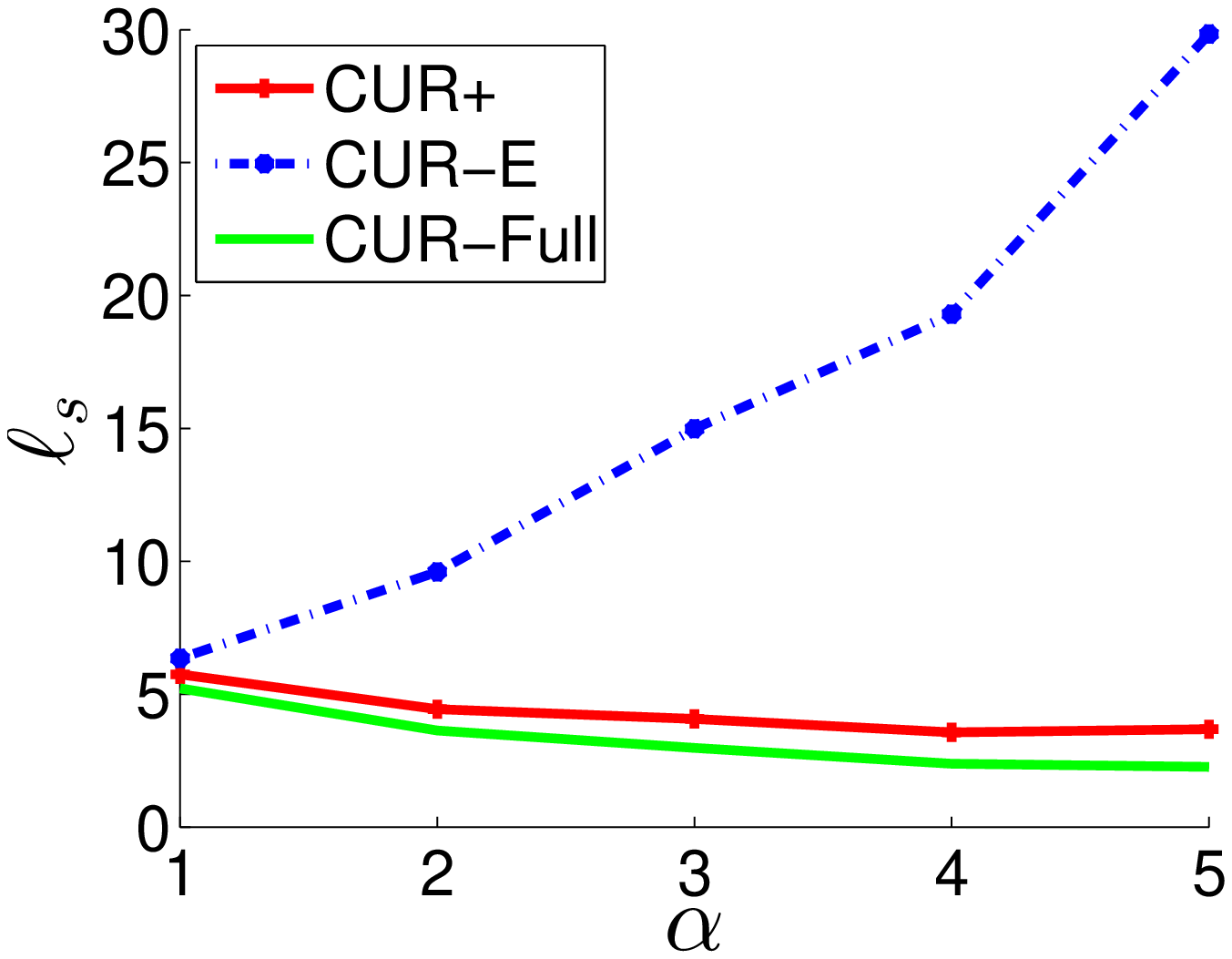}\\
\mbox{Gisette $r=20$}
\end{minipage}

\begin{minipage}[h]{1.3in}
\centering
\includegraphics[width= 1.3in]{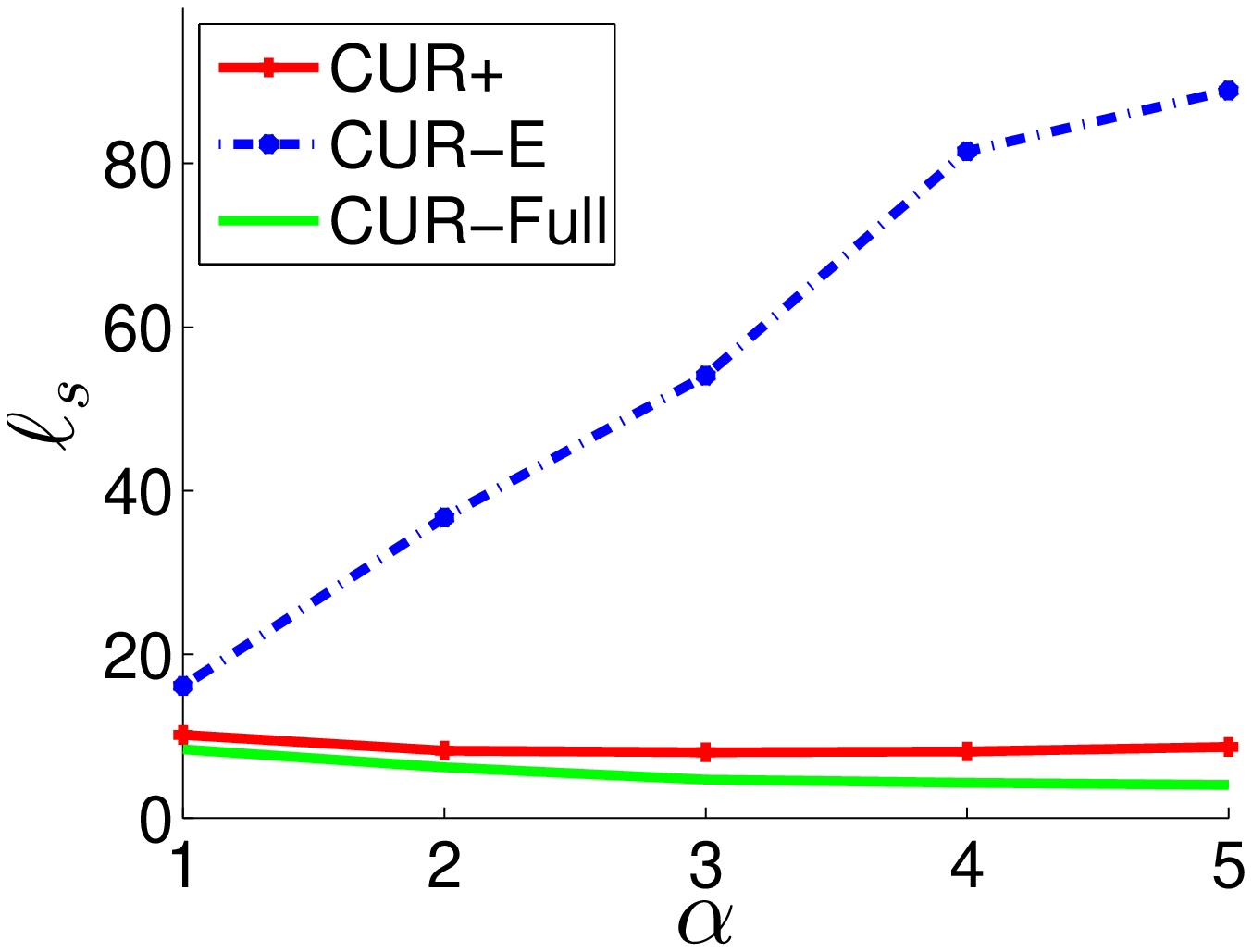}\\
\mbox{Enron $r=50$}
\end{minipage}
\begin{minipage}[h]{1.3in}
\centering
\includegraphics[width= 1.3in]{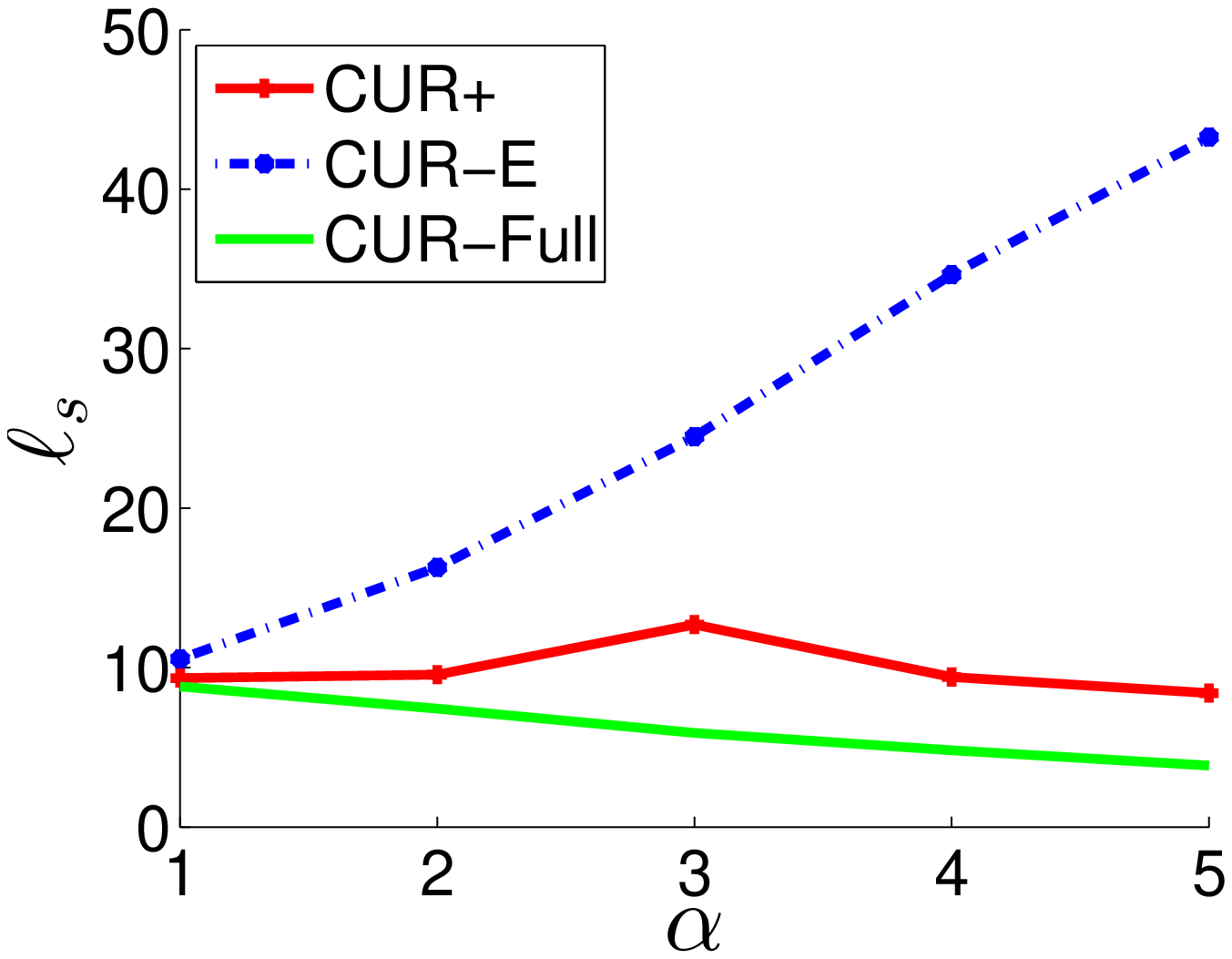}\\
\mbox{Dexter $r=50$}
\end{minipage}
\begin{minipage}[h]{1.3in}
\centering
\includegraphics[width= 1.3in]{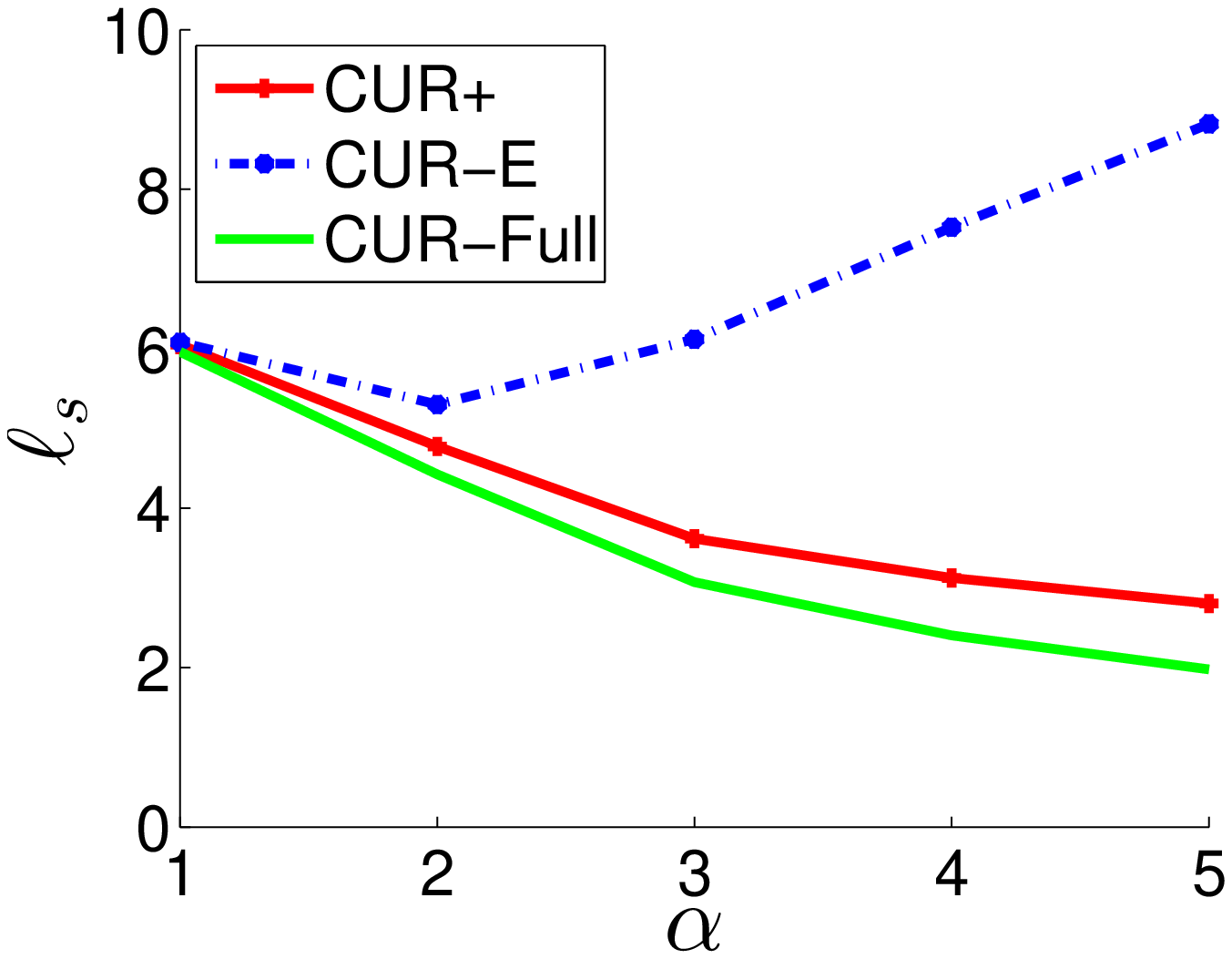}\\
\mbox{Farm Ads $r=50$}
\end{minipage}
\begin{minipage}[h]{1.3in}
\centering
\includegraphics[width= 1.3in]{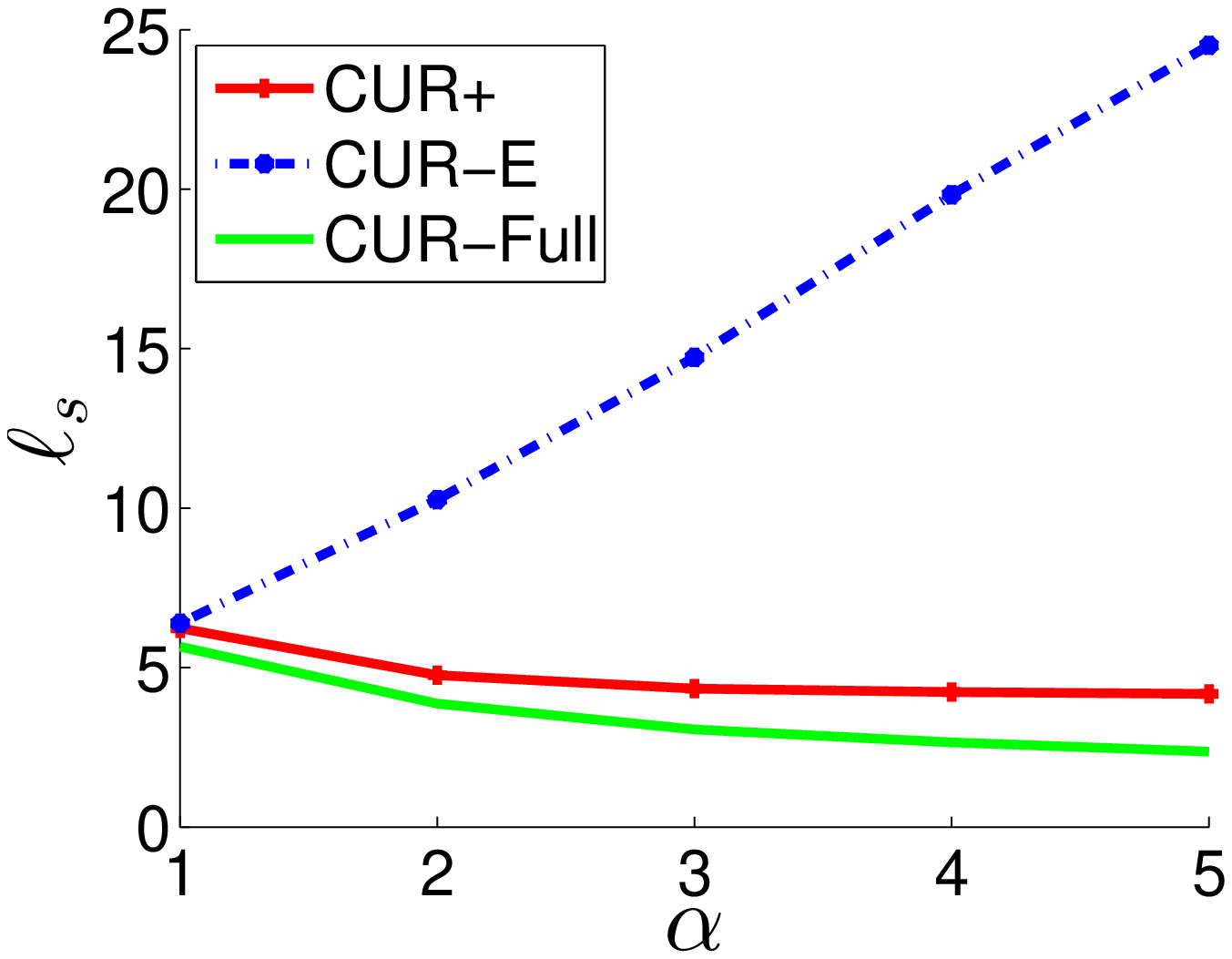}\\
\mbox{Gisette $r=50$}
\end{minipage}
%
\caption{Comparison of CUR algorithms with the number of observed entries $|\Omega|$ fixed as $|\Omega| = \Omega_0$. The number of sampled columns and rows are set as $d_1 = \alpha r$ and $d_2 = \alpha d_1$, respectively, where $r = 10,20,50$ and $\alpha$ is varied between $1$ and $5$.}\label{fig:vard}
\end{figure*}

We observe that with larger $\alpha$ (i.e. increasing numbers of rows and columns), the approximation errors for CUR$+$ and CUR-F decrease while, to our surprise, the error of CUR-E increases significantly. This counter-intuitive result can be explained by the fact that CUR-E estimates matrix $Z$ based on the observed entries in $\Omega$. Since the size of $Z$ is $d_1\times d_2$, which increases at the rate of $\alpha^3$. But on the other hand, $|\Omega|$, the number of observed entries based on which $Z$ is estimated, remains unchanged. As a result, with increasing values of $\alpha$, it becomes more and more difficult to come up with an accurate estimation of $Z$ and consequentially a worse and worse approximation of $M$. We have verified this explanation in Fig~\ref{fig:triple} by simultaneously increasing the number of observed entries in $\Omega$ and observing that the approximation error of CUR-E decreases with increasing $\alpha$, although with perturbation. It is also to our surprise that when $r$ is increasing, the relative spectral-norm difference $\ell_s$ is increasing. This may due to the fact that we normalize the spectral-norm, dividing it by $\|M-M_r\|$ which decreases fast. And we observe that $\|M-\hat M\|$ decreases when $r$ becomes larger and larger.

\begin{figure*}[!t]
\centering
\begin{minipage}[h]{1.3in}
\centering
\includegraphics[width= 1.3in]{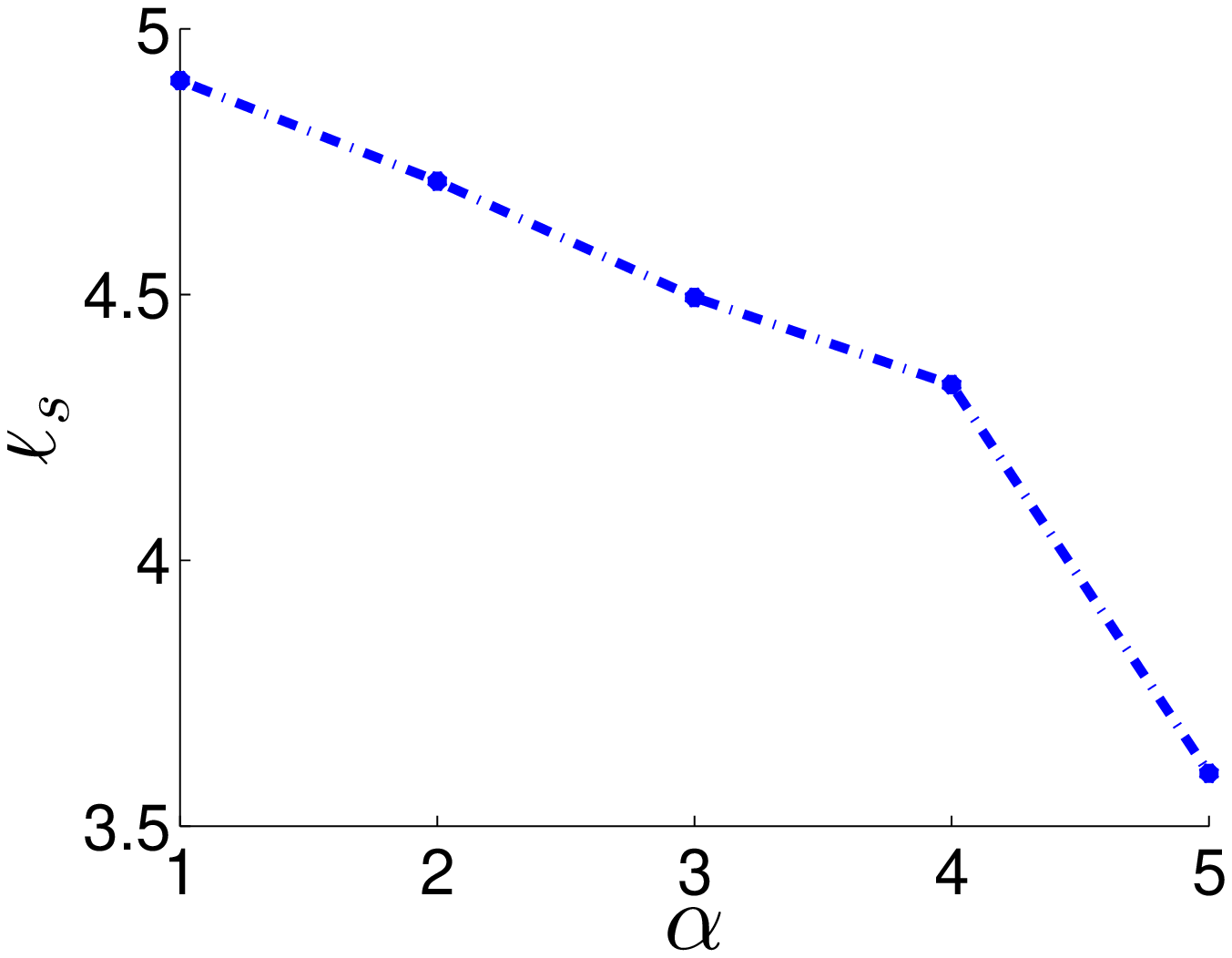}\\
\mbox{Enron $r=10$}
\end{minipage}
\begin{minipage}[h]{1.3in}
\centering
\includegraphics[width= 1.3in]{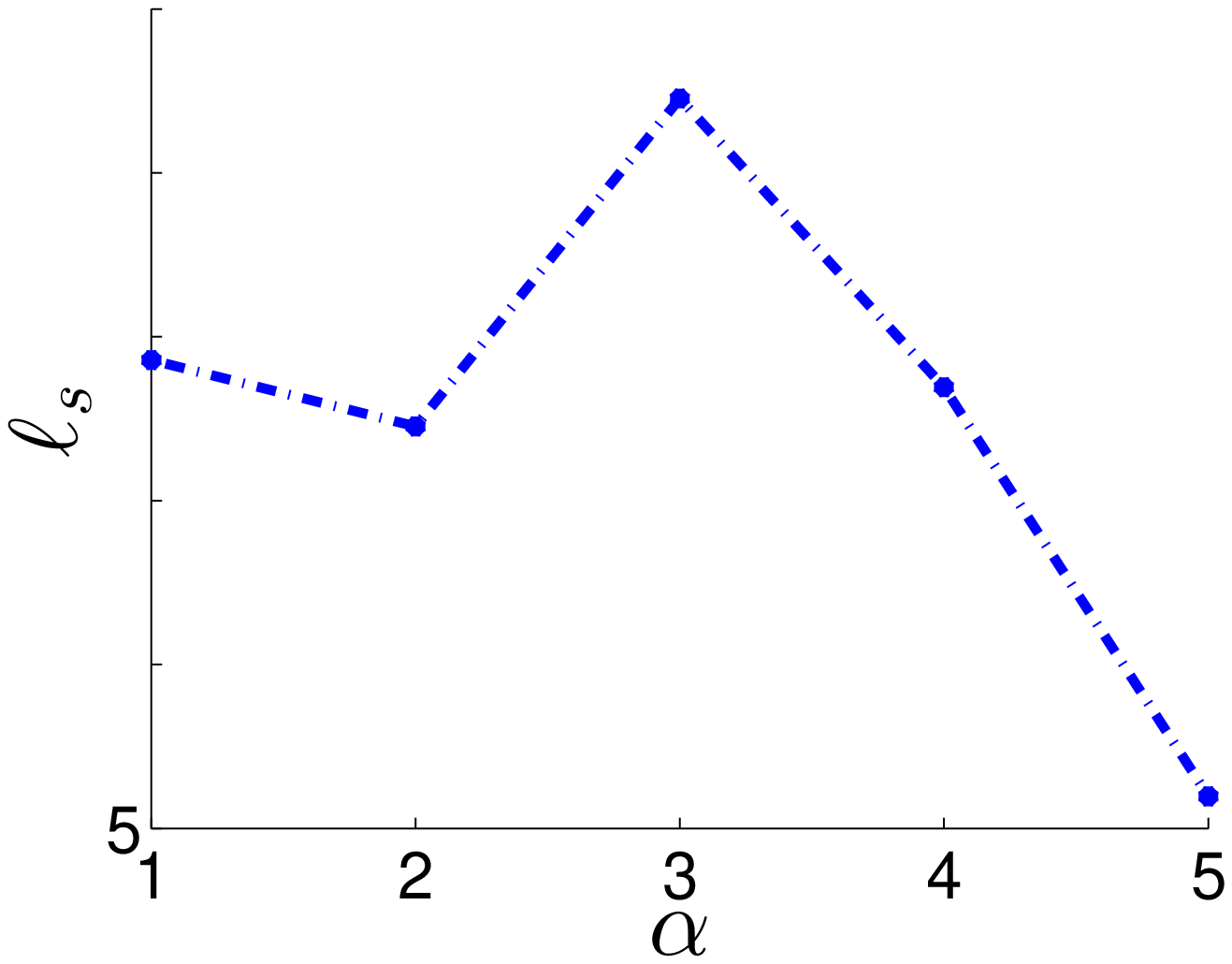}\\
\mbox{Dexter $r=10$}
\end{minipage}
\begin{minipage}[h]{1.3in}
\centering
\includegraphics[width= 1.3in]{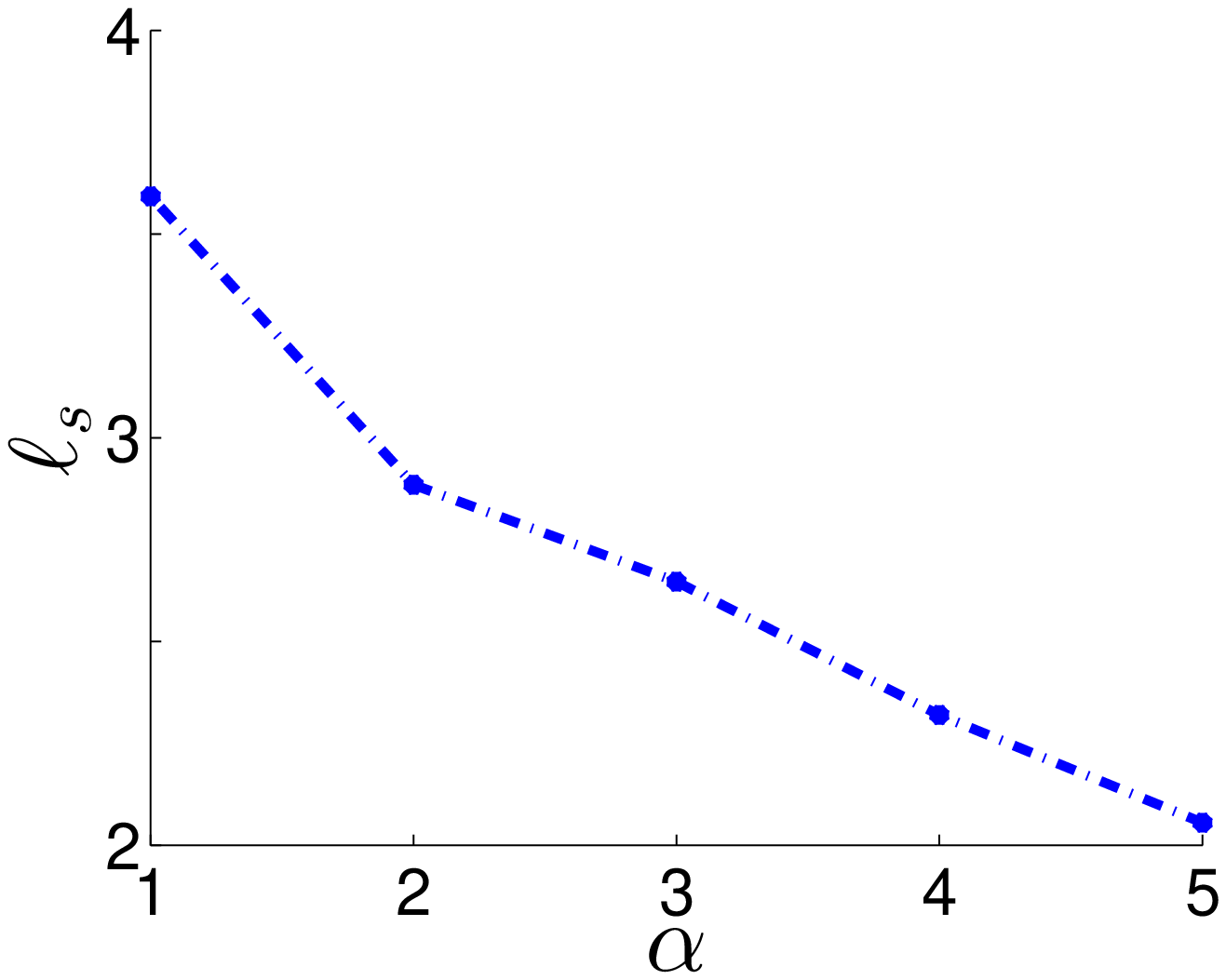}\\
\mbox{Farm Ads $r=10$}
\end{minipage}
\begin{minipage}[h]{1.3in}
\centering
\includegraphics[width= 1.3in]{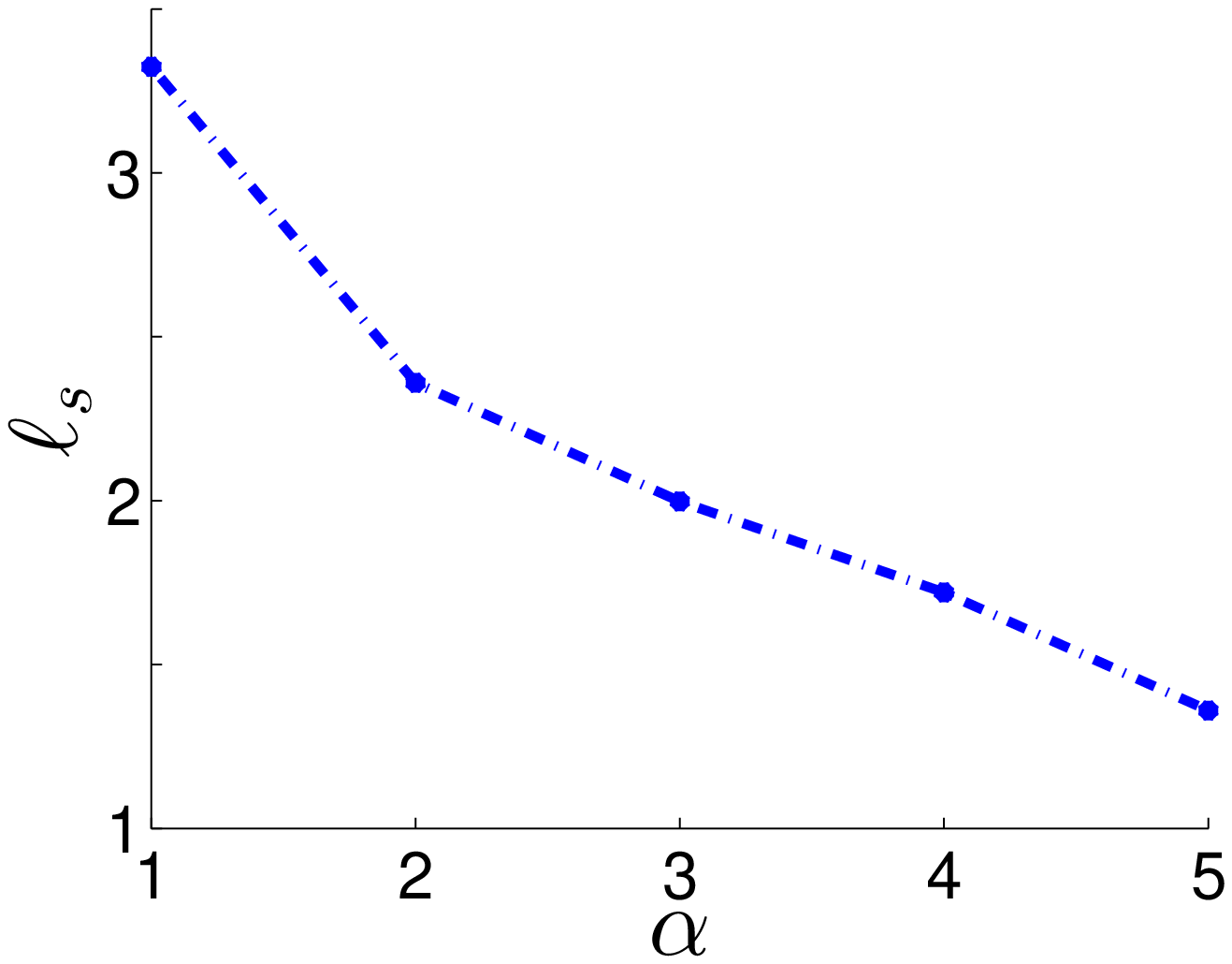}\\
\mbox{Gisette $r=10$}
\end{minipage}

\begin{minipage}[h]{1.3in}
\centering
\includegraphics[width= 1.3in]{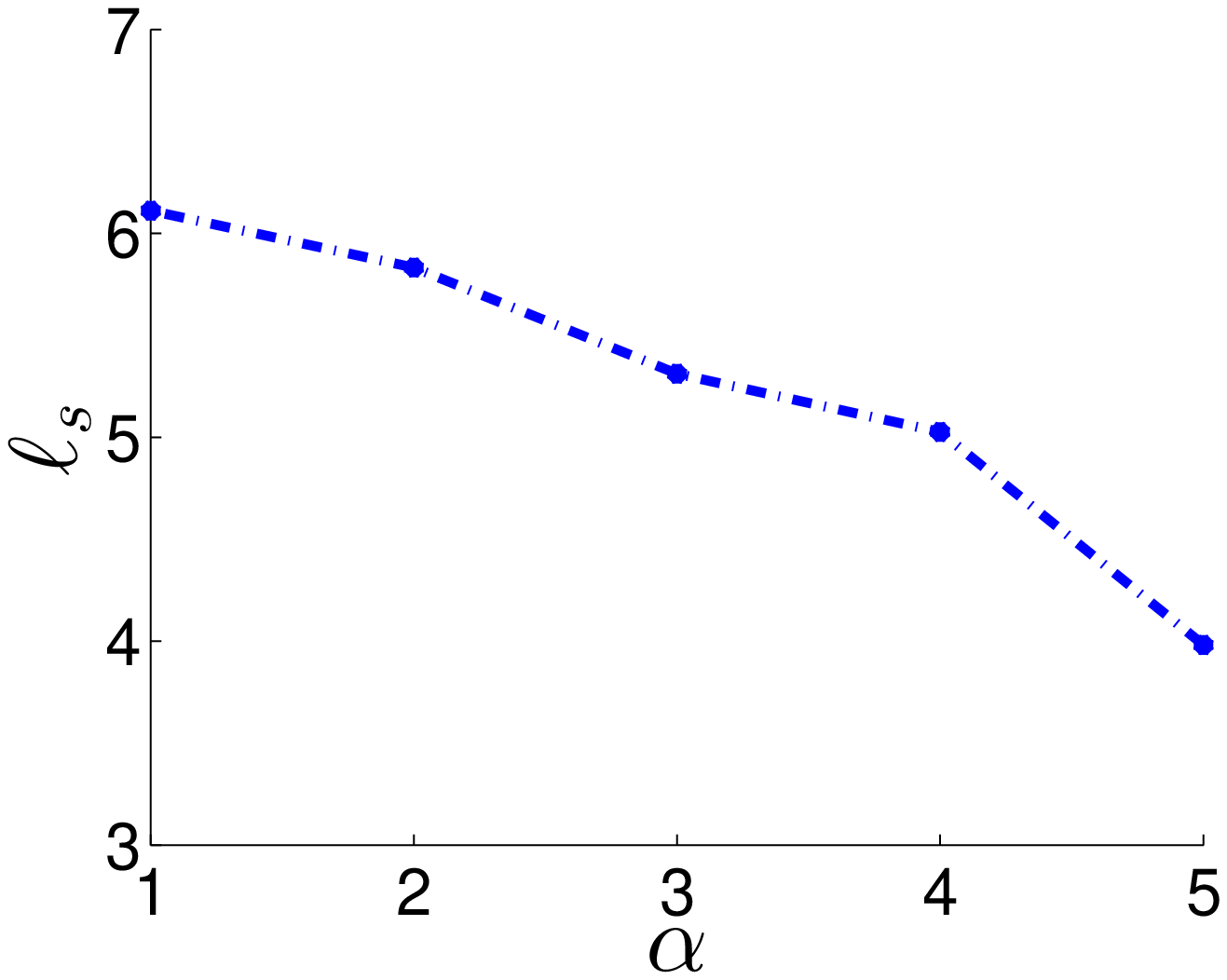}\\
\mbox{Enron $r=20$}
\end{minipage}
\begin{minipage}[h]{1.3in}
\centering
\includegraphics[width= 1.3in]{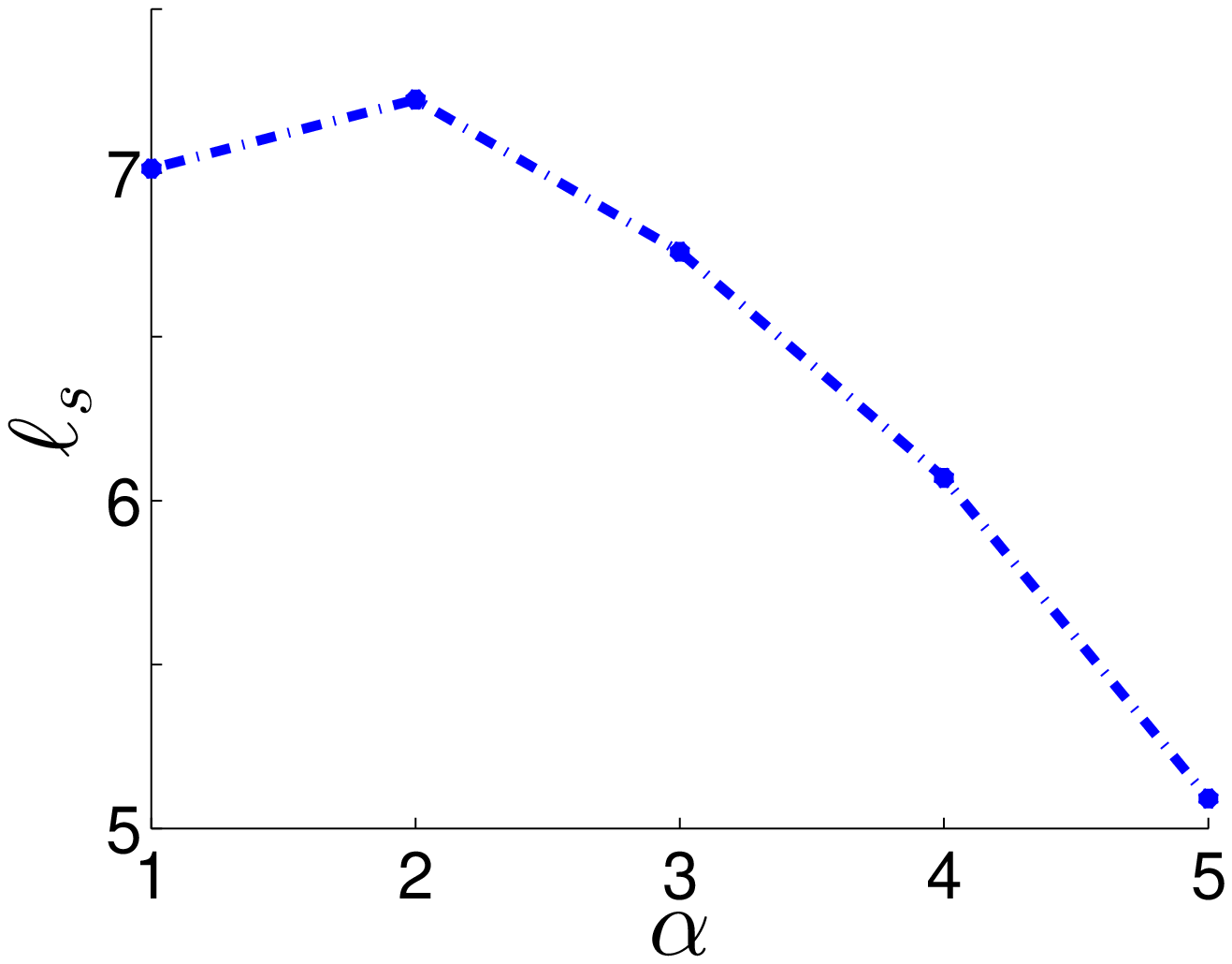}\\
\mbox{Dexter $r=20$}
\end{minipage}
\begin{minipage}[h]{1.3in}
\centering
\includegraphics[width= 1.3in]{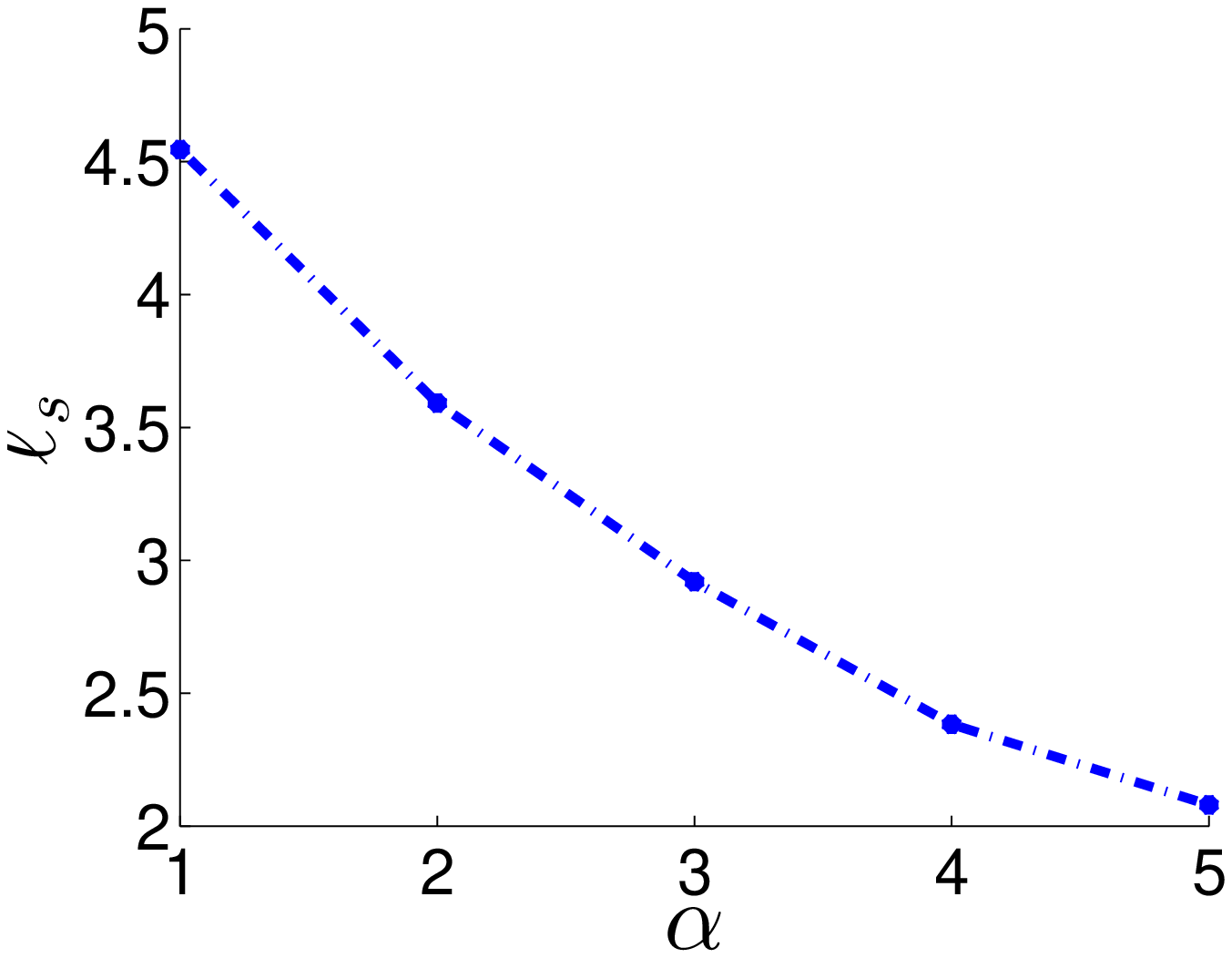}\\
\mbox{Farm Ads $r=20$}
\end{minipage}
\begin{minipage}[h]{1.3in}
\centering
\includegraphics[width= 1.3in]{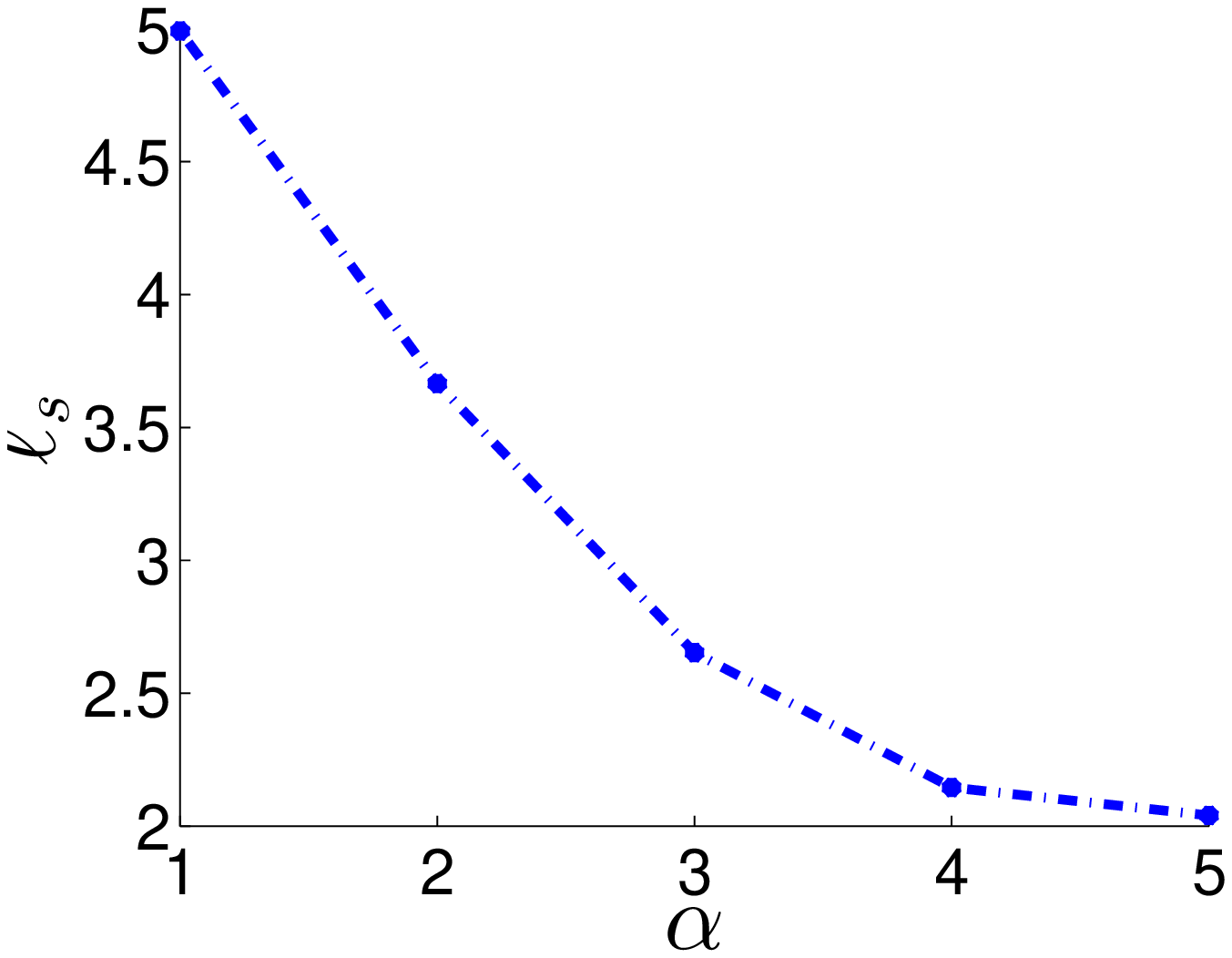}\\
\mbox{Gisette $r=20$}
\end{minipage}

\begin{minipage}[h]{1.3in}
\centering
\includegraphics[width= 1.3in]{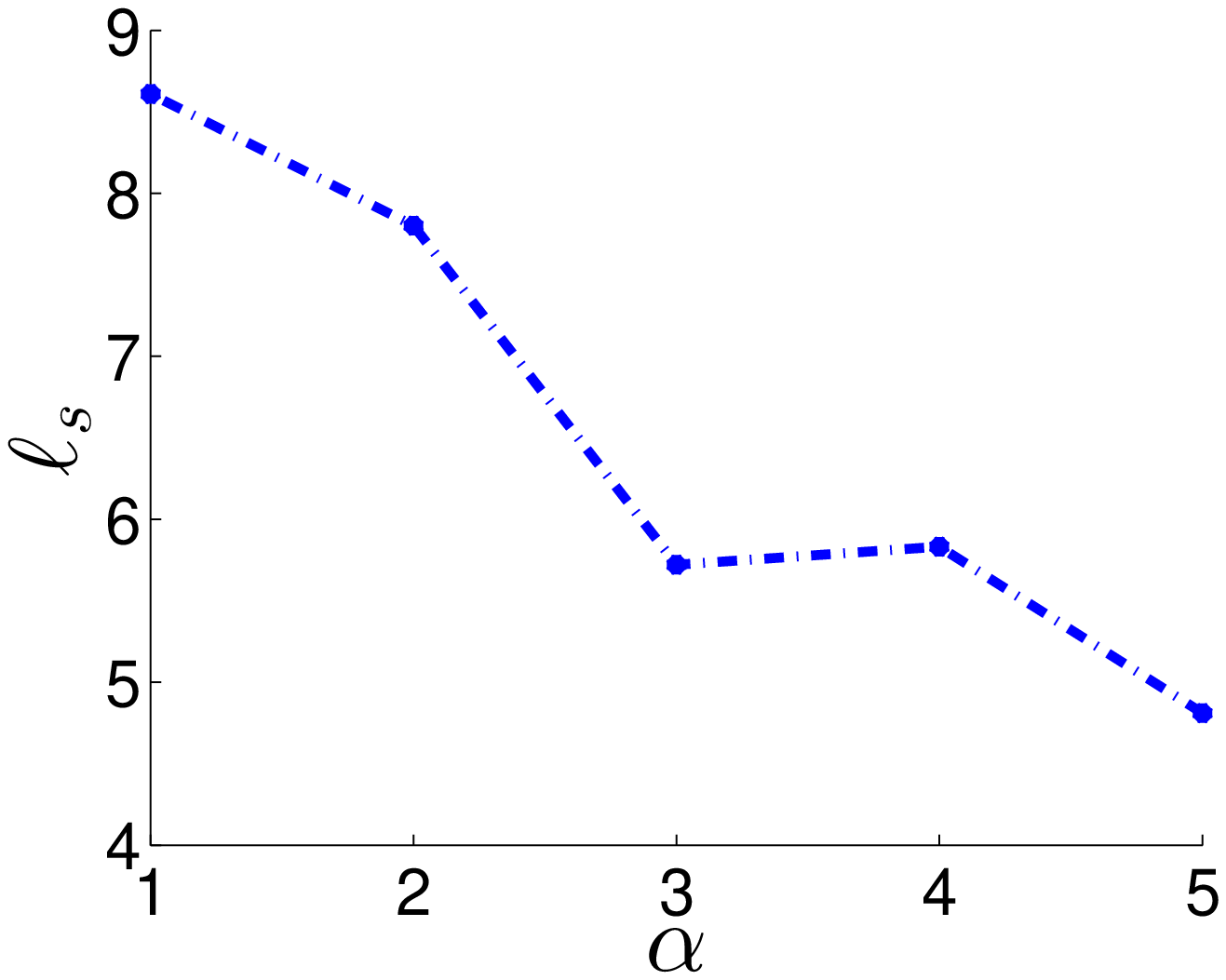}\\
\mbox{Enron $r=50$}
\end{minipage}
\begin{minipage}[h]{1.3in}
\centering
\includegraphics[width= 1.3in]{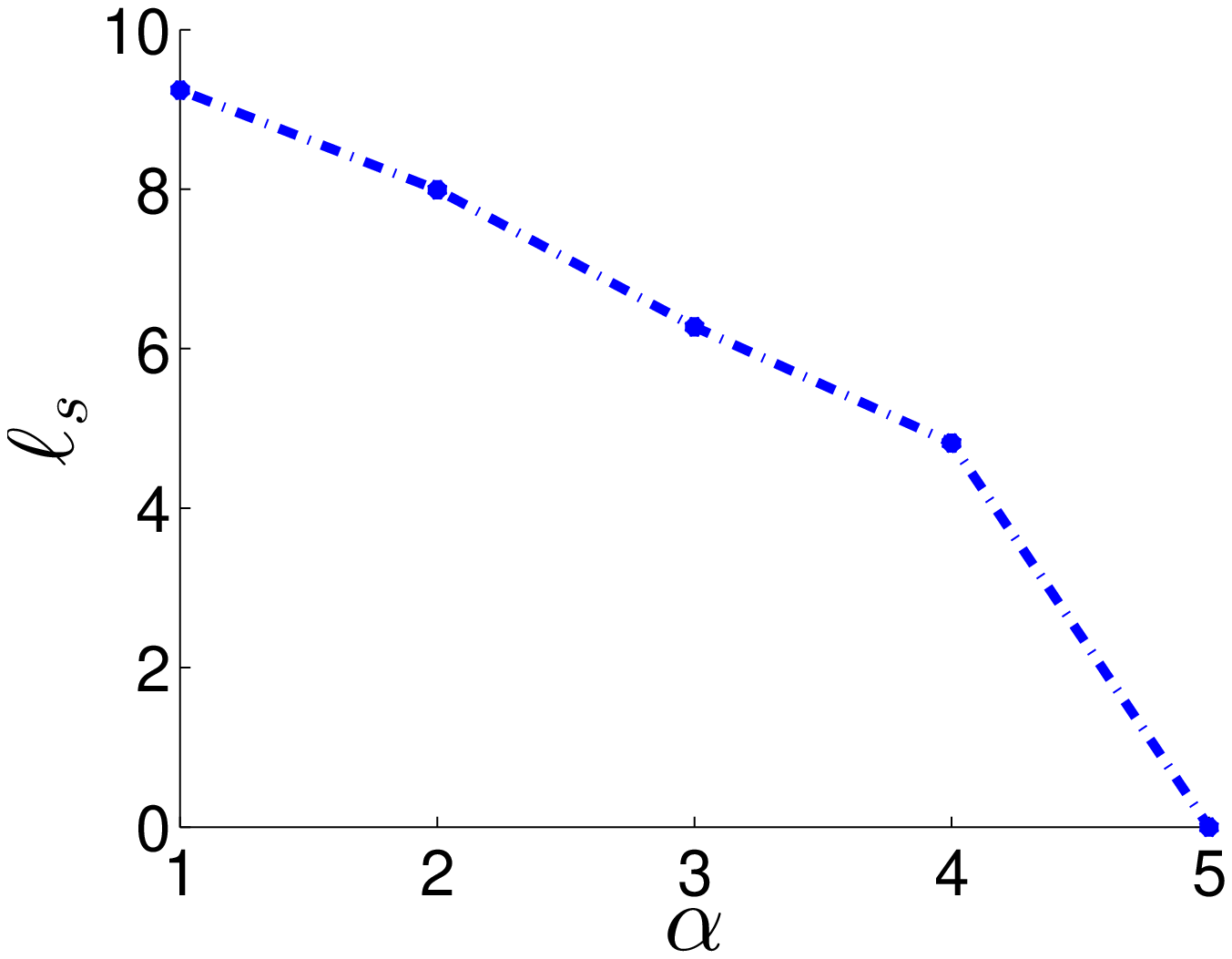}\\
\mbox{Dexter $r=50$}
\end{minipage}
\begin{minipage}[h]{1.3in}
\centering
\includegraphics[width= 1.3in]{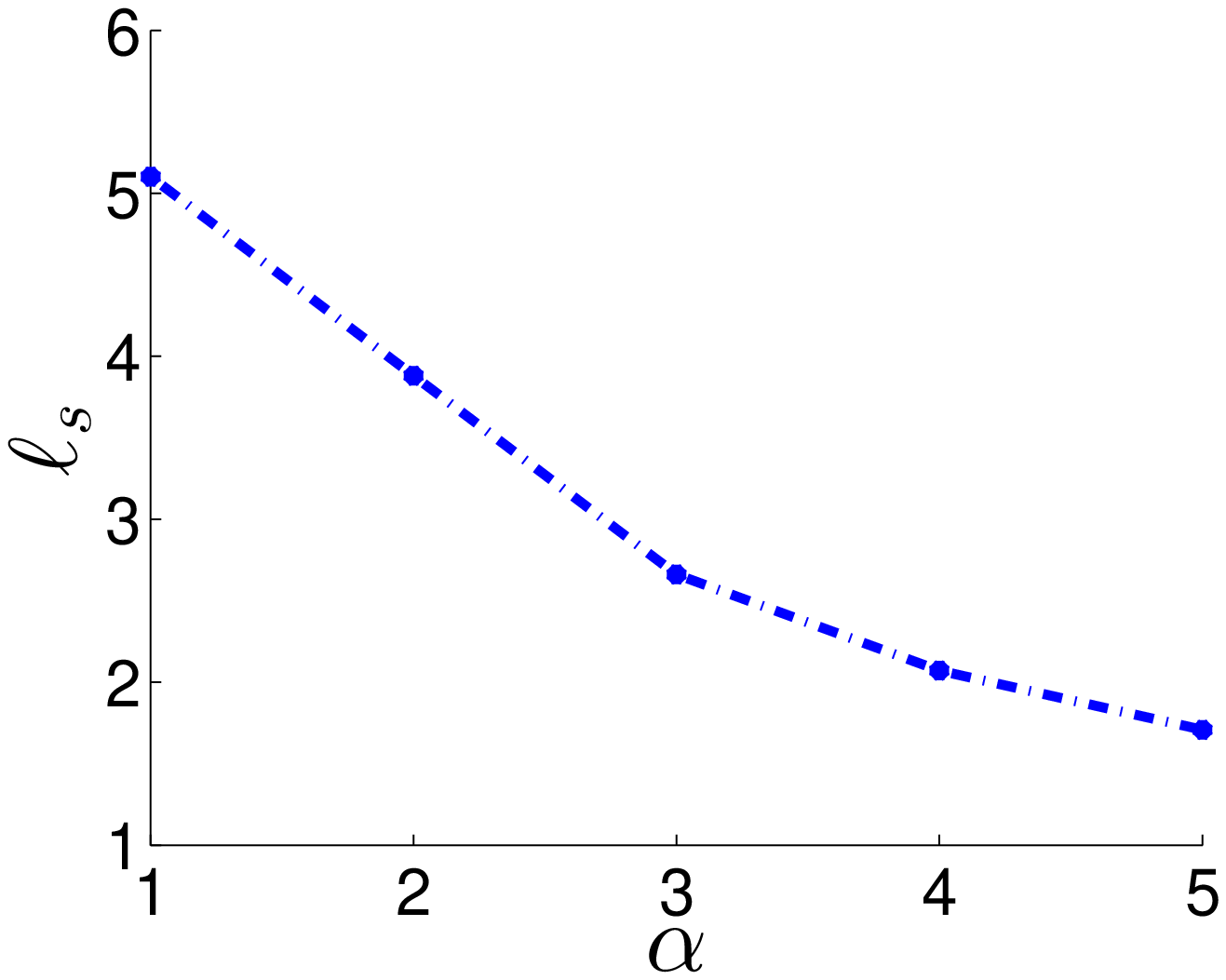}\\
\mbox{Farm Ads $r=50$}
\end{minipage}
\begin{minipage}[h]{1.3in}
\centering
\includegraphics[width= 1.3in]{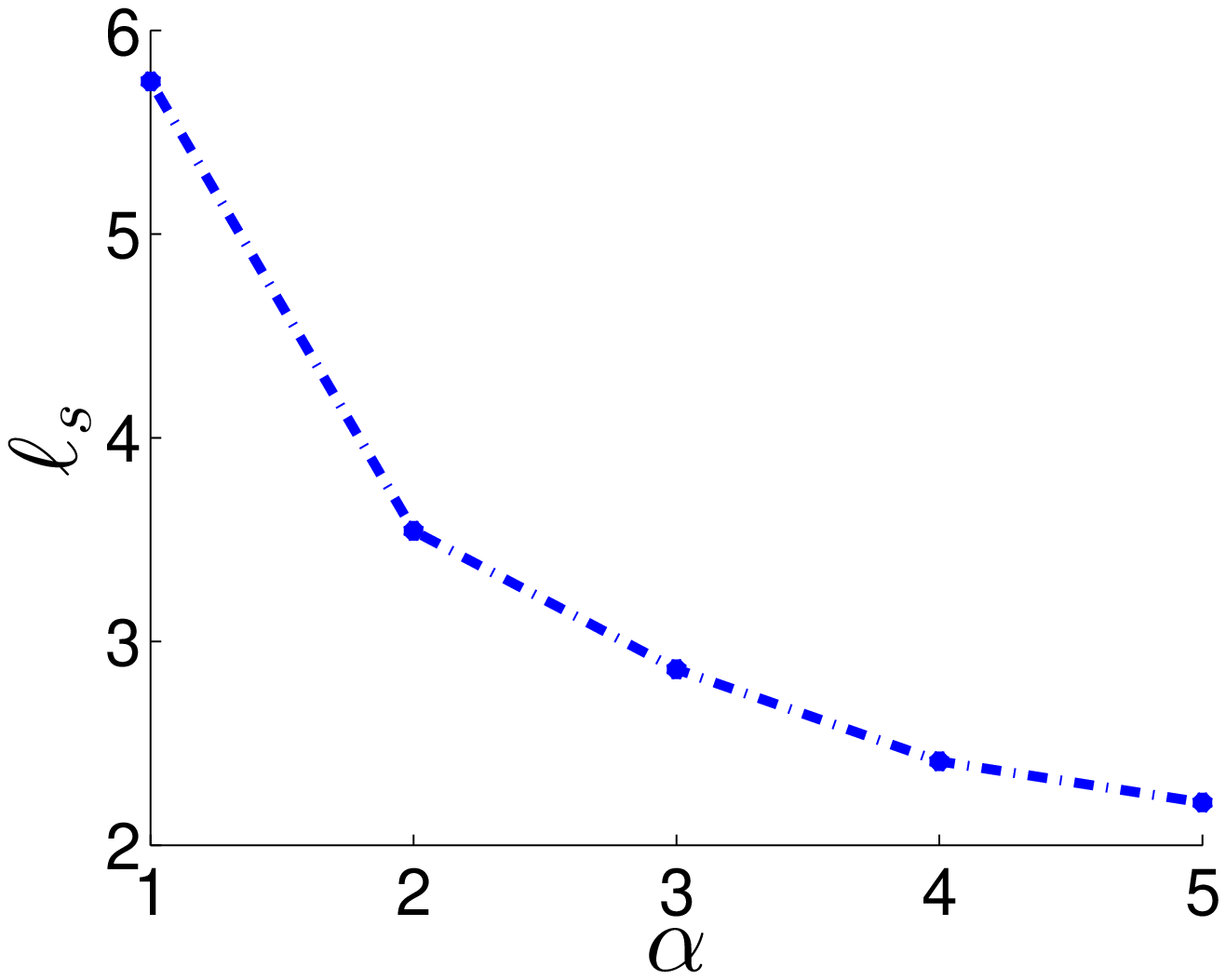}\\
\mbox{Gisette $r=50$}
\end{minipage}
\caption{The results of CUR-E algorithm when $|\Omega|$ increases with $\alpha^3$ for $r = 10$, $20$ and $50$.}\label{fig:triple}
\end{figure*}

In the second experiment, we fix the number of sampled rows and columns and vary the number of observed entries from $\Omega_0$ to $5\Omega_0$. Figure~\ref{fig:varo} shows the results for $r = 10, 20$ and $50$. Again, we found that CUR$+$ yields similar performance as CUR-F, and performs significantly better than CUR-E, although the gap between CUR$+$ and CUR-E does decline with increasing number of observed entries. It is also to our surprise that for datasets Enron and Farm Ads, the approximation error of CUR$+$ remains almost unchanged with increasing number of observed entries. We plan to examine this unusual phenomenon in the future.

\begin{figure*}[!t]
\centering
\begin{minipage}[h]{1.3in}
\centering
\includegraphics[width= 1.3in]{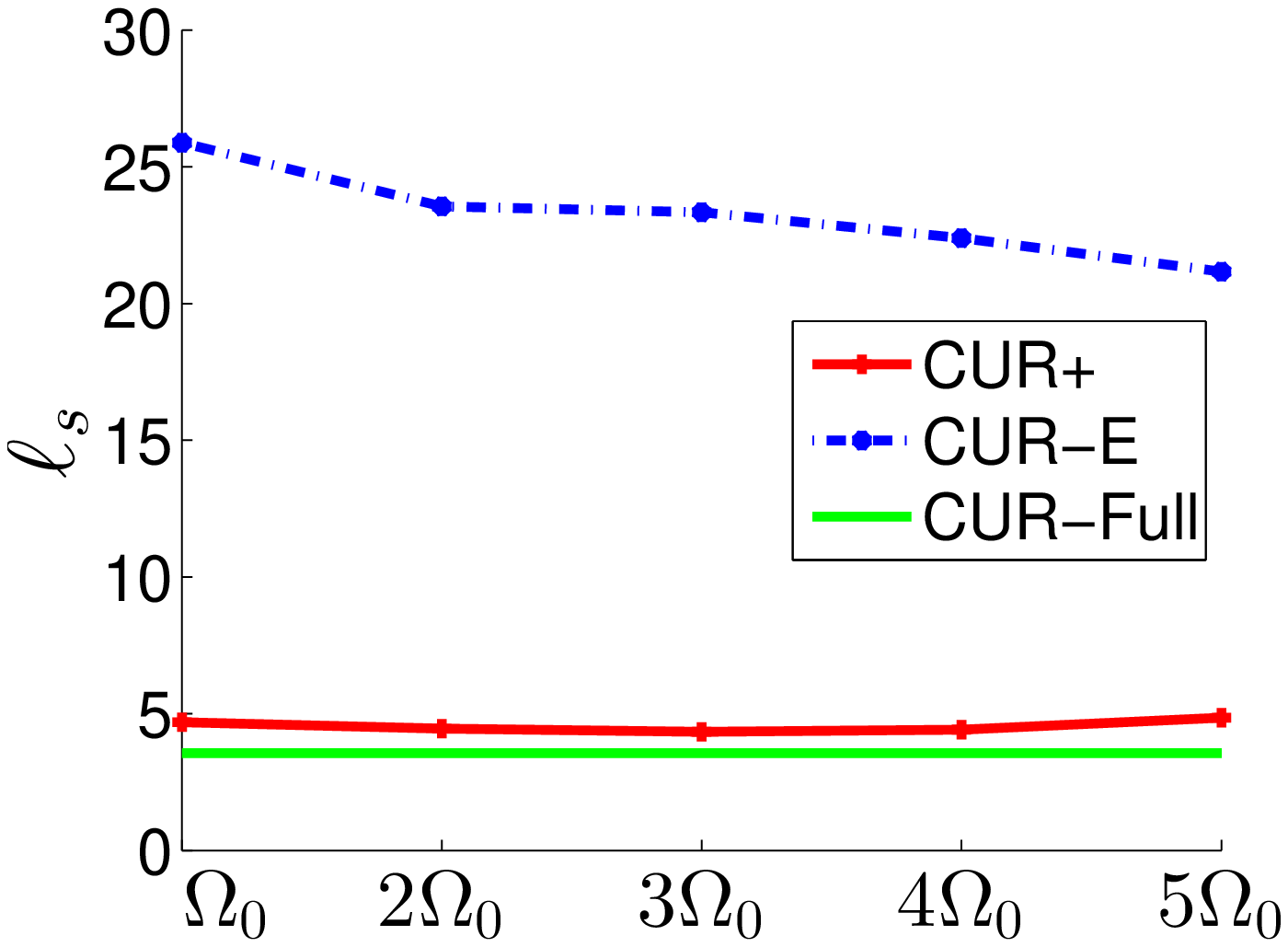}\\
\mbox{Enron $r=10$}
\end{minipage}
\begin{minipage}[h]{1.3in}
\centering
\includegraphics[width= 1.3in]{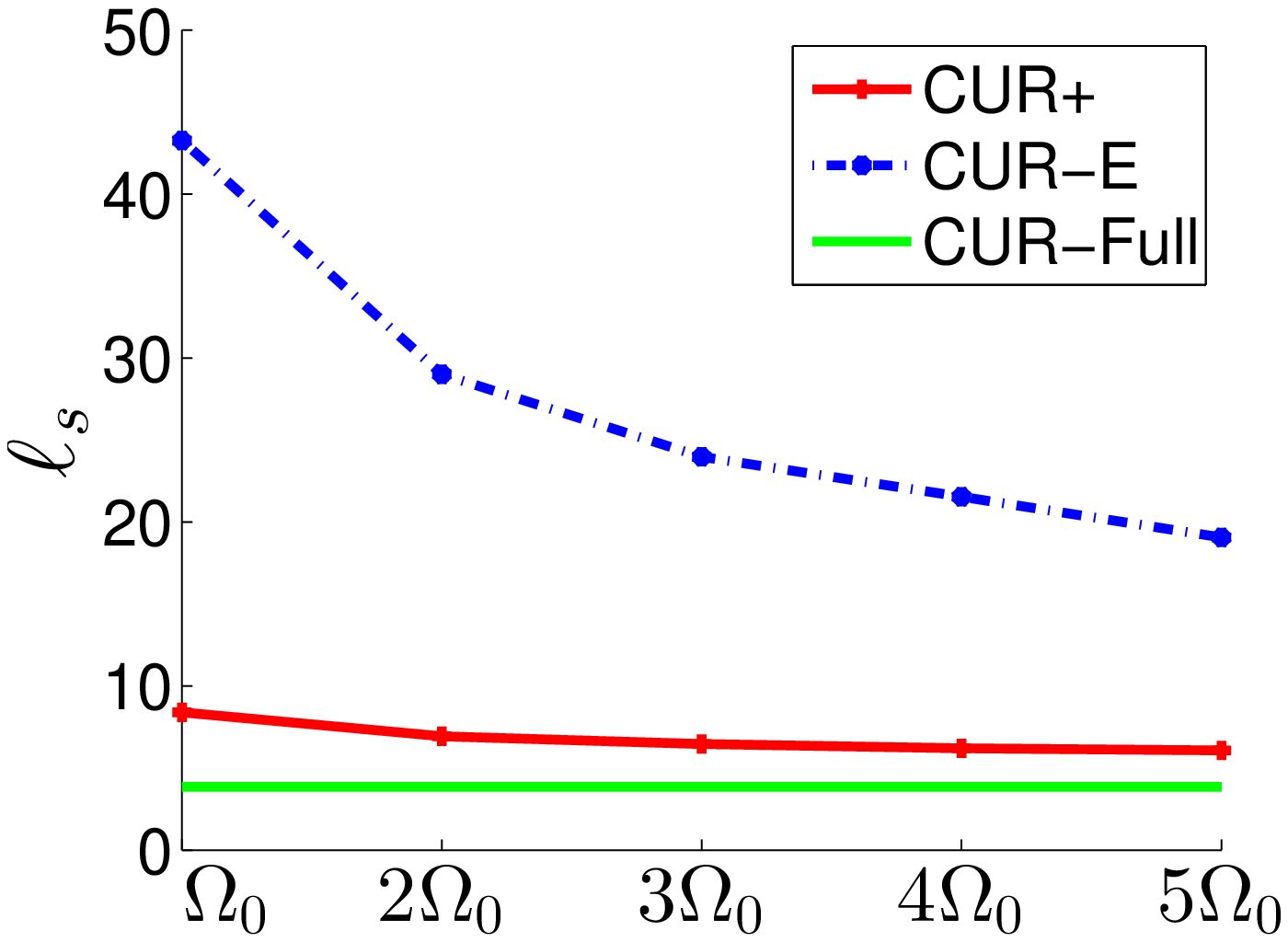}\\
\mbox{Dexter $r=10$}
\end{minipage}
\begin{minipage}[h]{1.3in}
\centering
\includegraphics[width= 1.3in]{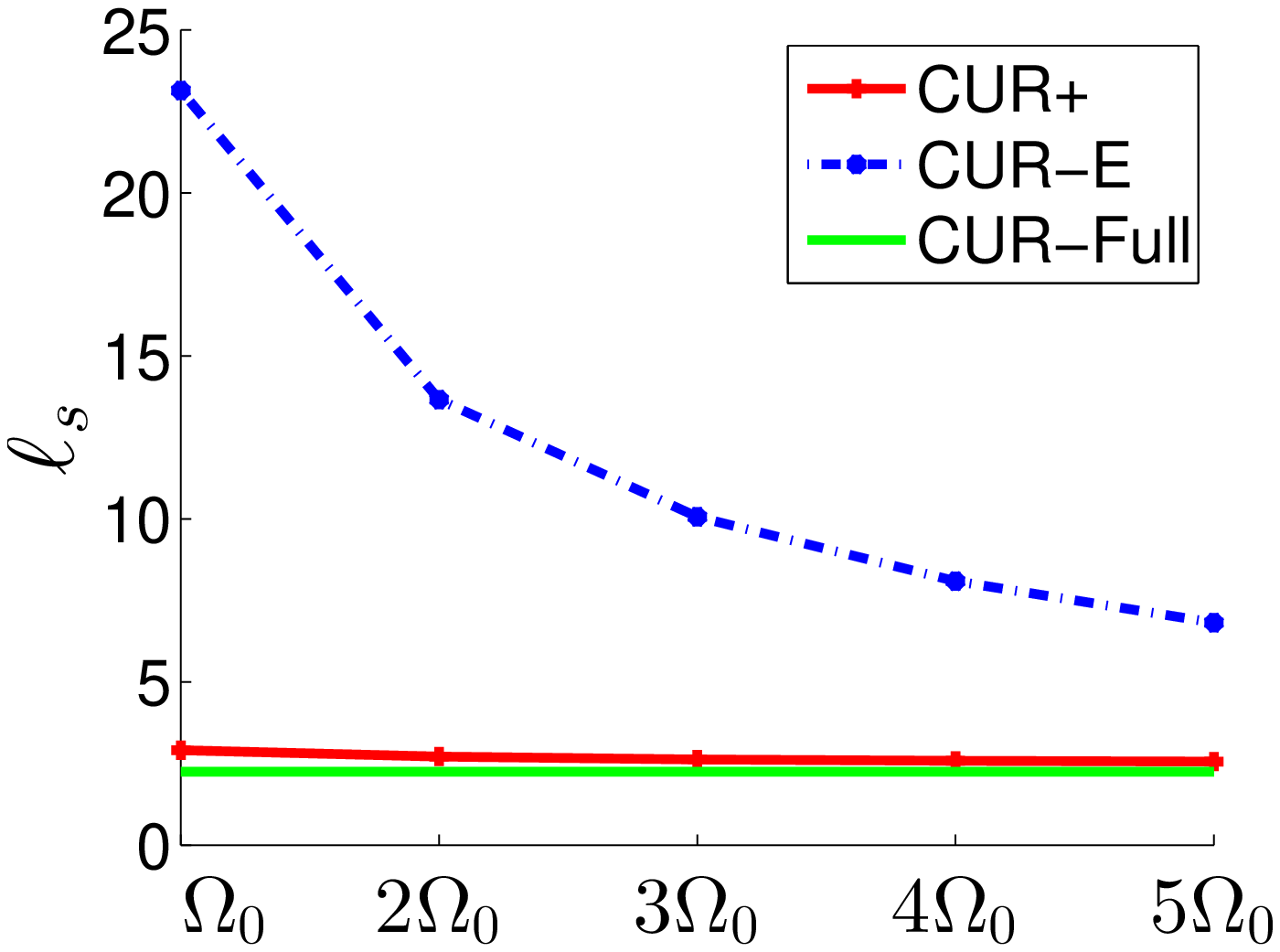}\\
\mbox{Farm Ads $r=10$}
\end{minipage}
\begin{minipage}[h]{1.3in}
\centering
\includegraphics[width= 1.3in]{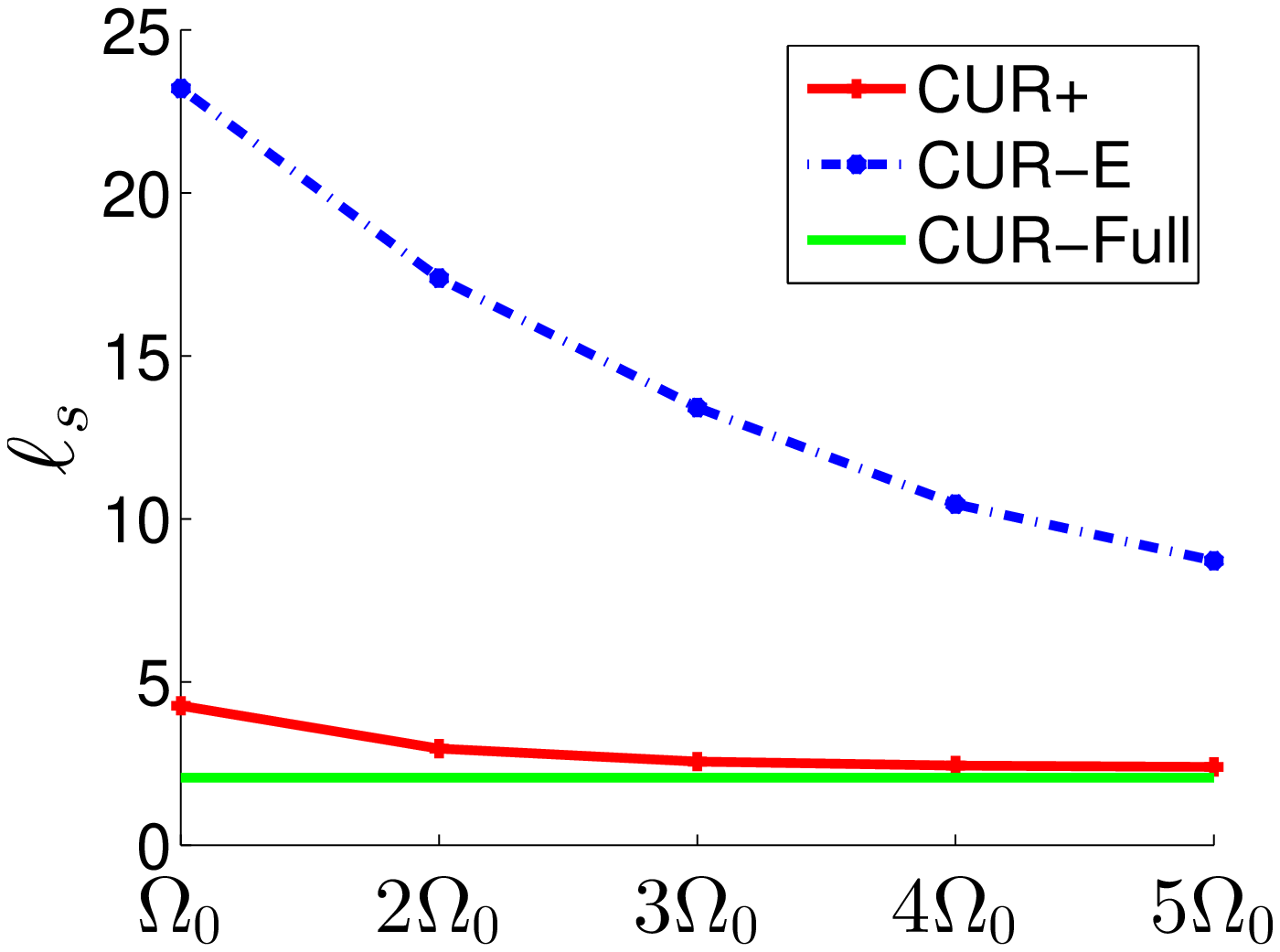}\\
\mbox{Gisette $r=10$}
\end{minipage}

\begin{minipage}[h]{1.3in}
\centering
\includegraphics[width= 1.3in]{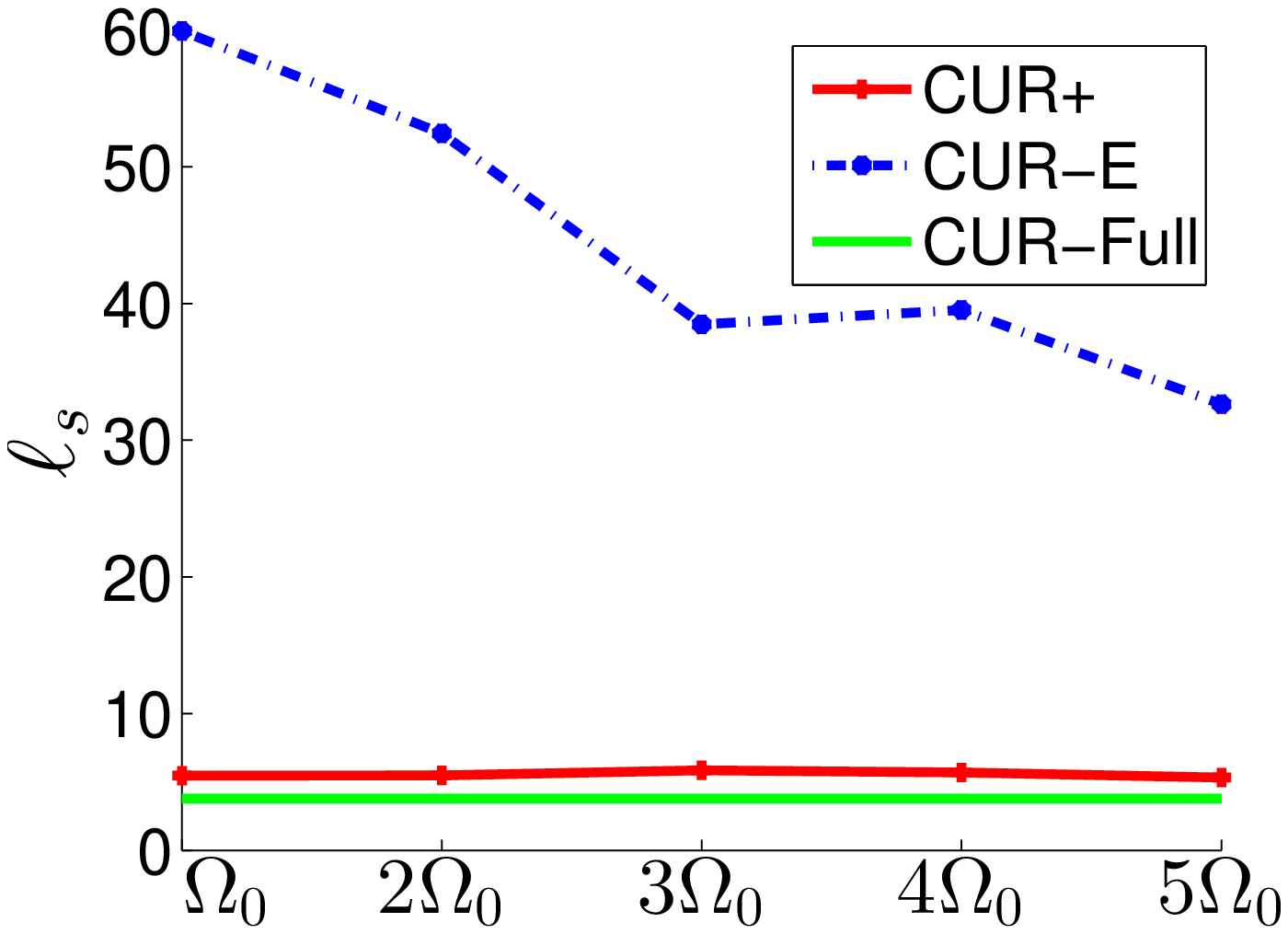}\\
\mbox{Enron $r=20$}
\end{minipage}
\begin{minipage}[h]{1.3in}
\centering
\includegraphics[width= 1.3in]{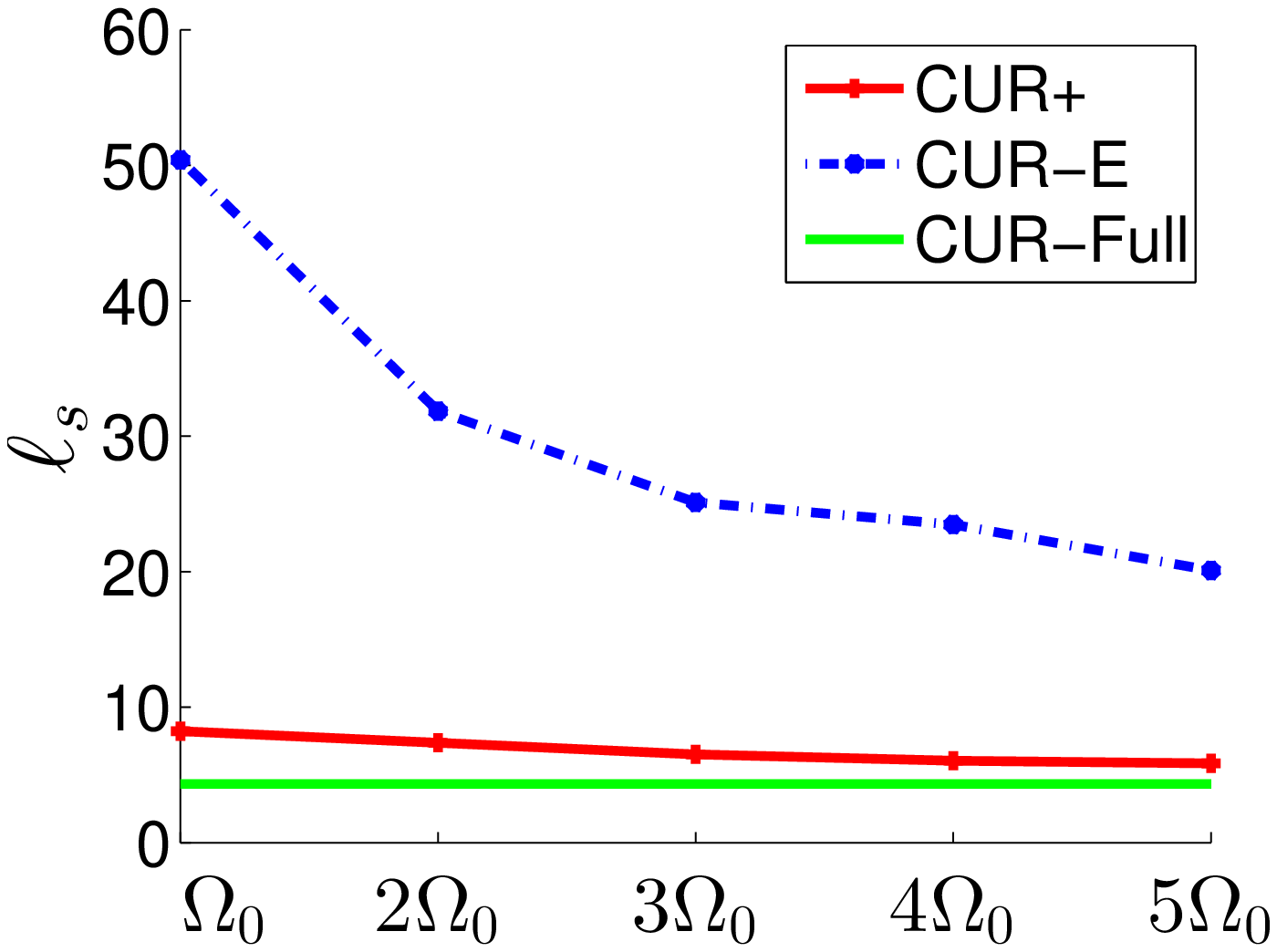}\\
\mbox{Dexter $r=20$}
\end{minipage}
\begin{minipage}[h]{1.3in}
\centering
\includegraphics[width= 1.3in]{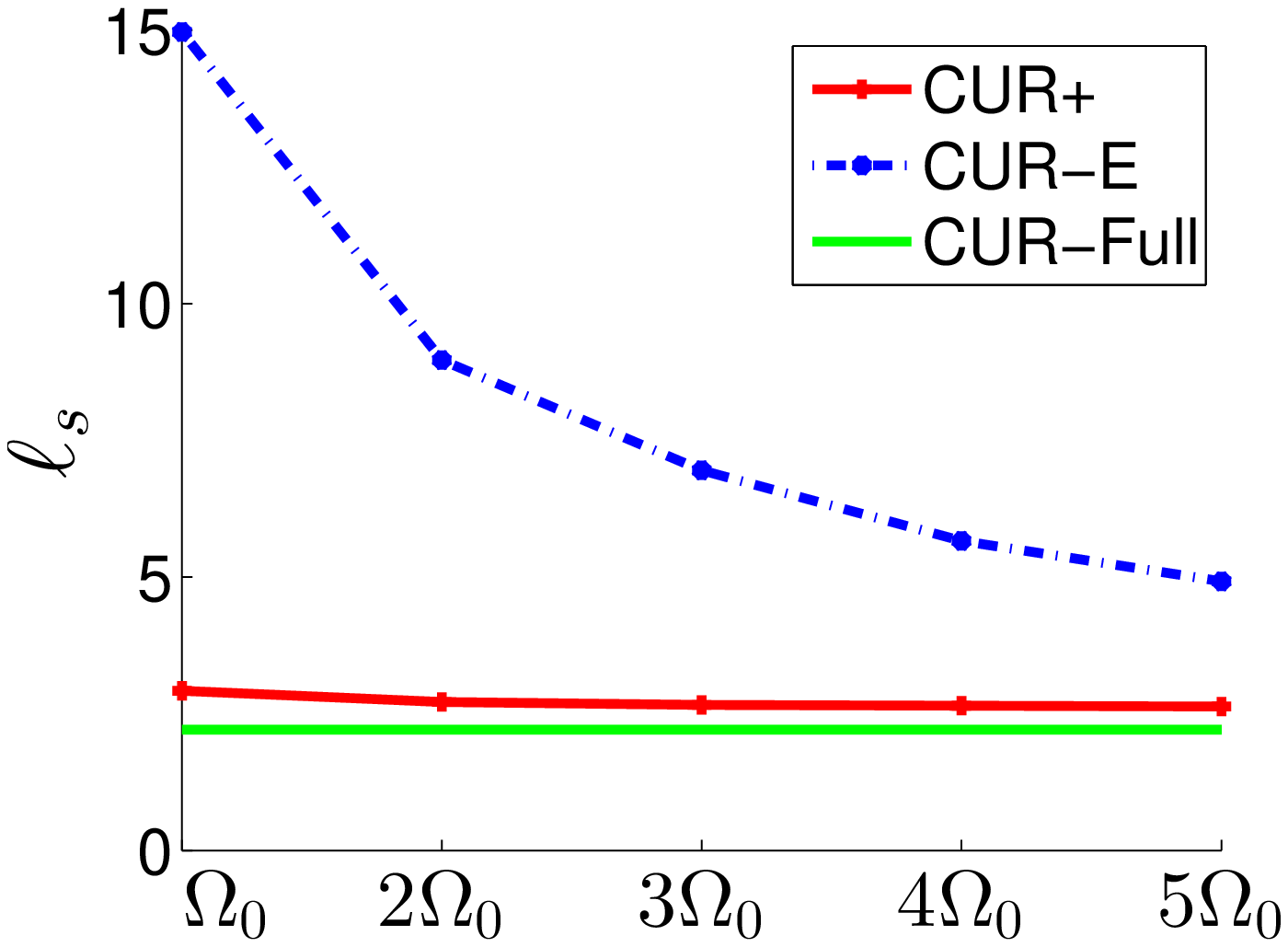}\\
\mbox{Farm Ads $r=20$}
\end{minipage}
\begin{minipage}[h]{1.3in}
\centering
\includegraphics[width= 1.3in]{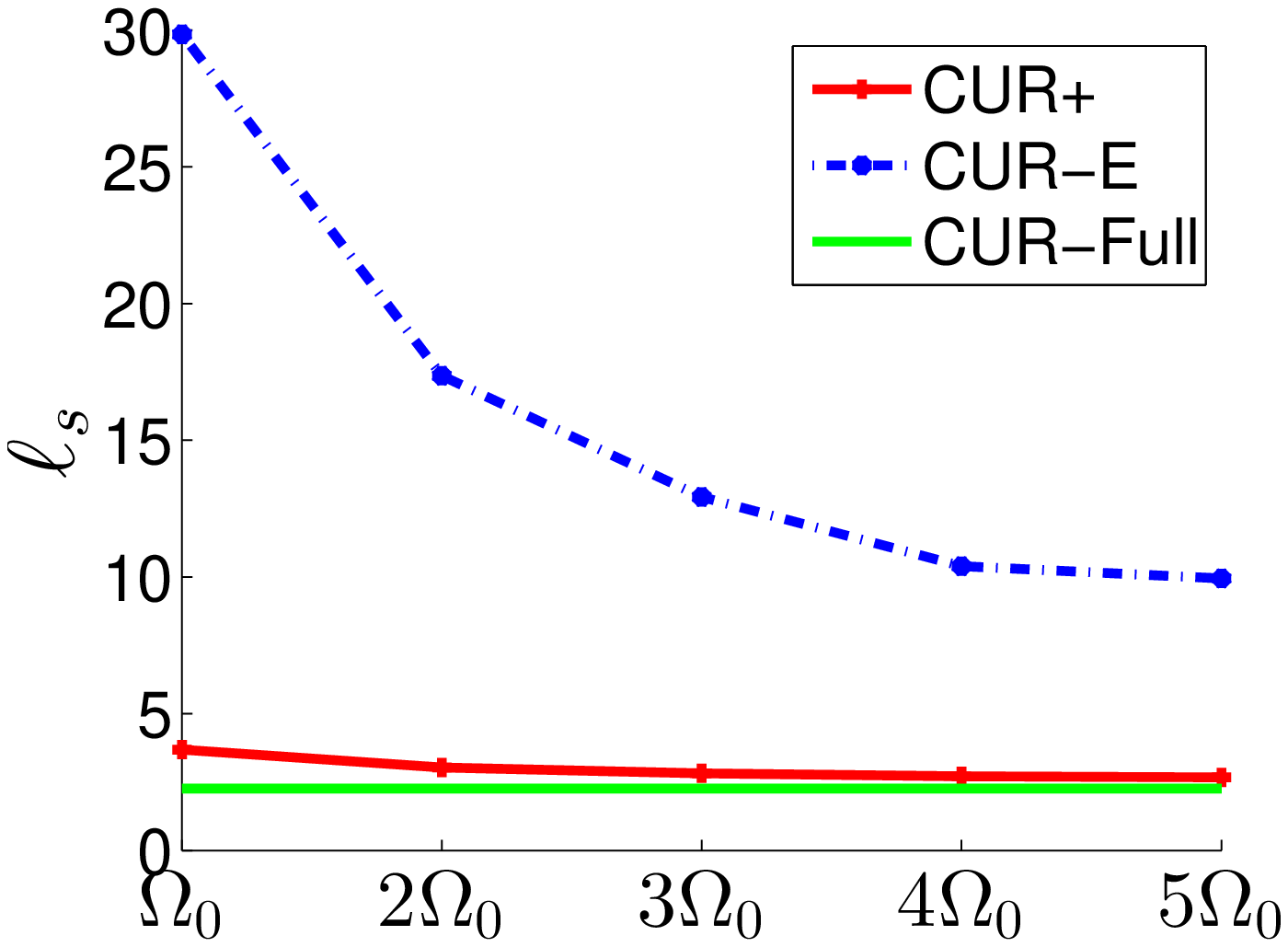}\\
\mbox{Gisette $r=20$}
\end{minipage}

\begin{minipage}[h]{1.3in}
\centering
\includegraphics[width= 1.3in]{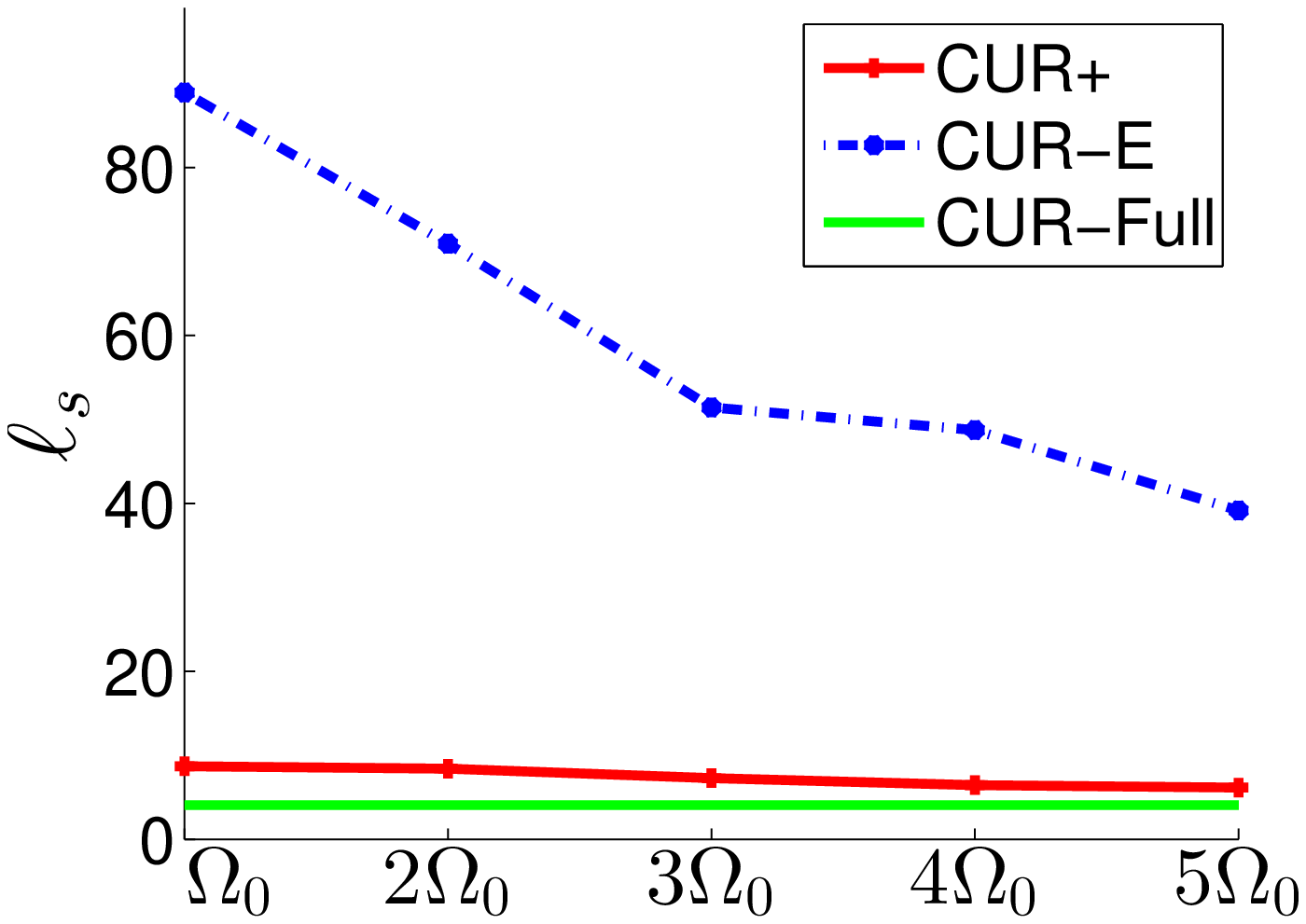}\\
\mbox{Enron $r=50$}
\end{minipage}
\begin{minipage}[h]{1.3in}
\centering
\includegraphics[width= 1.3in]{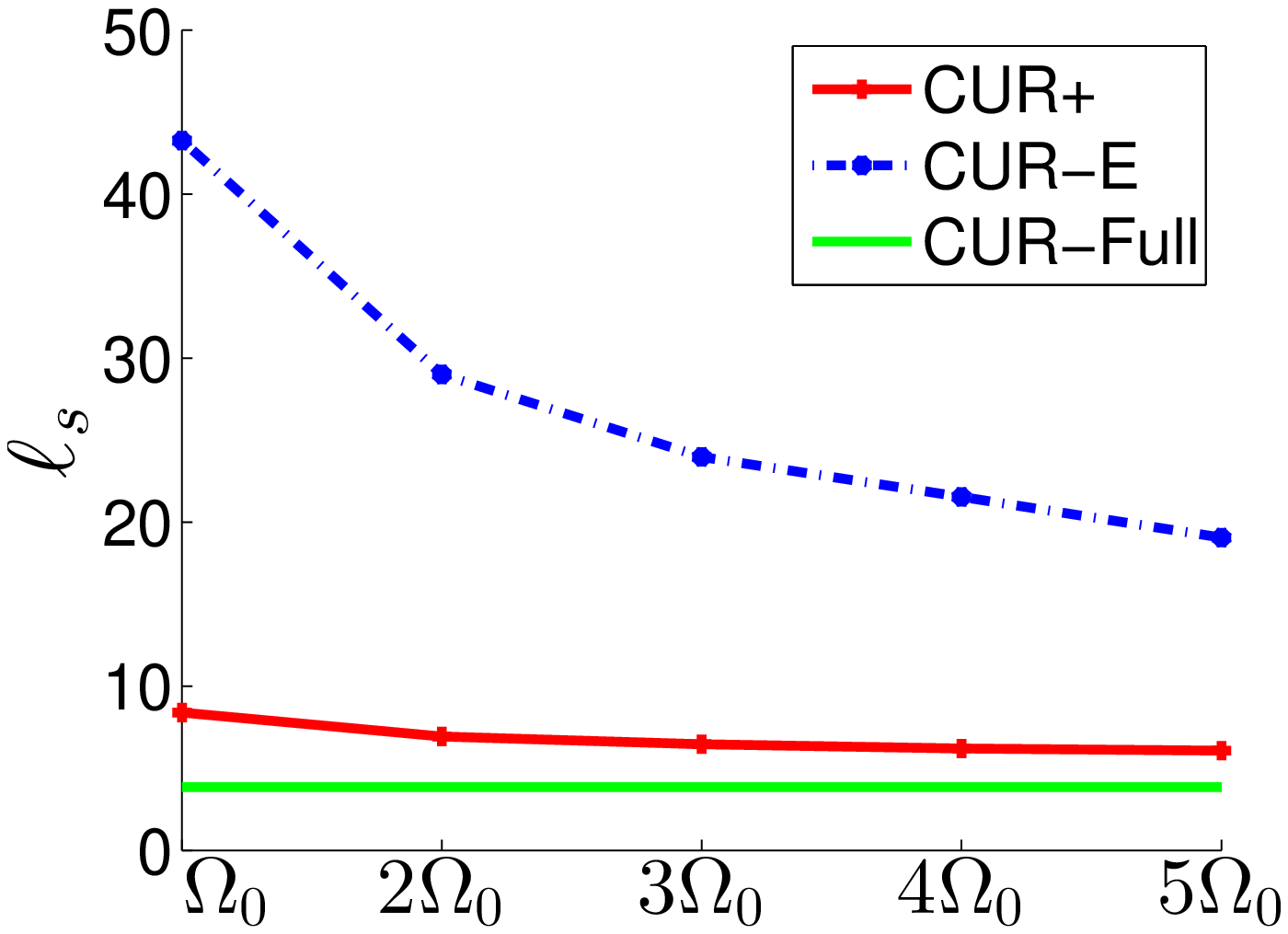}\\
\mbox{Dexter $r=50$}
\end{minipage}
\begin{minipage}[h]{1.3in}
\centering
\includegraphics[width= 1.3in]{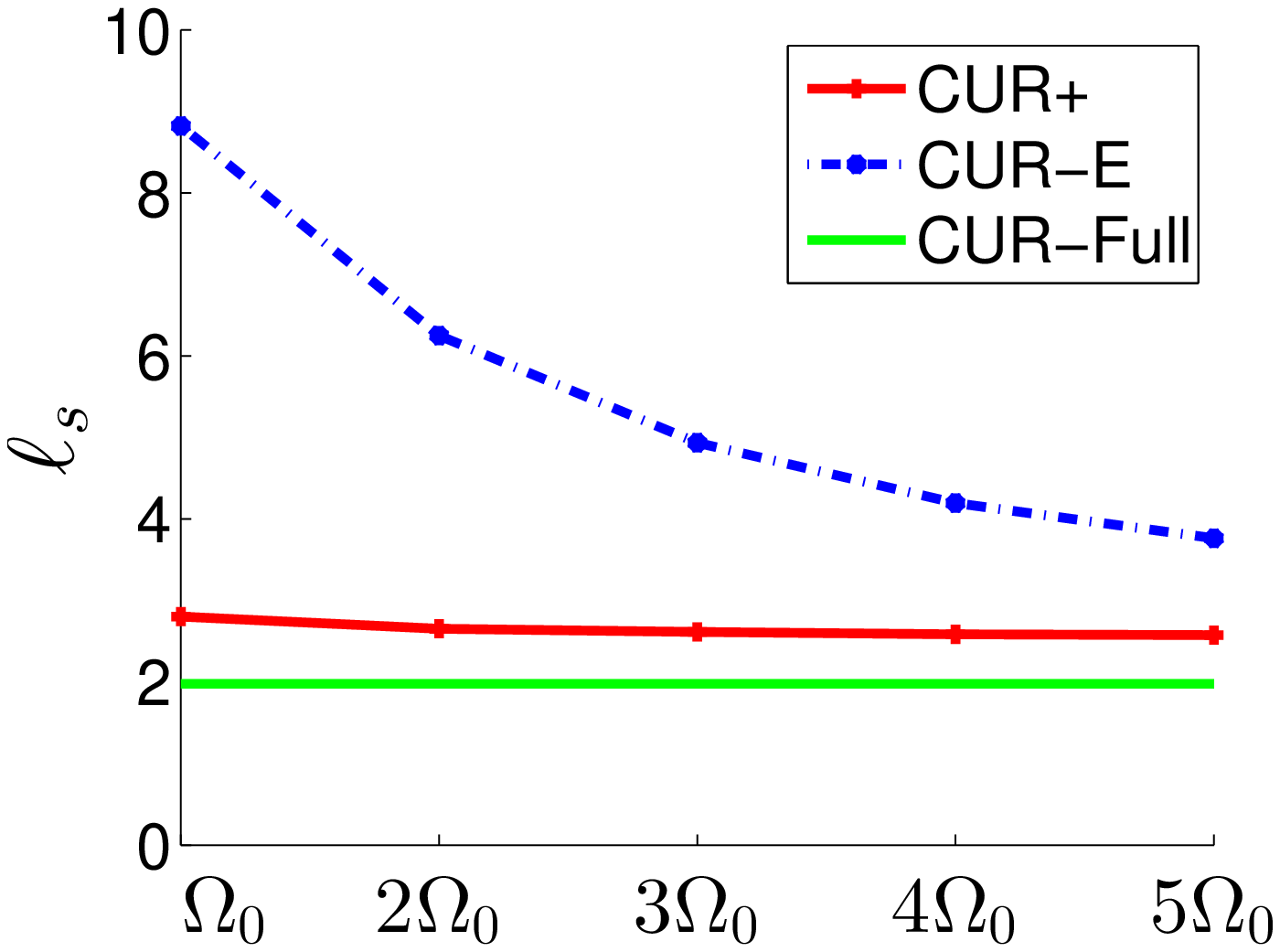}\\
\mbox{Farm Ads $r=50$}
\end{minipage}
\begin{minipage}[h]{1.3in}
\centering
\includegraphics[width= 1.3in]{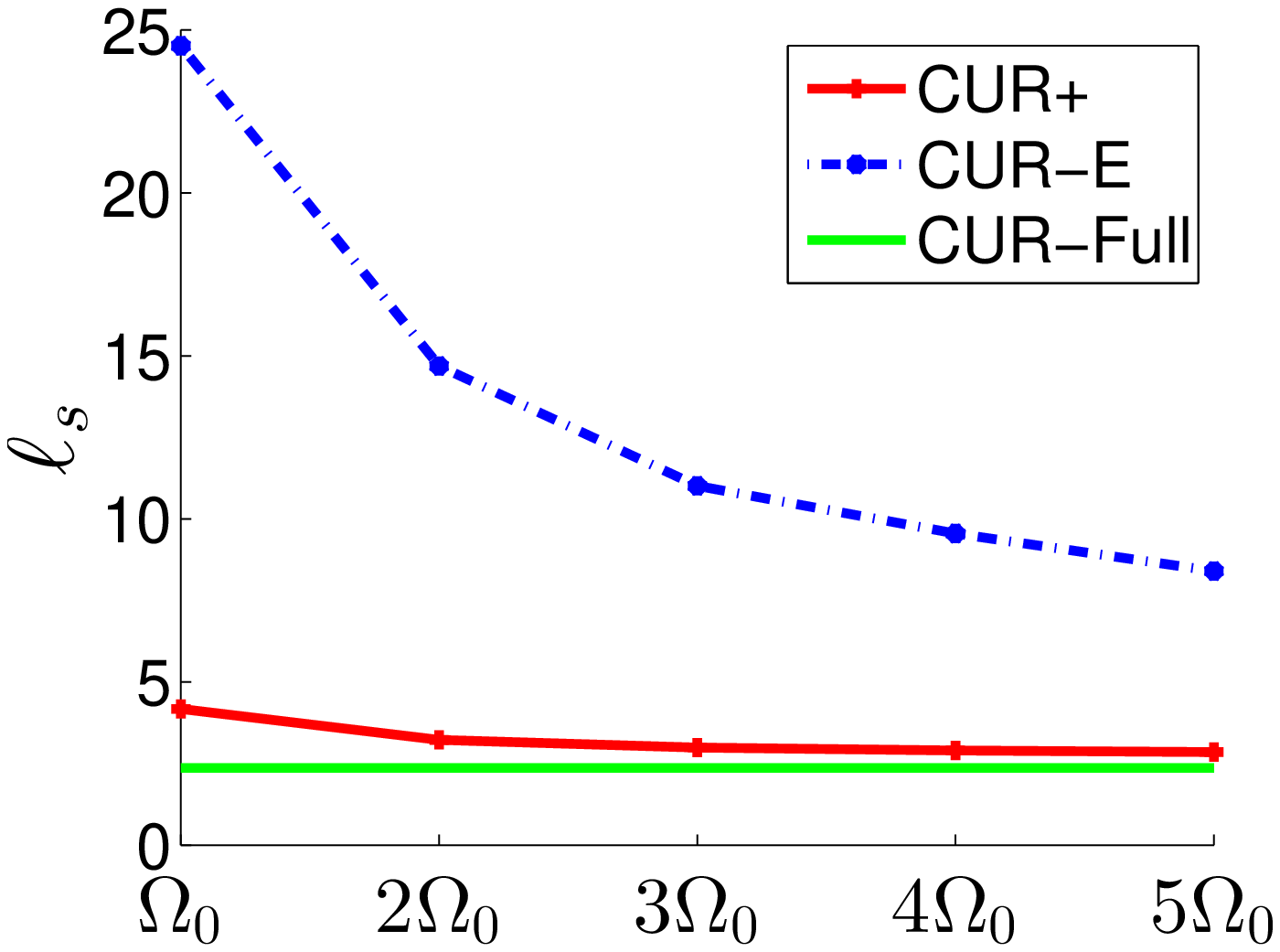}\\
\mbox{Gisette $r=50$}
\end{minipage}
\caption{Comparison of CUR algorithms with the number of sampled columns (rows) fixed as $d_1 = 5 r$ ($d_2 = 5 d_1$), where $r = 10,20,50$. The number of observed entries $|\Omega|$ is varied $\Omega_0$ to $5\Omega_0$.}\label{fig:varo}
\end{figure*}

\paragraph{Results Measured by Frobenius Norm} Similar results on relative Frobenius norm are also reported. The results are plotted in Figure~\ref{fig:fvard} when $|\Omega|$ is fixed and we vary $\alpha$, and in Figure~\ref{fig:fvaro} when $\alpha$ is fixed and we vary $|\Omega|$. We can see that similar as the results measured by spectral norm, the proposed CUR$+$ works significantly better than the CUR-E method, and yields a similar performance as the CUR-F algorithm that has an access to the full target matrix $M$.

\begin{figure*}[!t]
\centering
\begin{minipage}[h]{1.3in}
\centering
\includegraphics[width= 1.3in]{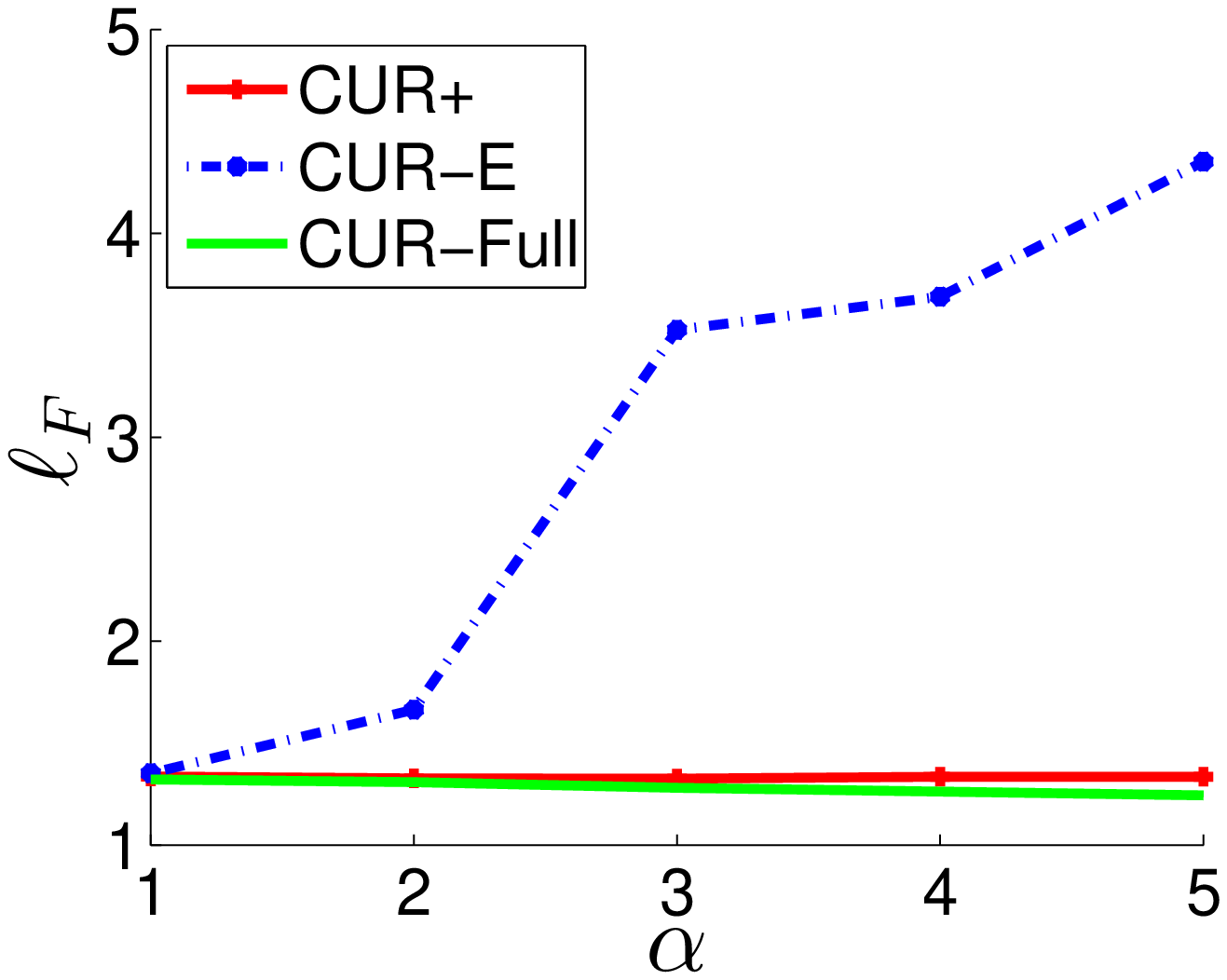}\\
\mbox{Enron $r=10$}
\end{minipage}
\begin{minipage}[h]{1.3in}
\centering
\includegraphics[width= 1.3in]{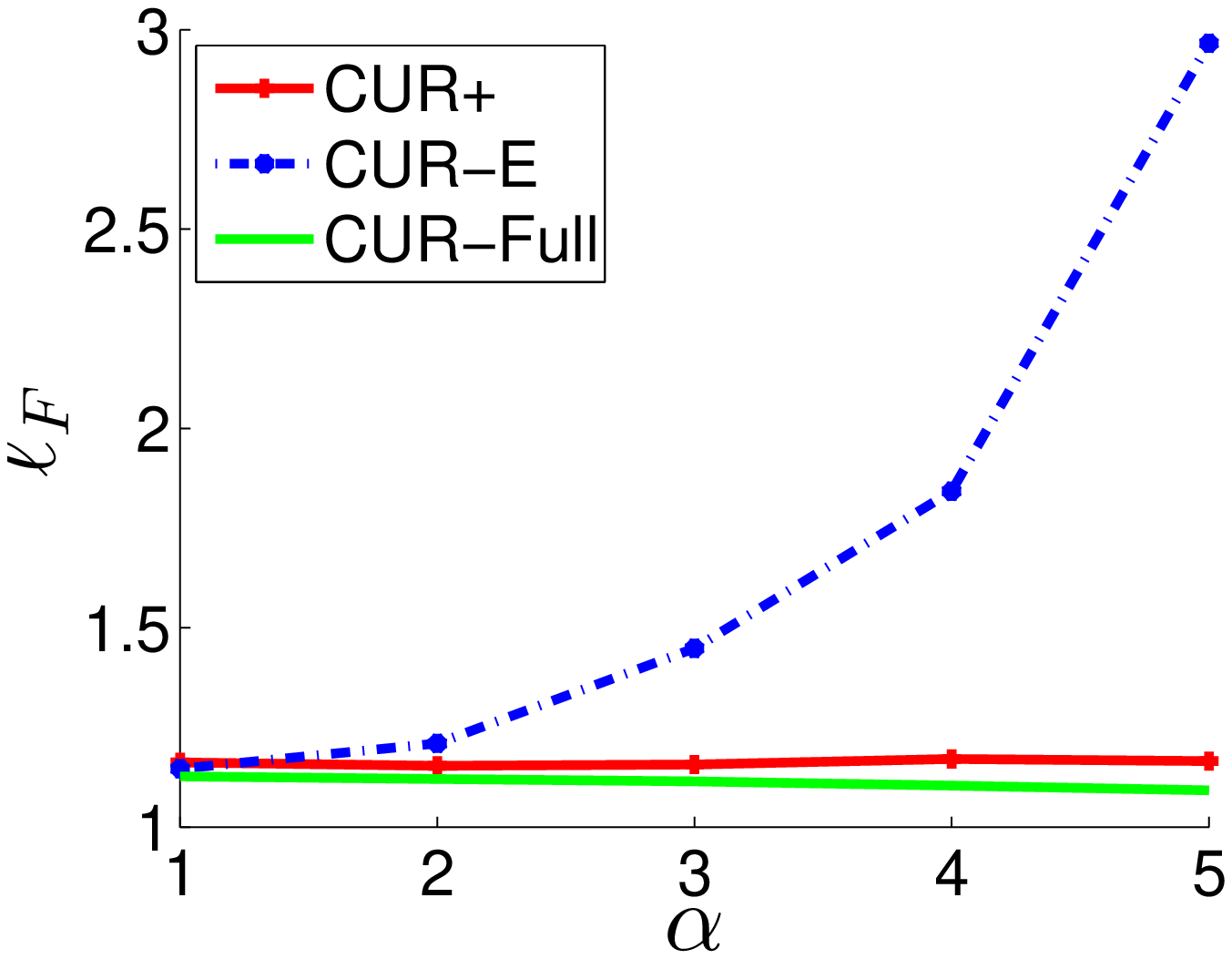}\\
\mbox{Dexter $r=10$}
\end{minipage}
\begin{minipage}[h]{1.3in}
\centering
\includegraphics[width= 1.3in]{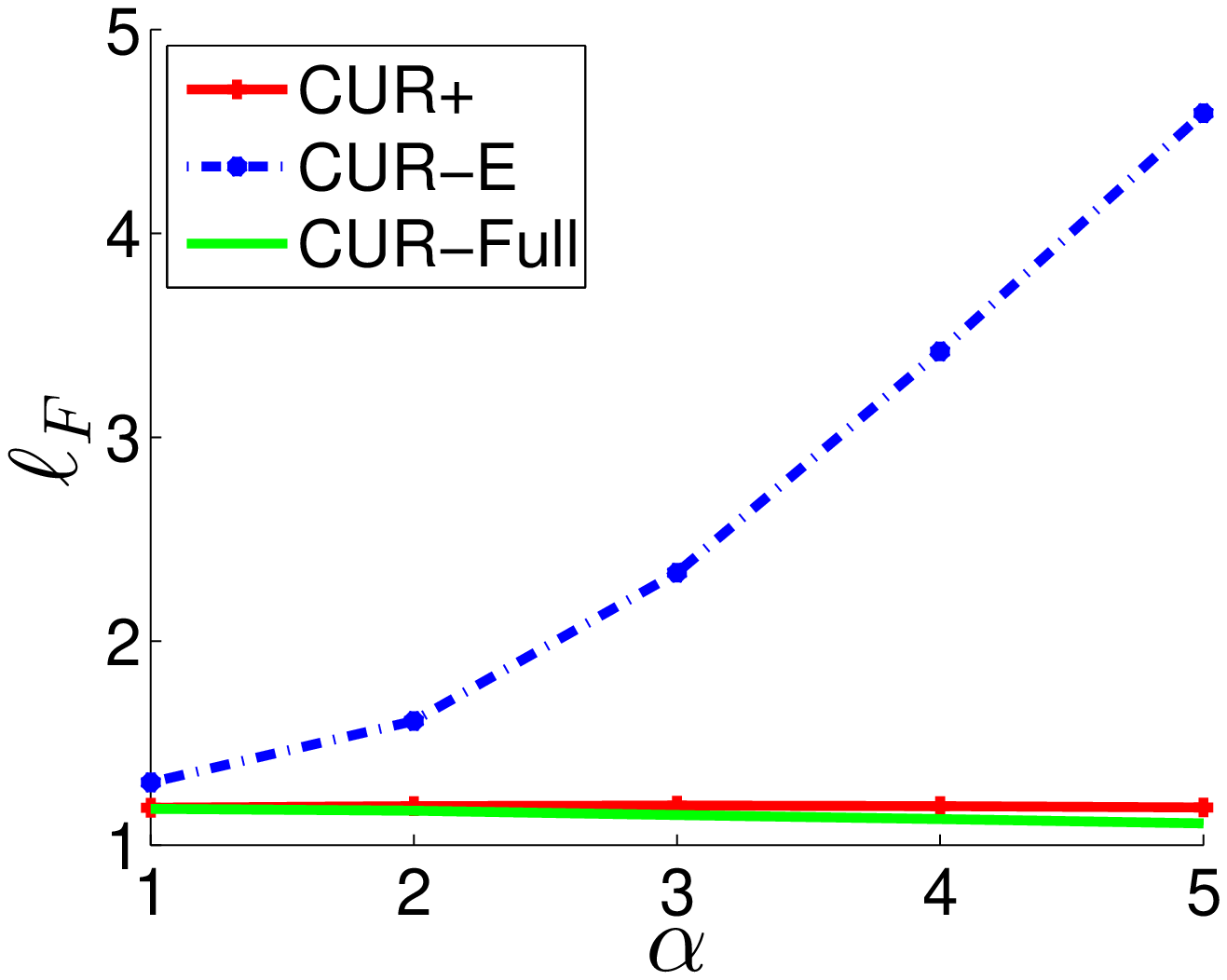}\\
\mbox{Farm Ads $r=10$}
\end{minipage}
\begin{minipage}[h]{1.3in}
\centering
\includegraphics[width= 1.3in]{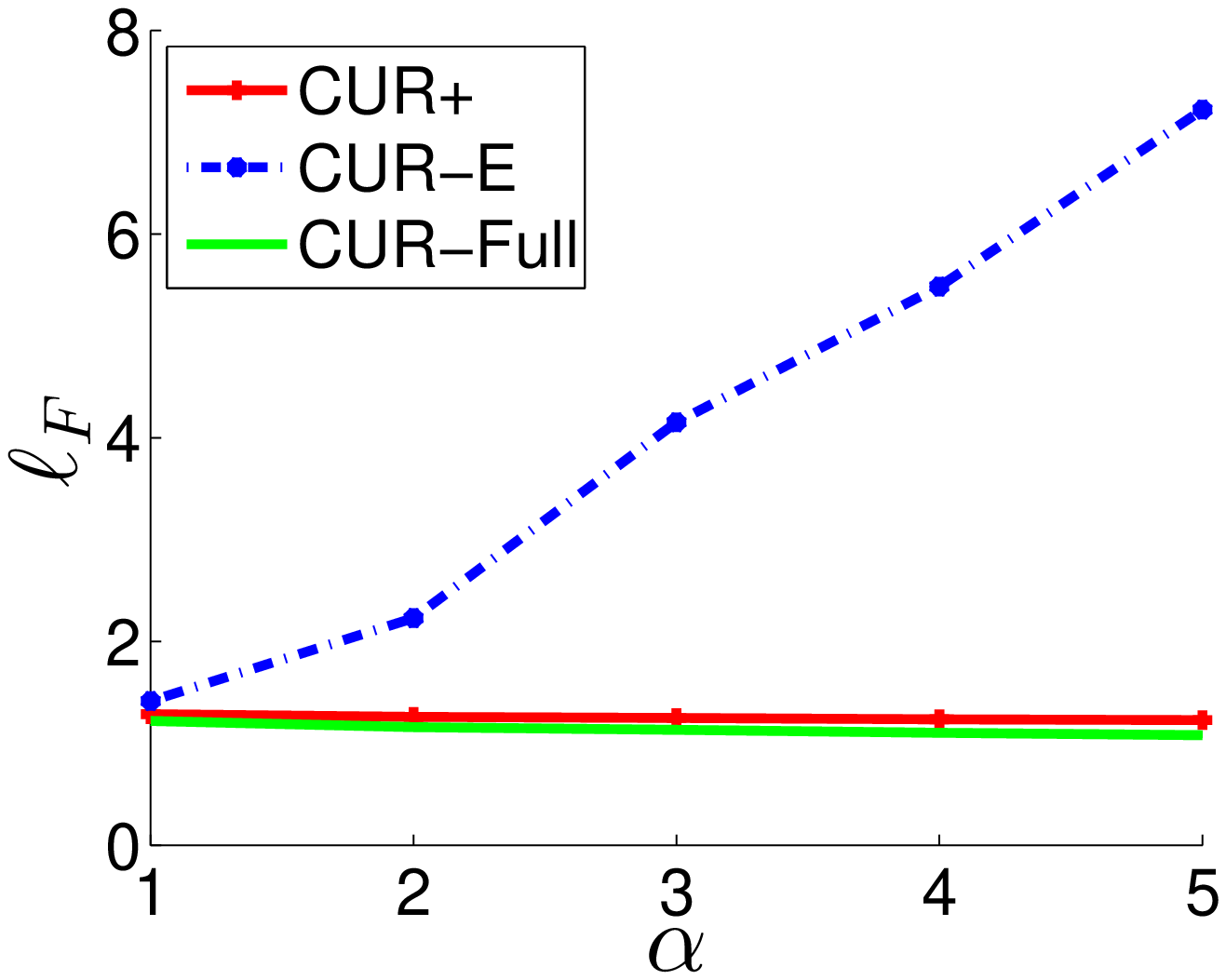}\\
\mbox{Gisette $r=10$}
\end{minipage}

\begin{minipage}[h]{1.3in}
\centering
\includegraphics[width= 1.3in]{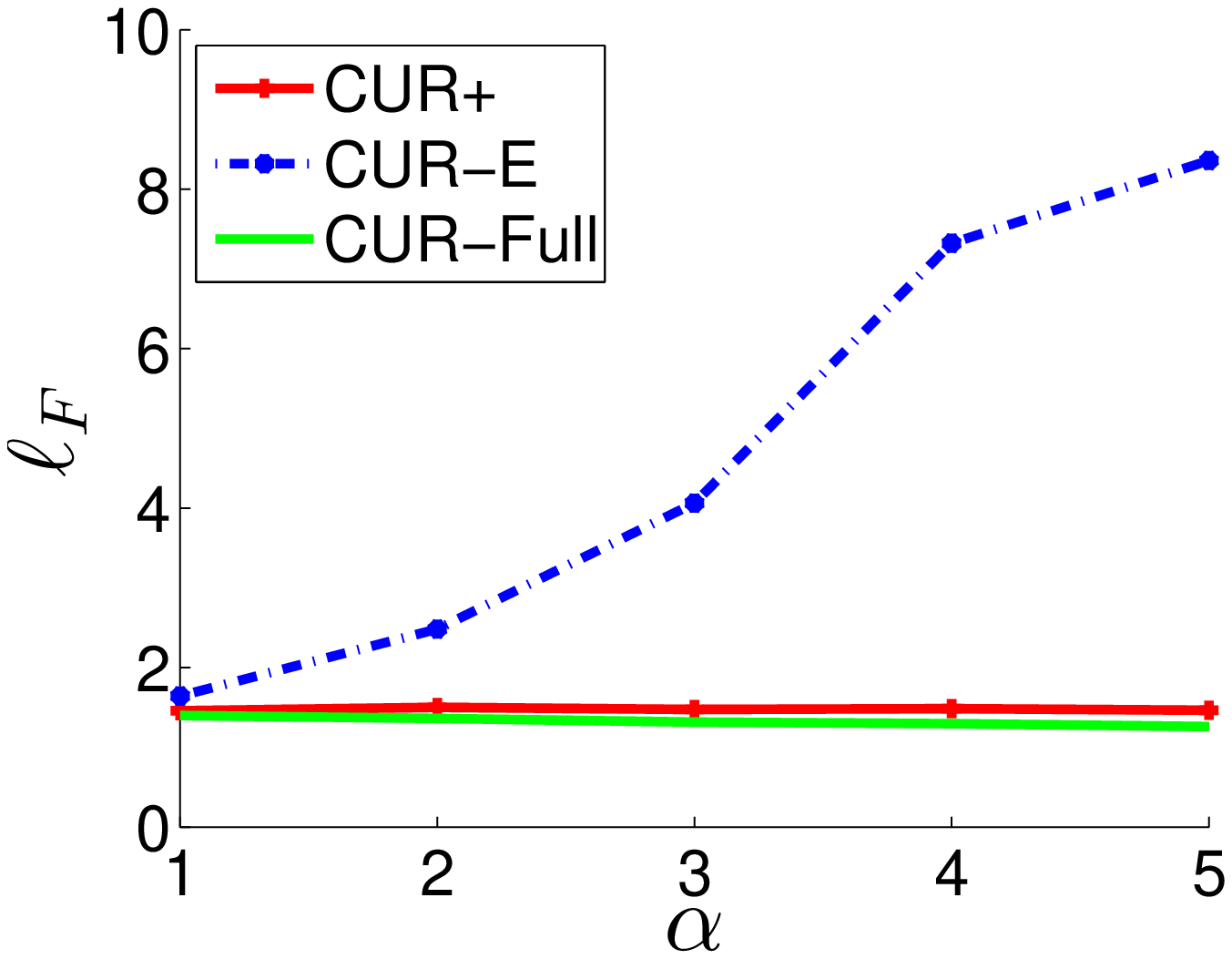}\\
\mbox{Enron $r=20$}
\end{minipage}
\begin{minipage}[h]{1.3in}
\centering
\includegraphics[width= 1.3in]{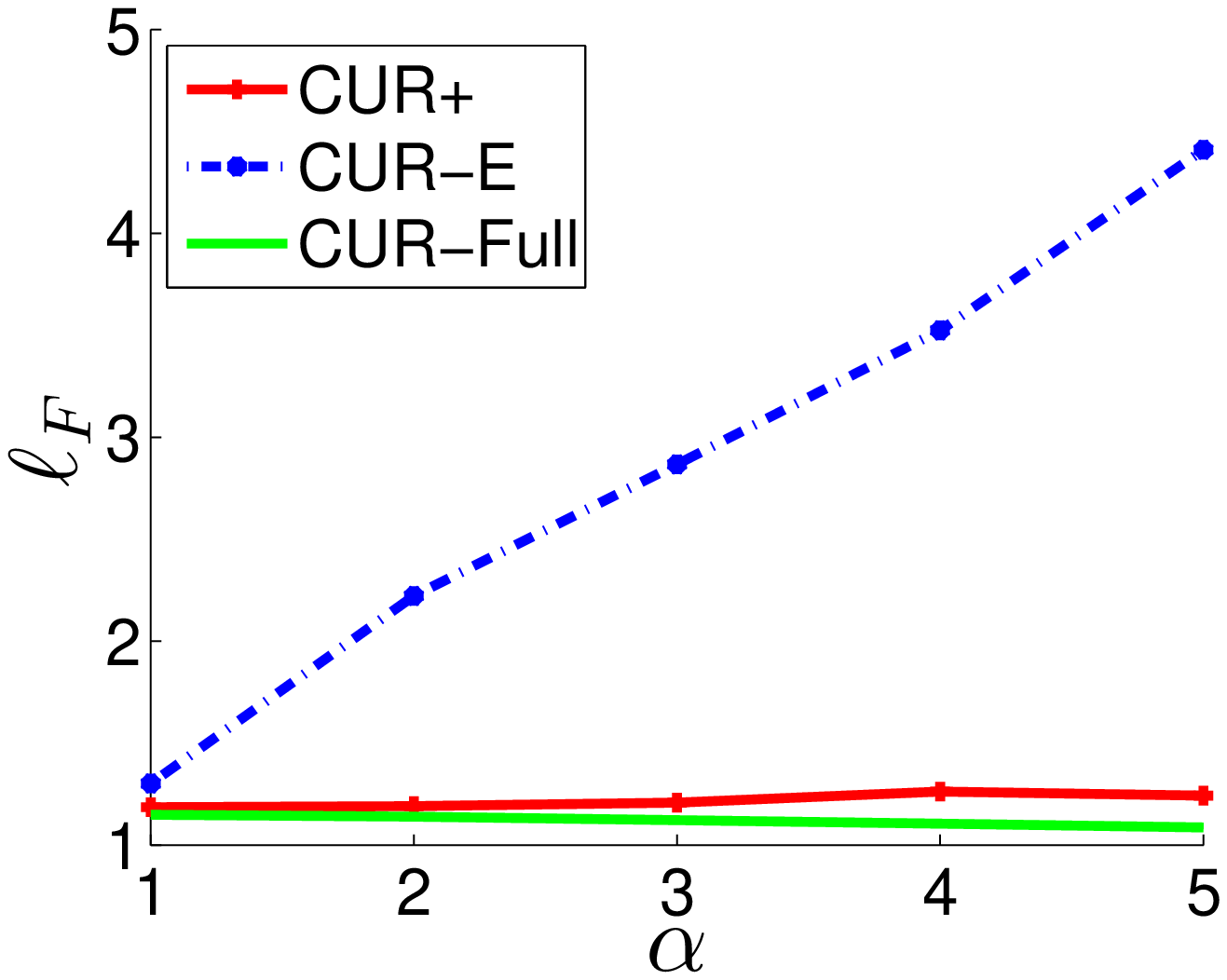}\\
\mbox{Dexter $r=20$}
\end{minipage}
\begin{minipage}[h]{1.3in}
\centering
\includegraphics[width= 1.3in]{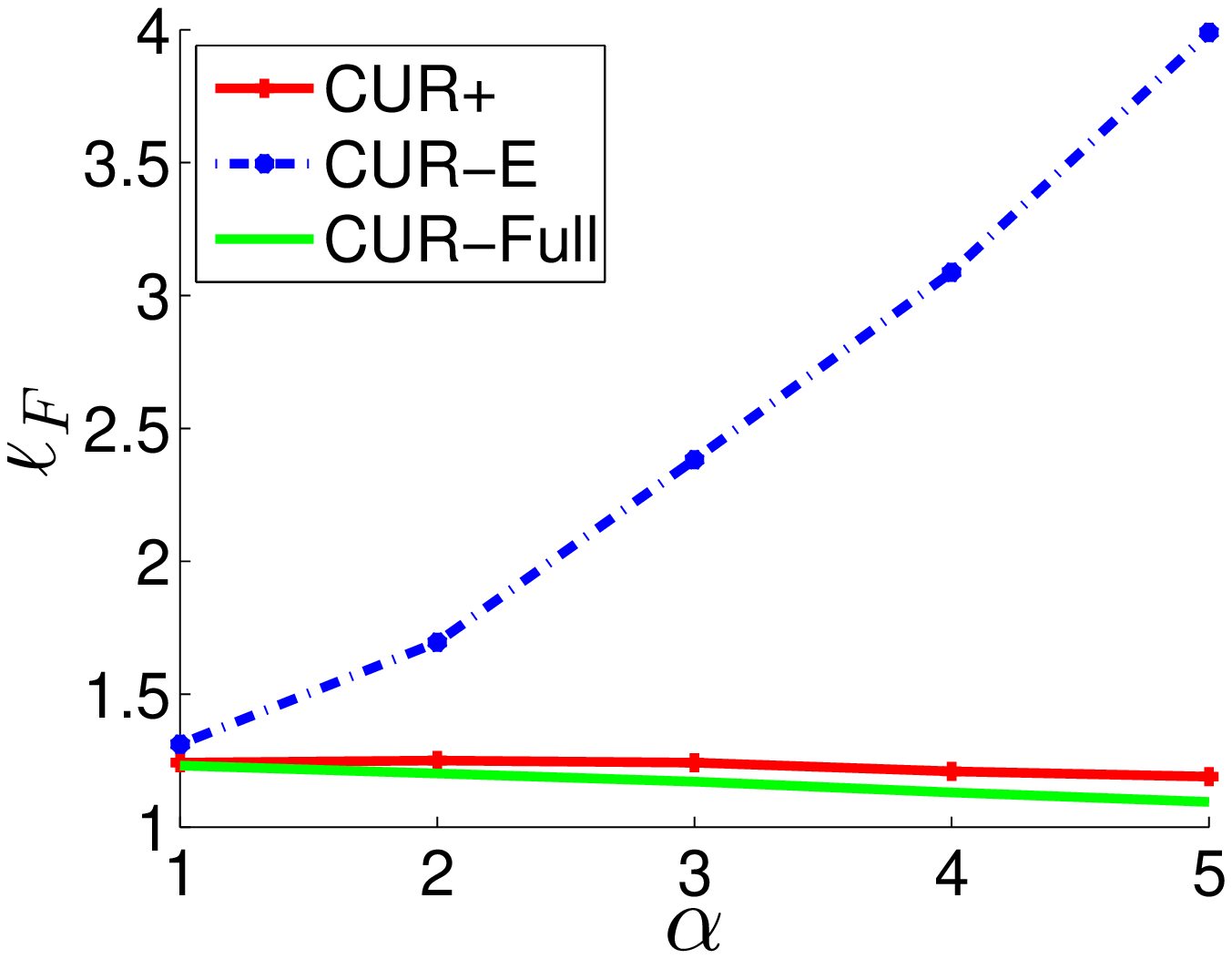}\\
\mbox{Farm Ads $r=20$}
\end{minipage}
\begin{minipage}[h]{1.3in}
\centering
\includegraphics[width= 1.3in]{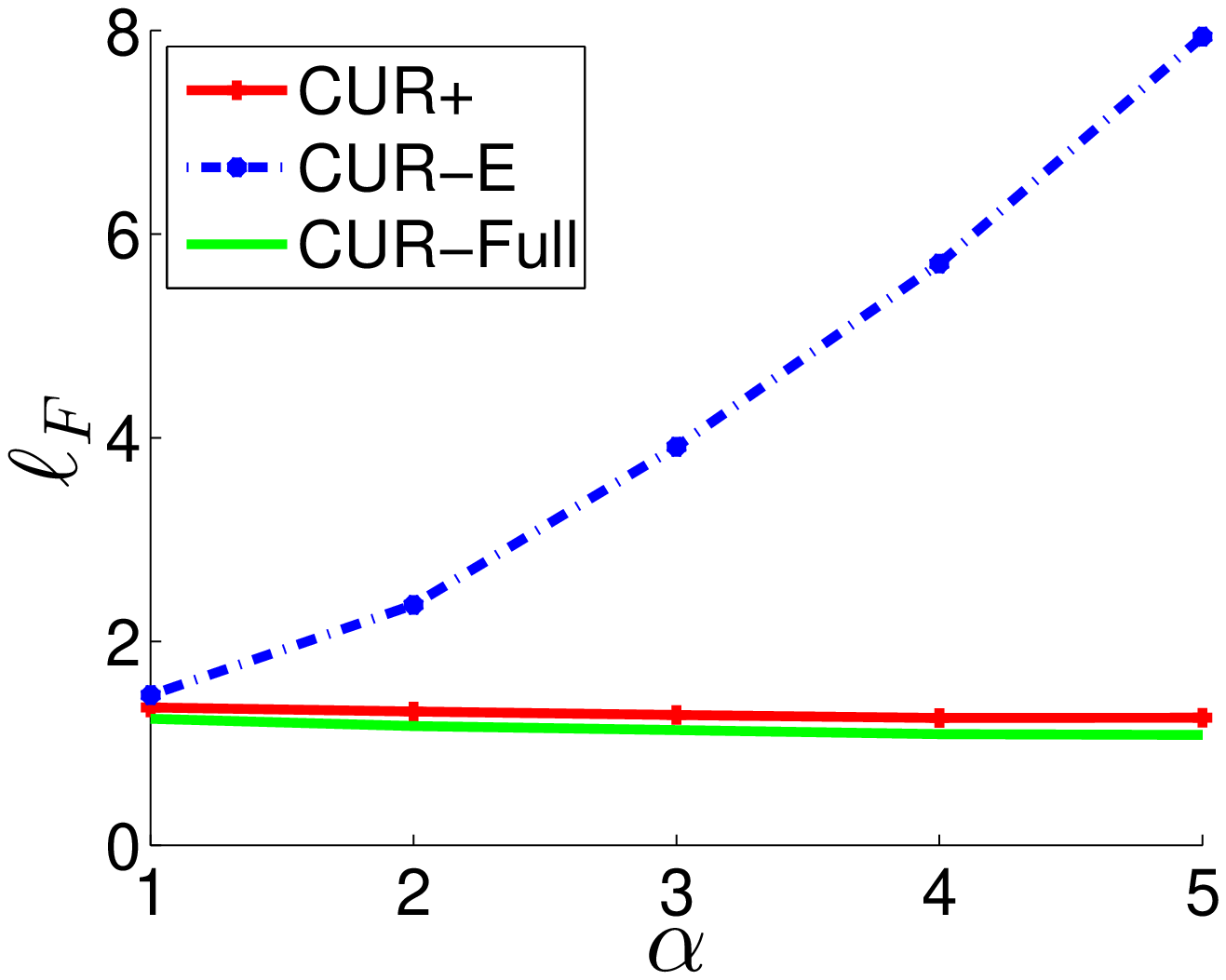}\\
\mbox{Gisette $r=20$}
\end{minipage}

\begin{minipage}[h]{1.3in}
\centering
\includegraphics[width= 1.3in]{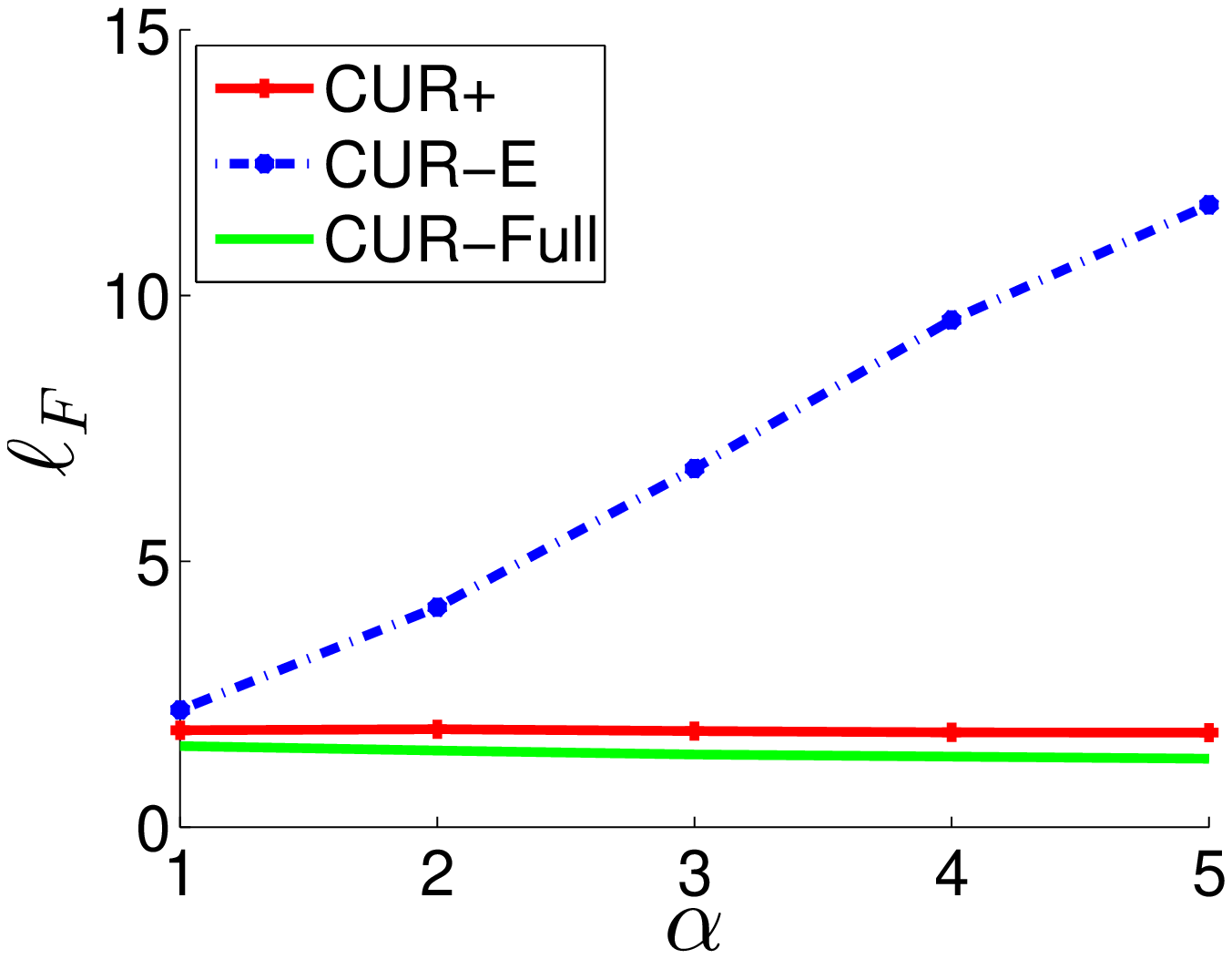}\\
\mbox{Enron $r=50$}
\end{minipage}
\begin{minipage}[h]{1.3in}
\centering
\includegraphics[width= 1.3in]{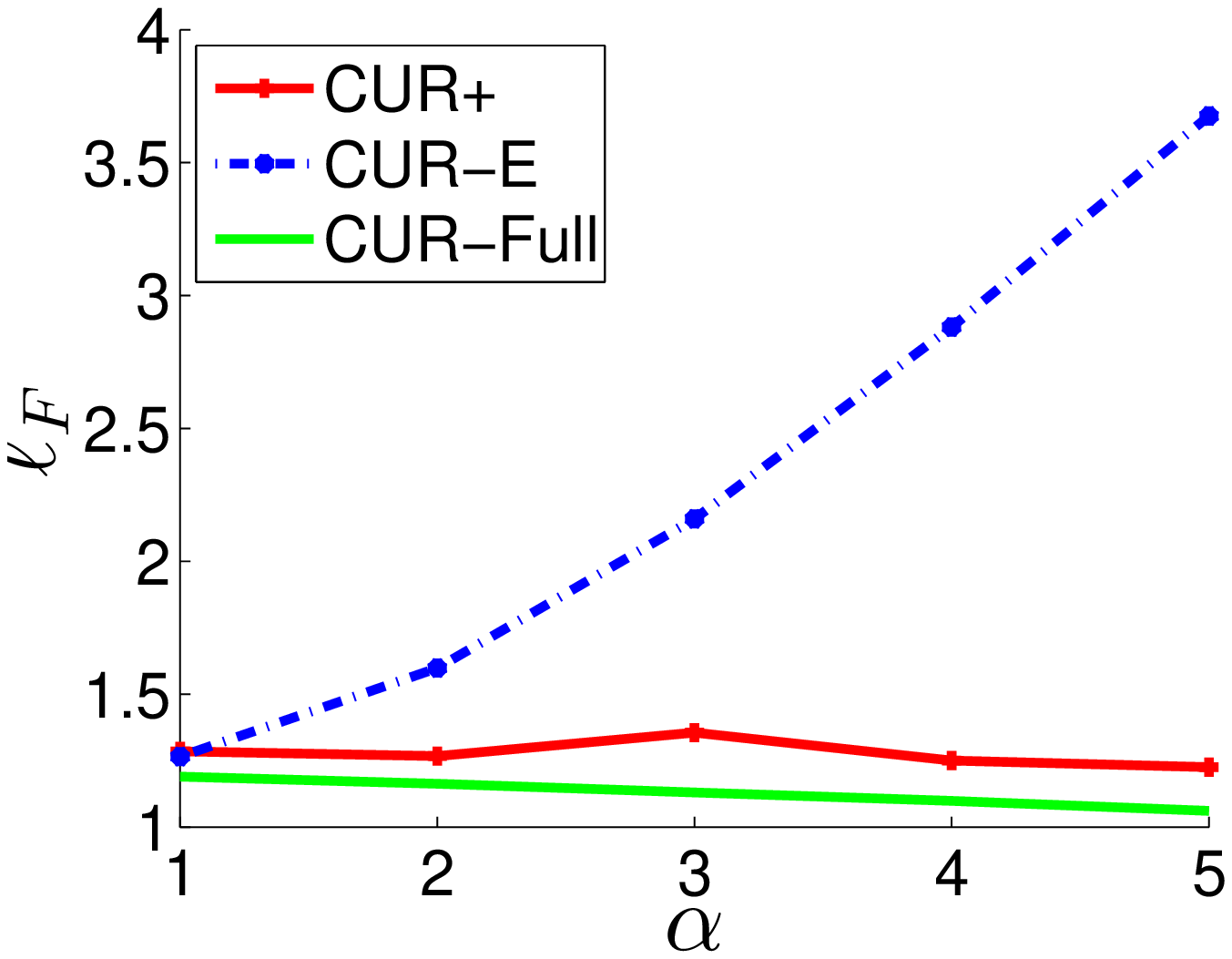}\\
\mbox{Dexter $r=50$}
\end{minipage}
\begin{minipage}[h]{1.3in}
\centering
\includegraphics[width= 1.3in]{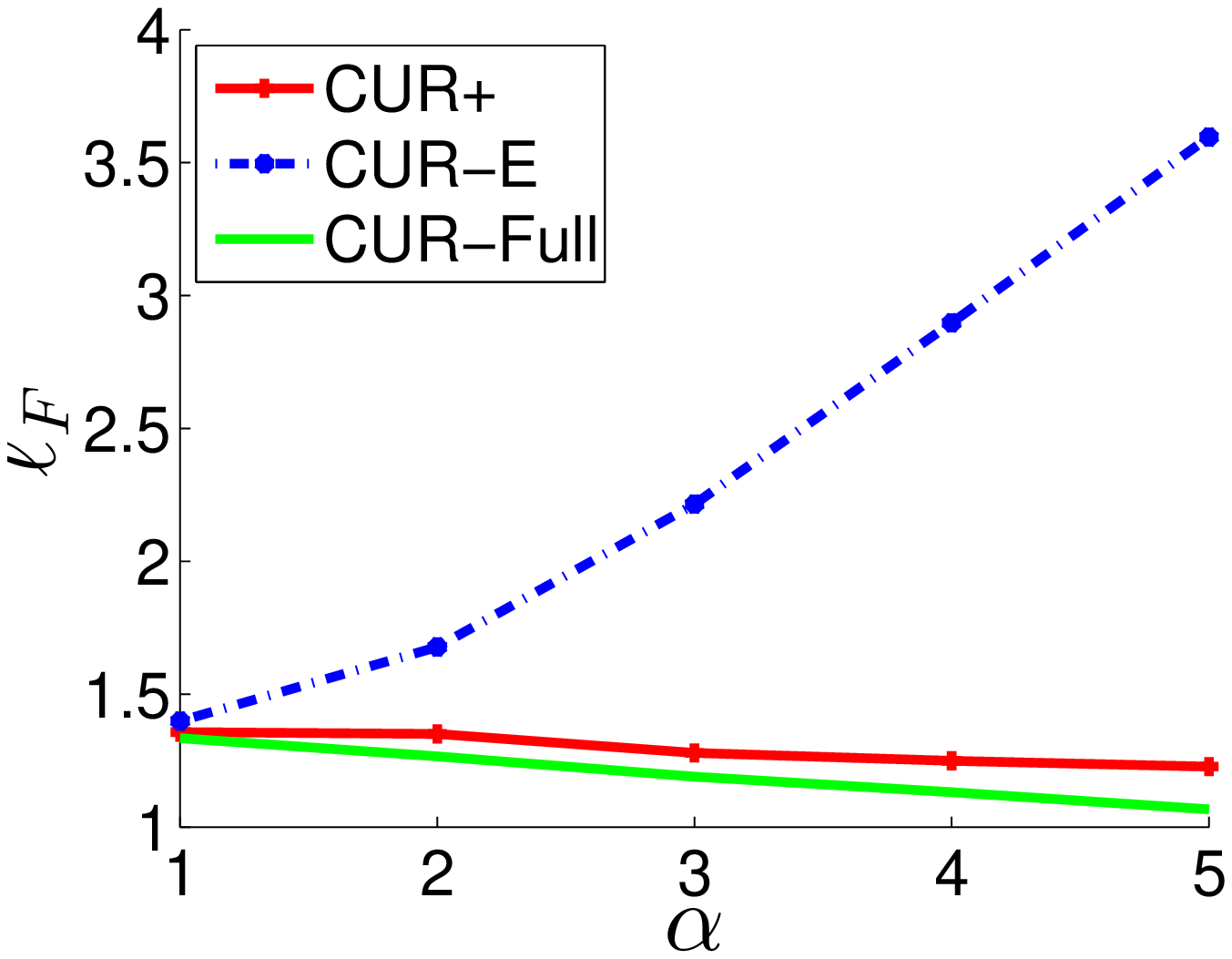}\\
\mbox{Farm Ads $r=50$}
\end{minipage}
\begin{minipage}[h]{1.3in}
\centering
\includegraphics[width= 1.3in]{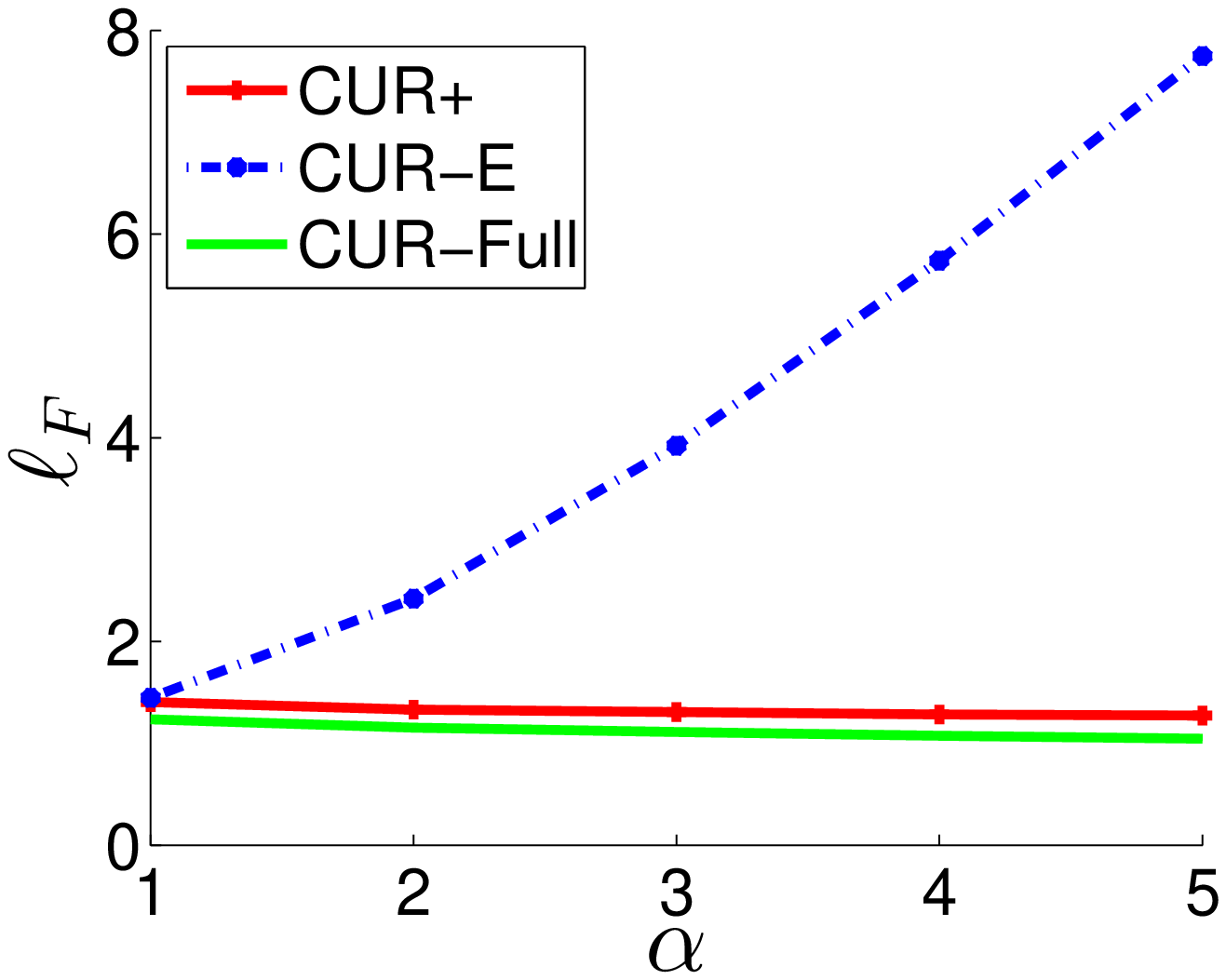}\\
\mbox{Gisette $r=50$}
\end{minipage}
\caption{Comparison of CUR algorithms measured by Frobenius norm with the number of observed entries $|\Omega|$ fixed as $|\Omega| = \Omega_0$. The number of sampled columns and rows are set as $d_1 = \alpha r$ and $d_2 = \alpha d_1$, respectively, where $r =10, 20,50$ and $\alpha$ is varied between $1$ and $5$.}\label{fig:fvard}
\end{figure*}

\begin{figure*}[t]
\centering
\begin{minipage}[h]{1.3in}
\centering
\includegraphics[width= 1.3in]{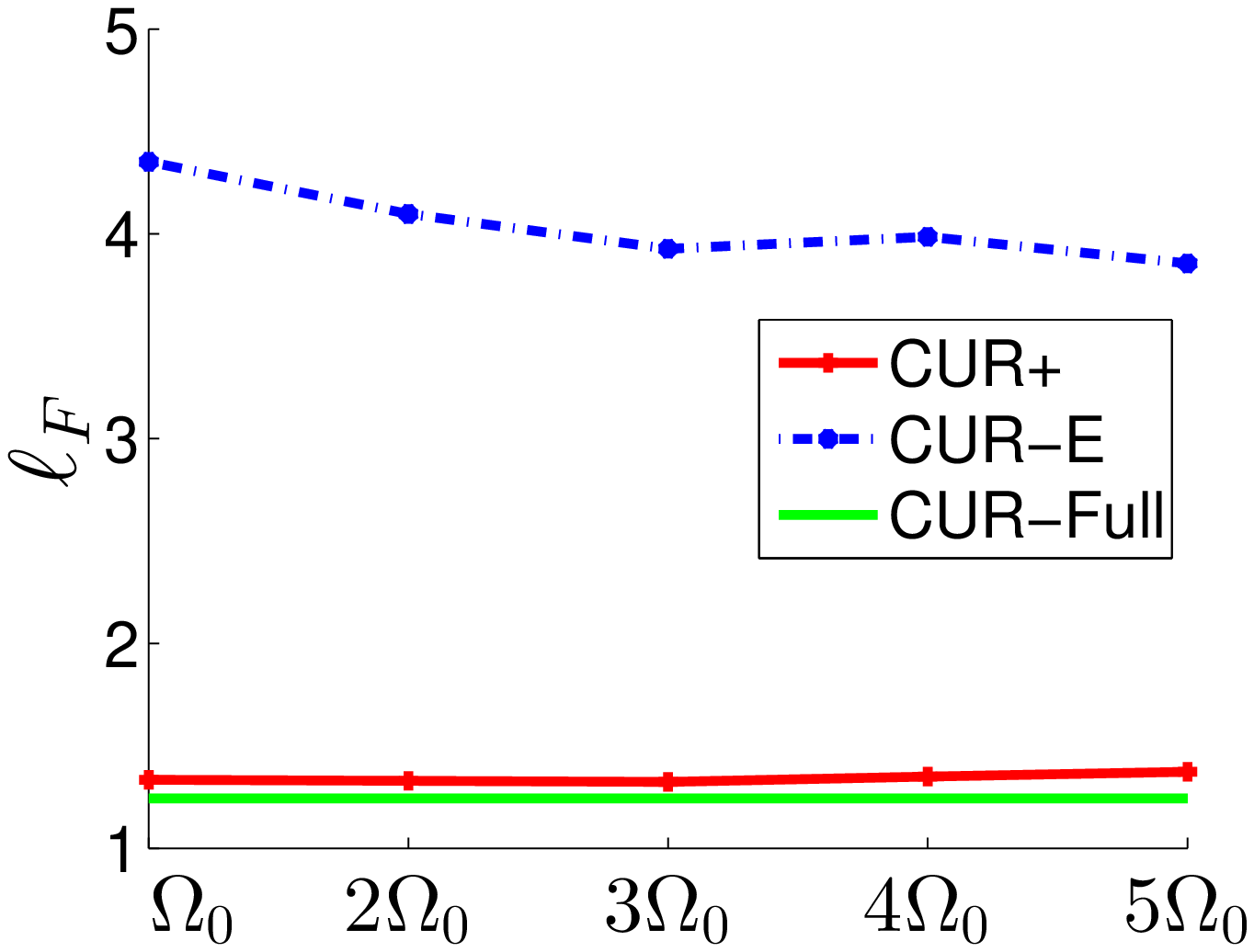}\\
\mbox{Enron $r=10$}
\end{minipage}
\begin{minipage}[h]{1.3in}
\centering
\includegraphics[width= 1.3in]{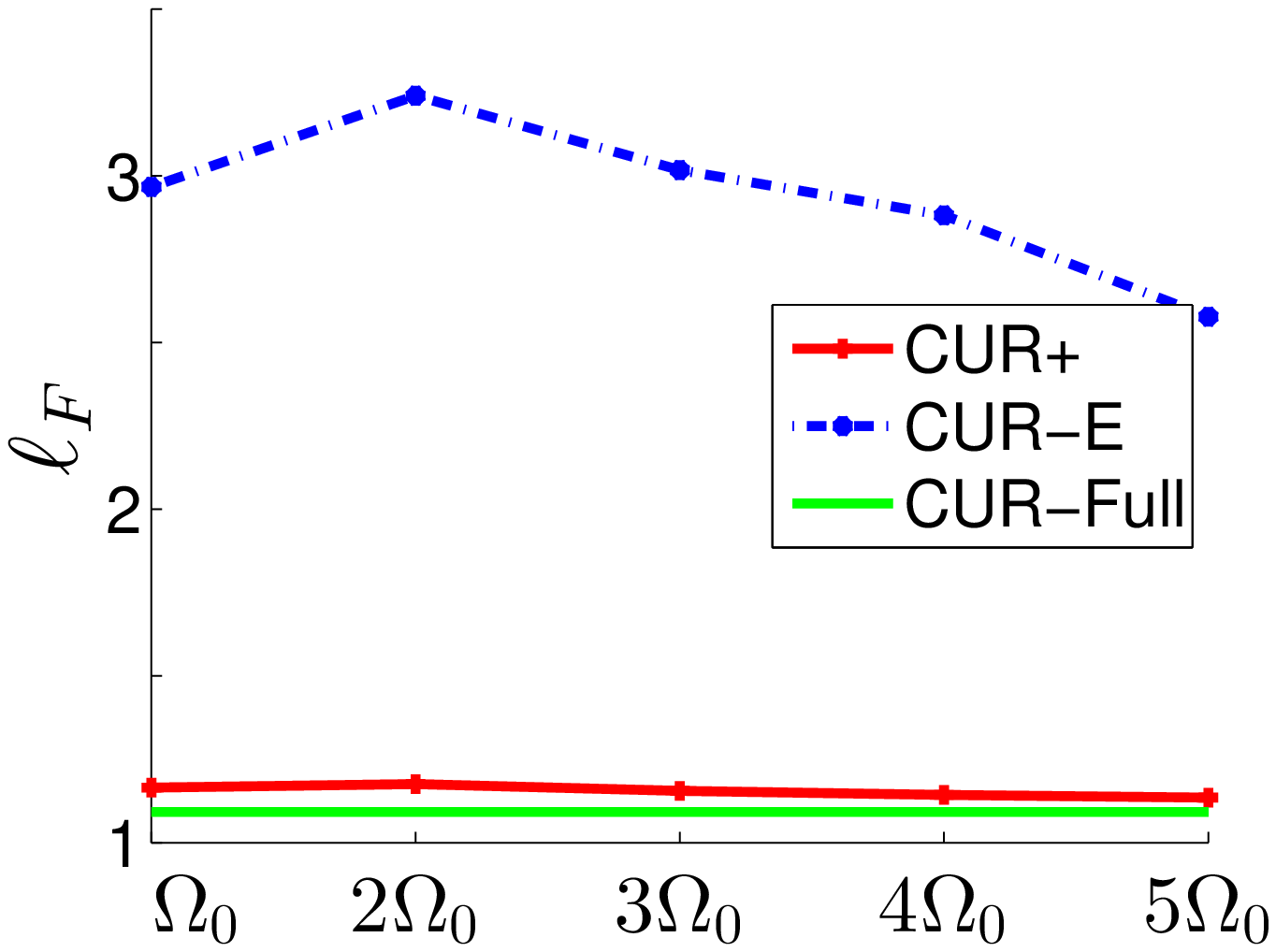}\\
\mbox{Dexter $r=10$}
\end{minipage}
\begin{minipage}[h]{1.3in}
\centering
\includegraphics[width= 1.3in]{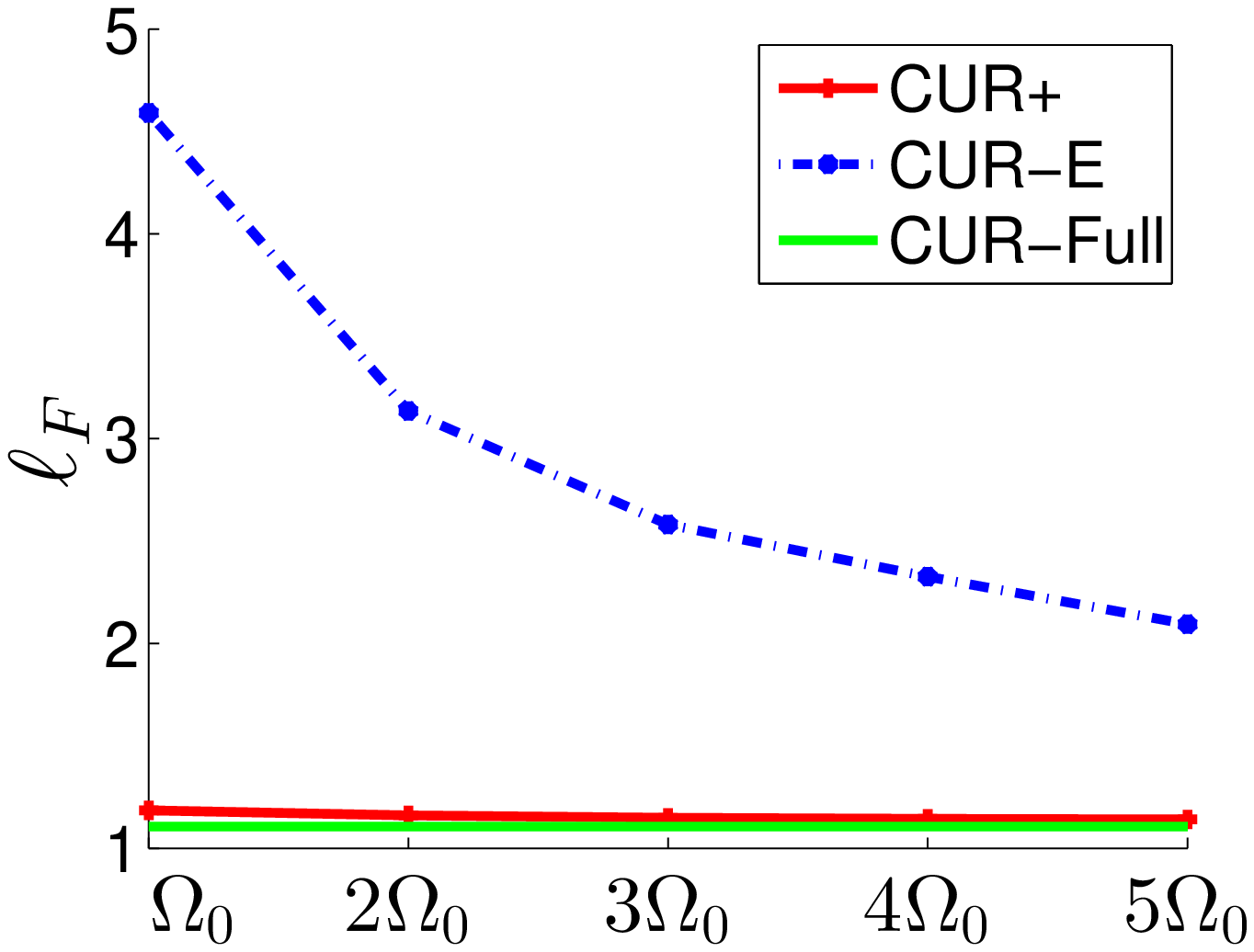}\\
\mbox{Farm Ads $r=10$}
\end{minipage}
\begin{minipage}[h]{1.3in}
\centering
\includegraphics[width= 1.3in]{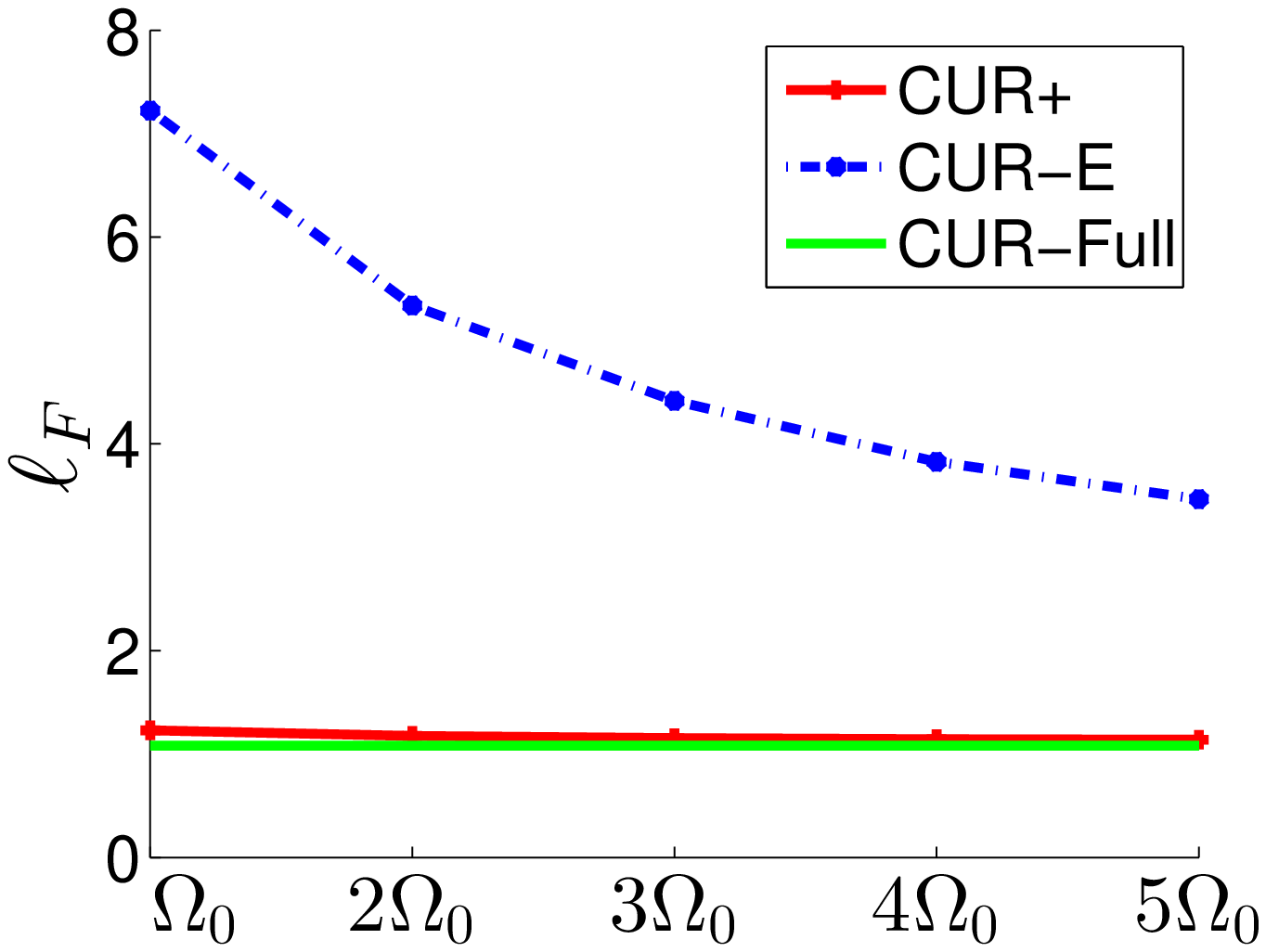}\\
\mbox{Gisette $r=10$}
\end{minipage}

\begin{minipage}[h]{1.3in}
\centering
\includegraphics[width= 1.3in]{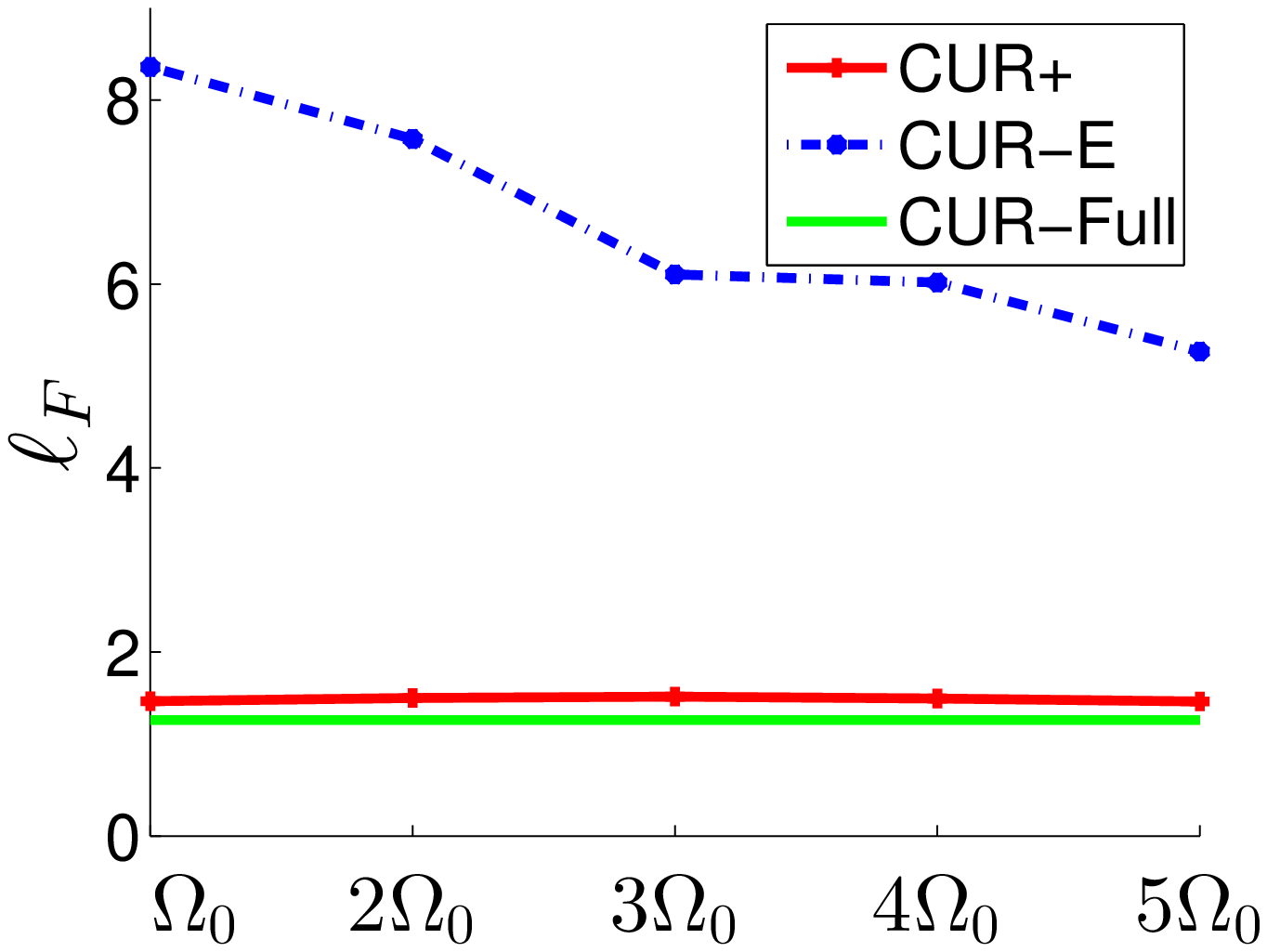}\\
\mbox{Enron $r=20$}
\end{minipage}
\begin{minipage}[h]{1.3in}
\centering
\includegraphics[width= 1.3in]{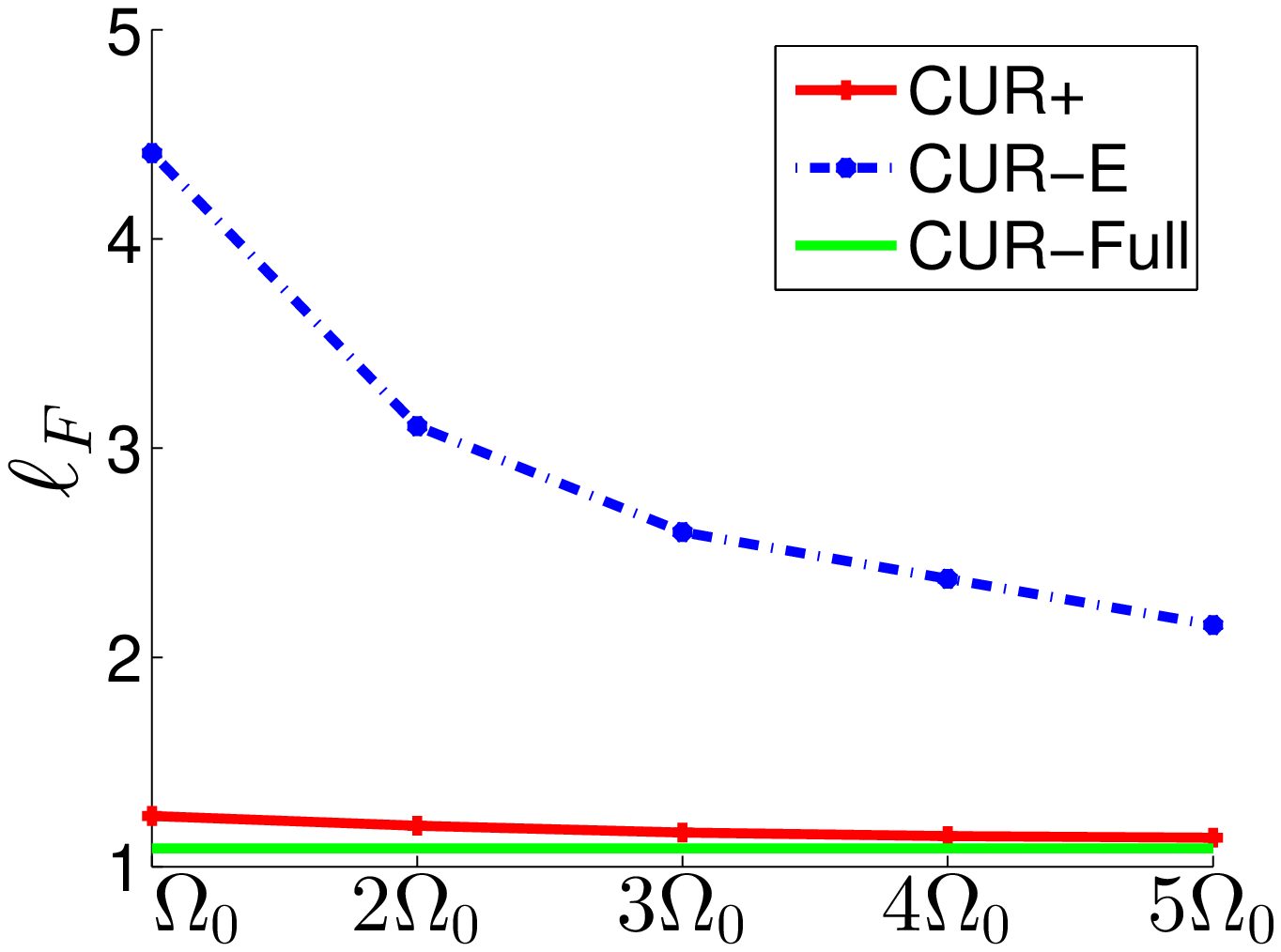}\\
\mbox{Dexter $r=20$}
\end{minipage}
\begin{minipage}[h]{1.3in}
\centering
\includegraphics[width= 1.3in]{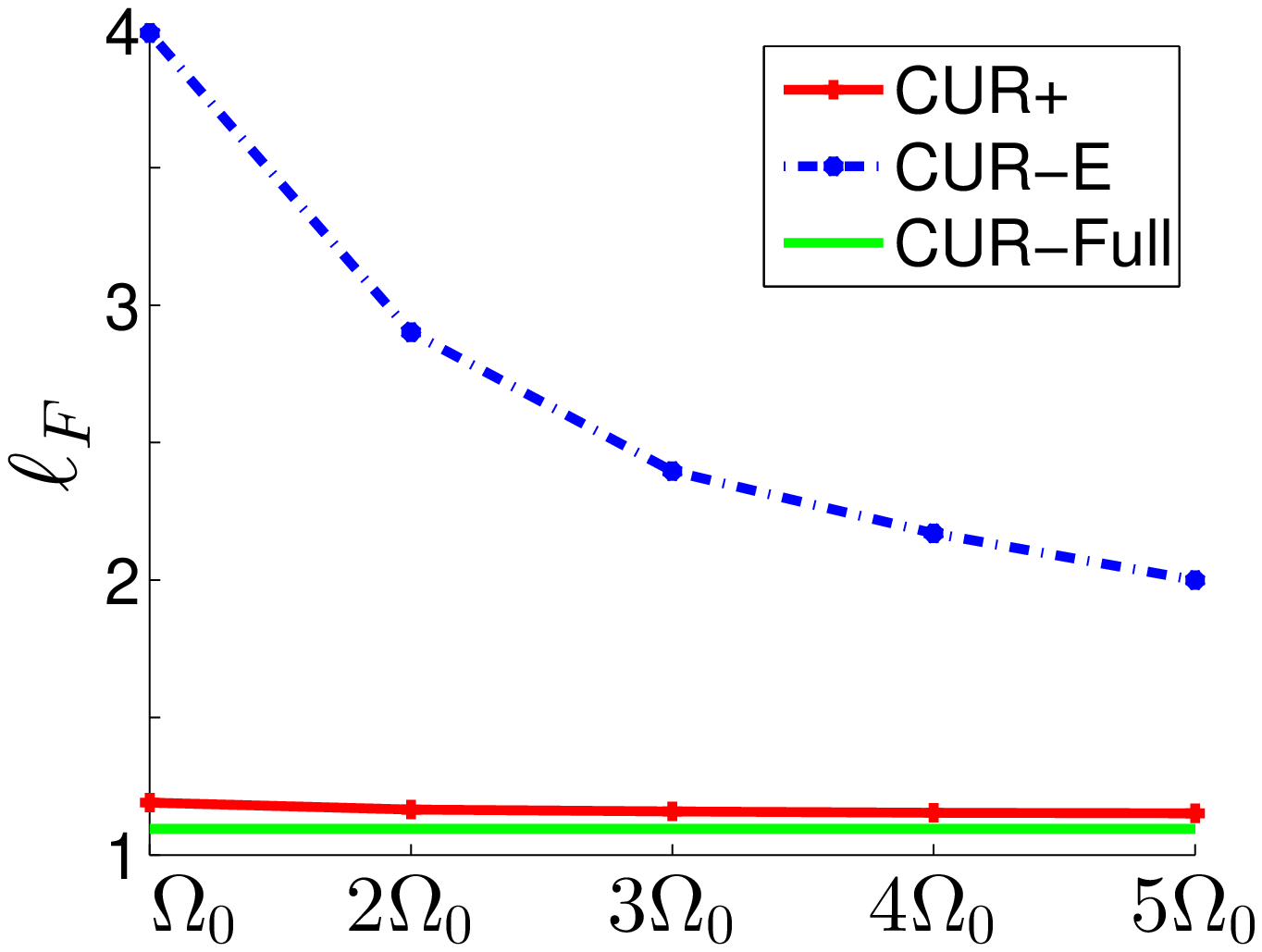}\\
\mbox{Farm Ads $r=20$}
\end{minipage}
\begin{minipage}[h]{1.3in}
\centering
\includegraphics[width= 1.3in]{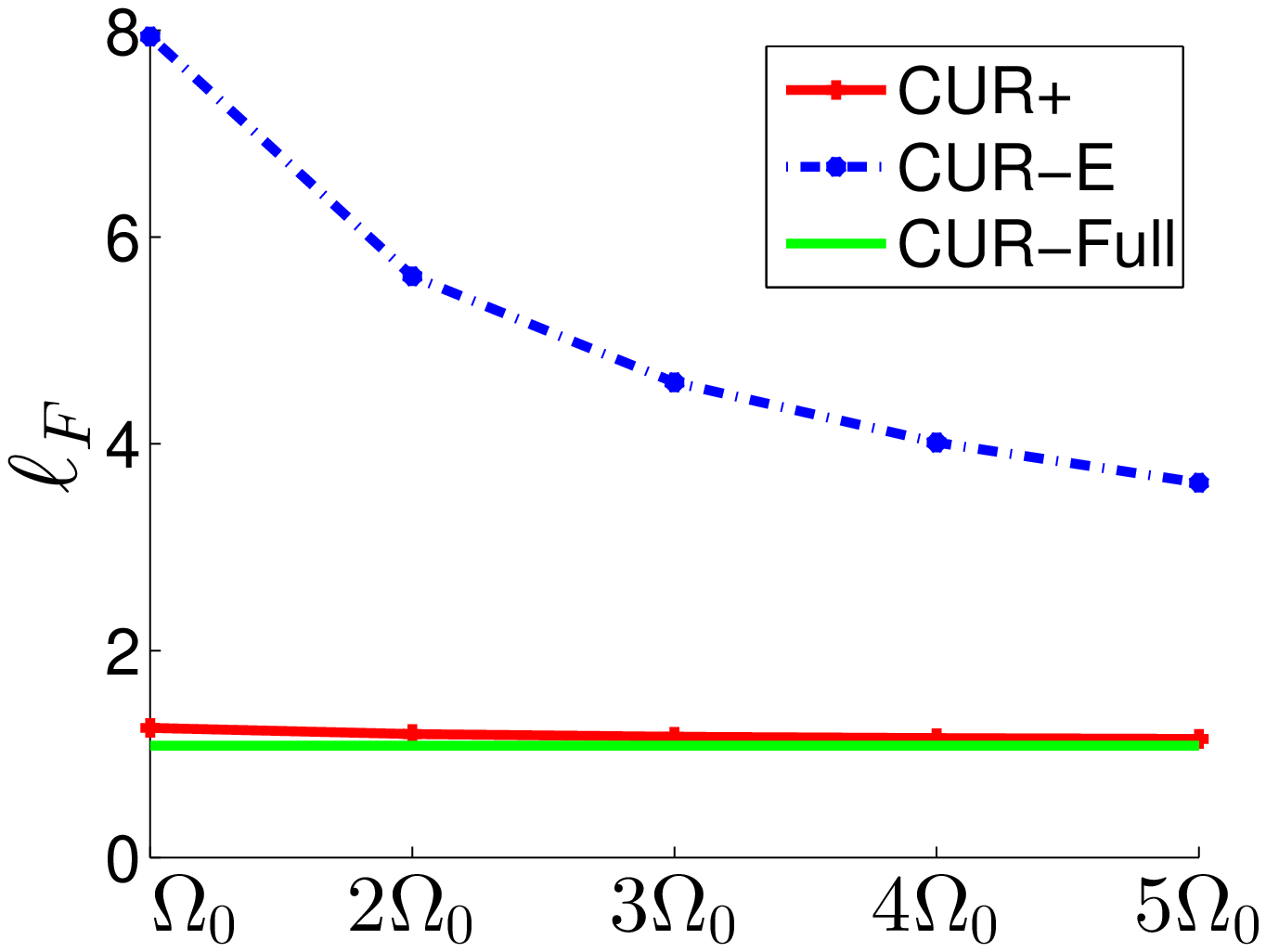}\\
\mbox{Gisette $r=20$}
\end{minipage}

\begin{minipage}[h]{1.3in}
\centering
\includegraphics[width= 1.3in]{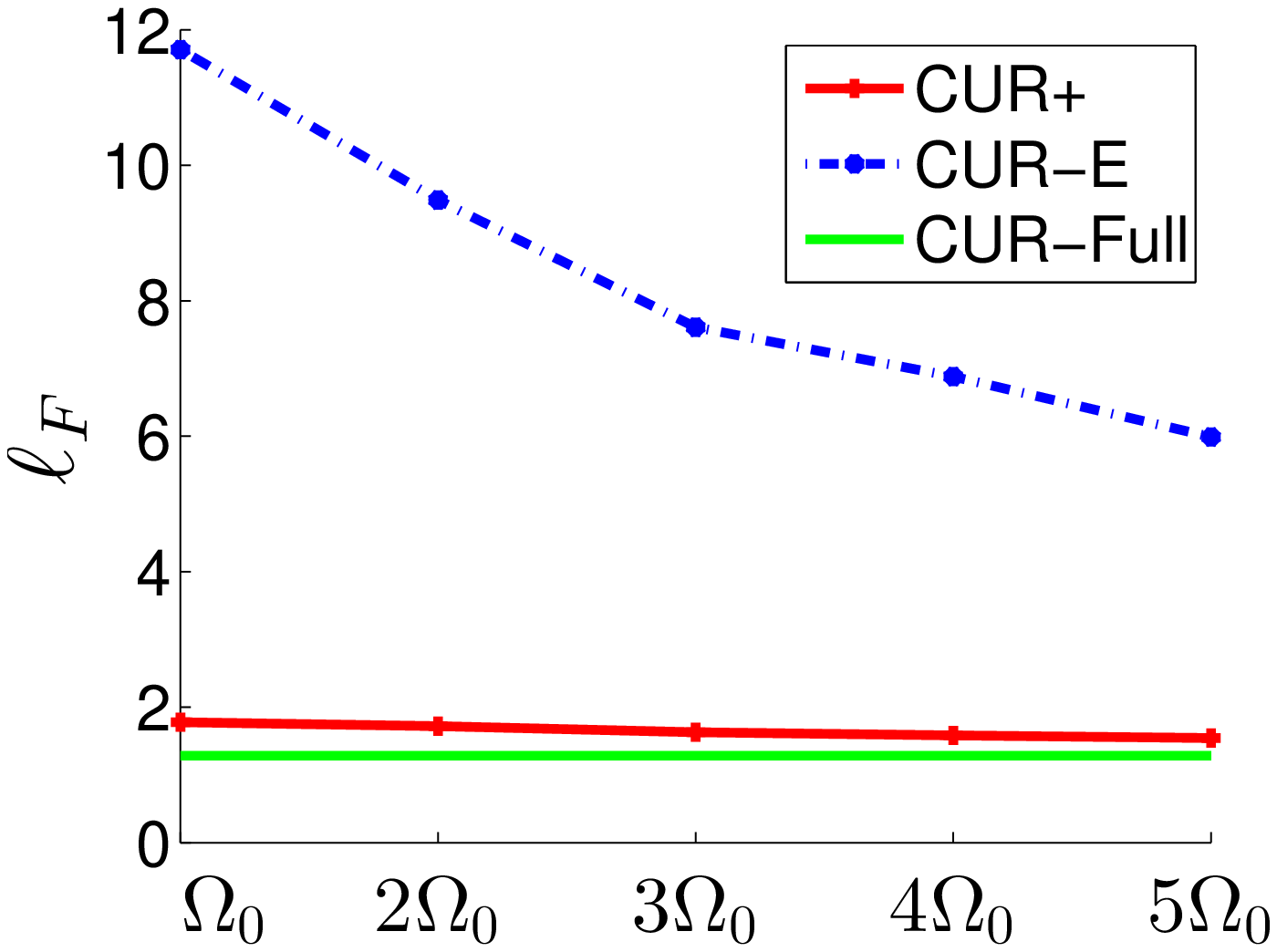}\\
\mbox{Enron $r=50$}
\end{minipage}
\begin{minipage}[h]{1.3in}
\centering
\includegraphics[width= 1.3in]{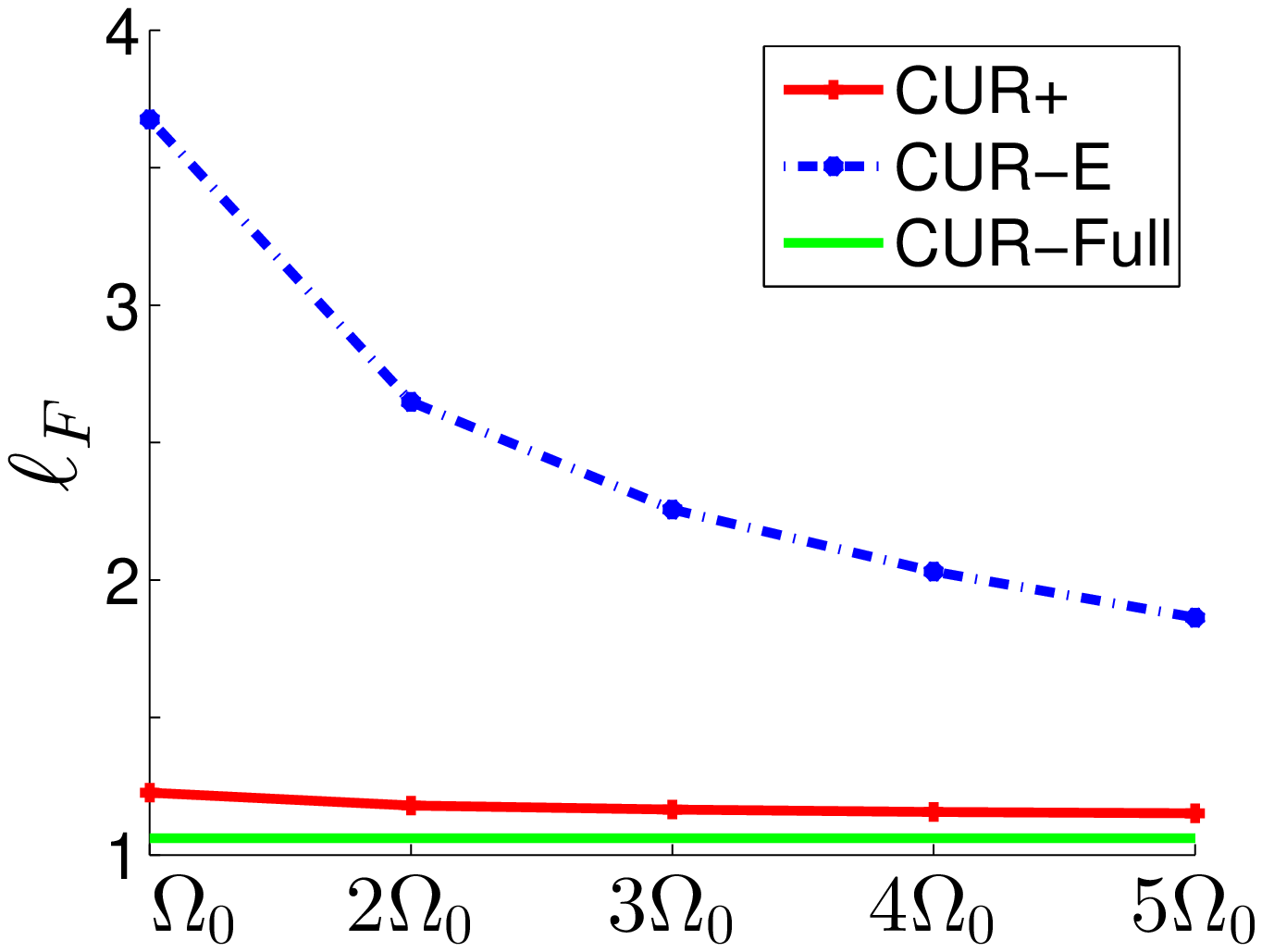}\\
\mbox{Dexter $r=50$}
\end{minipage}
\begin{minipage}[h]{1.3in}
\centering
\includegraphics[width= 1.3in]{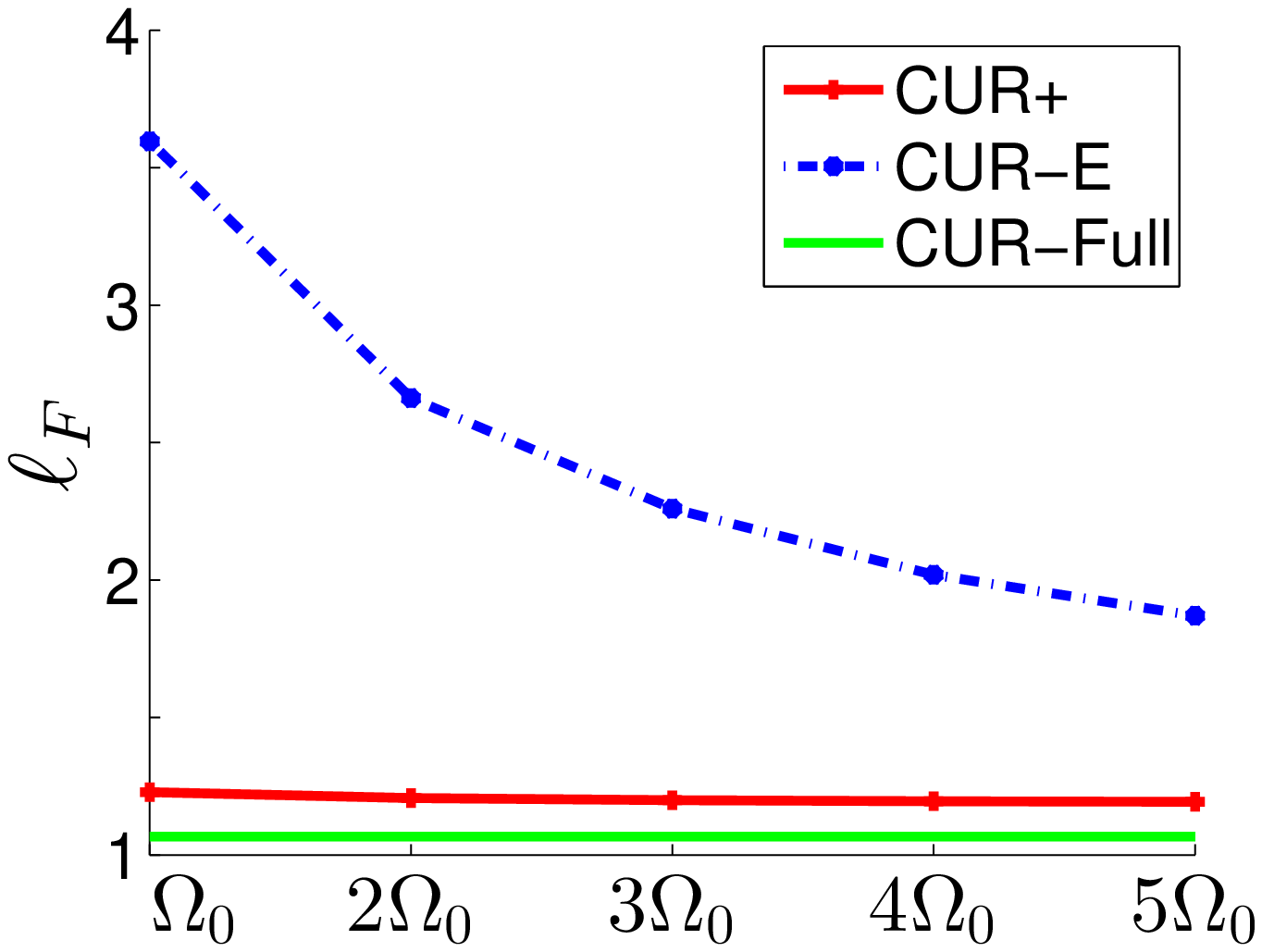}\\
\mbox{Farm Ads $r=50$}
\end{minipage}
\begin{minipage}[h]{1.3in}
\centering
\includegraphics[width= 1.3in]{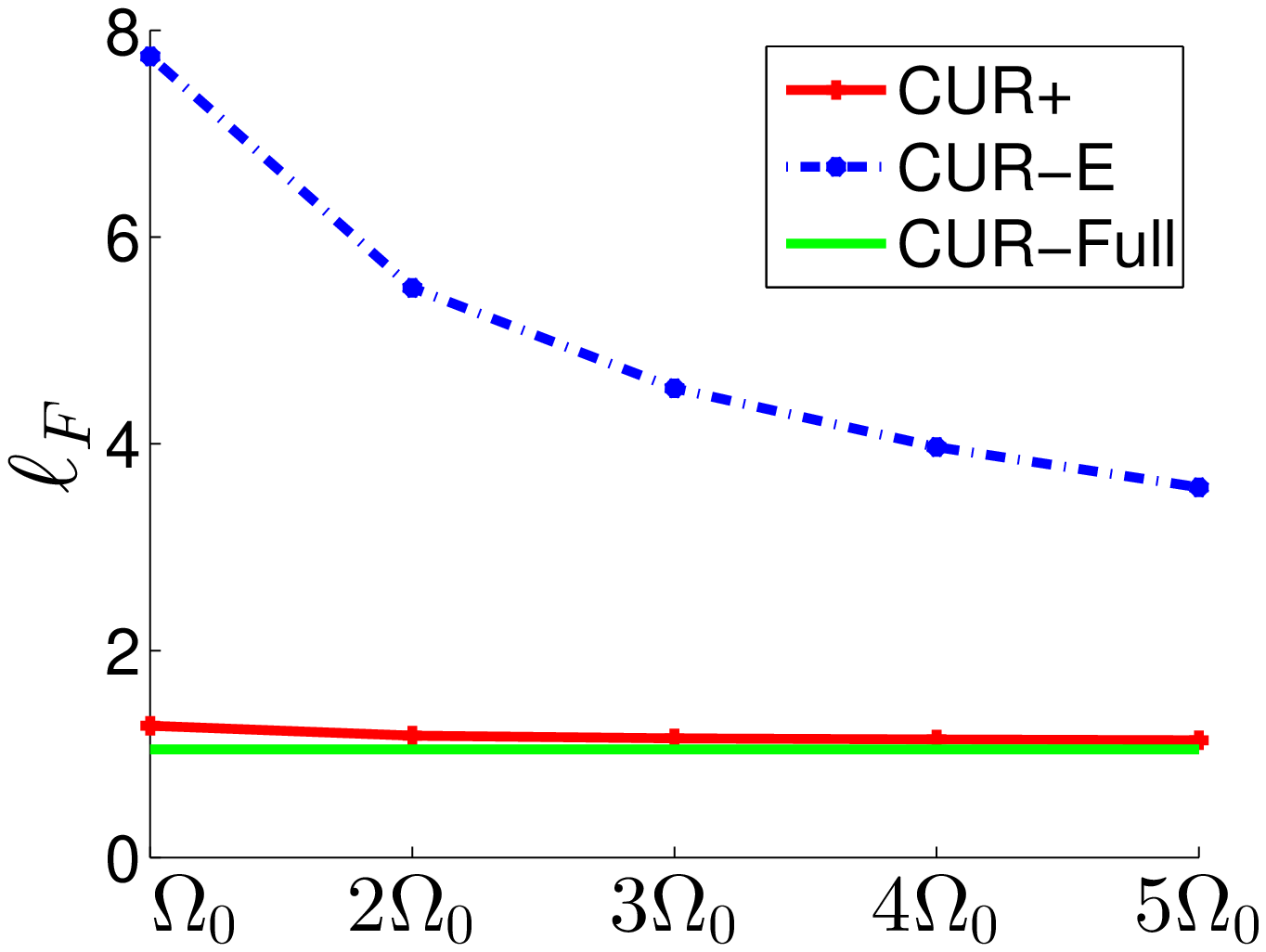}\\
\mbox{Gisette $r=50$}
\end{minipage}
\caption{Comparison of CUR algorithms measured by Forbenius norm with the number of sampled columns and rows fixed as $d_1 = 5 r$ and $d_2 = 5 d_1$, respectively, where $r = 10, 20$ and $50$. The number of observed entries $|\Omega|$ is varied from $\Omega_0$ to $5\Omega_0$.}\label{fig:fvaro}
\end{figure*}

\section{Conclusion}\label{sec:conclusion}

In this paper, we propose a CUR-style low rank approximation algorithm for partially observed matrix. Our analysis shows that the proposed algorithm only needs $O(nr\ln r)$ number of observed entries to perfectly recover a low-rank matrix, improving the results of the existing algorithms for matrix completion (of course under a slightly stronger condition). We also show the the spectral error bound for the proposed algorithm when the target matrix is of full rank. Empirical studies on both synthetic data and real datasets verify our theoretical claims and furthermore, demonstrate that the proposed algorithm is more effective in handling partially observed matrix than the existing CUR algorithms. Since adaptive sampling has shown promising results for low rank matrix approximation~\cite{krishnamurthy-20130-low}, in the future, we plan to combine the proposed algorithm with adaptive sampling strategy to further reduce the error bound. We also plan to exploit the recent studies on matrix approximation/completion with non-uniform sampling and extend the CUR algorithm to the case when observed entries are non-uniform sampled.

\appendix

\section{Appendix}

We will first give the supporting theorems we will use in the analysis. Then we will give the detailed proof of the three theorems in the paper.
\subsection{Supporting Theorems}
The following results are used throughout the analysis.
\begin{thm} \label{thm:1}
(Theorem 9.1 in~\citep{DBLP:journals/siamrev/HalkoMT11}) Let $M$ be an $n\times m$ matrix with singular value decomposition $M = U\Sigma V^{\top}$. There is a fixed $r > 0$. Choose a test matrix $\Psi \in \R^{m\times d}$ and construct sample matrix $Y = M \Psi$. Partition $M$ as in (\ref{eqn:partition})
\begin{eqnarray}
    M = U\Sigma V^{\top} = \begin{array}{cc}r & m - r \\ \mbox{$[$}U_1 & U_2\mbox{$]$} \end{array}\left[ \begin{array}{cc} \Sigma_1 & \\ & \Sigma_2\end{array}\right]\left[\begin{array}{c}V_1^{\top} \\ V_2^{\top} \end{array} \right] \label{eqn:partition}
\end{eqnarray}
and define $\Psi_1 = V_1^{\top} \Psi$ and $\Psi_2 = V_2^{\top}\Psi$. Assuming $\Psi_1$ has full row rank, the approximation error satisfies
\[
    \|M - P_Y(M)\|^2_2 \leq \|\Sigma_2\|^2_2 + \|\Sigma_2\Psi_2\Psi_1^{\dagger}\|_2^2
\]
where $P_Y(M)$ projects column vectors in $M$ in the subspace spanned by the column vectors in $Y$ and $^\dagger$ denotes the pseudoinverse.
\end{thm}

\begin{thm} \label{lemm:1}
(Derived From Theorem 2.2 of~\citep{tropp-2011-improved})
Let $\X$ be a finite set of PSD matrices with dimension $k$ (means the size of the square matrix is $k\times k$). $\lambda_{\max}(\cdot)$ and $\lambda_{\min}(\cdot)$ calculate the maximum and minimum eigen value respectively.

Suppose that
\[
    \max_{X \in \X} \lambda_{\max}(X) \leq B.
\]
Sample $\{X_1, \ldots, X_{\ell}\}$ uniformly at random from $\X$ without replacement. Compute
\[
    \mu_{\max} = \ell \lambda_{\max}(\E[X_1]), \quad     \mu_{\min} = \ell \lambda_{\min}(\E[X_1])
\]
Then
\begin{eqnarray*}
& & \Pr\left\{\lambda_{\max}\left(\sum_{i=1}^{\ell} X_i\right) \geq (1 + \rho ) \mu_{\max} \right\} \leq k \exp\frac{-\mu_{\max}}{B}\left[(1 + \rho )\ln (1 + \rho )-\rho \right] \text{ for }\rho \in[0,1)\\
& & \Pr\left\{\lambda_{\min}\left(\sum_{i=1}^{\ell} X_i\right) \leq (1 - \rho ) \mu_{\min} \right\} \leq k \exp \frac{-\mu_{\min}}{B}\left[(1 - \rho )\ln(1 - \rho )+\rho \right] \text{ for }\rho \ge 0
\end{eqnarray*}
\end{thm}

\begin{thm} \label{thm:perturbation}
Let $A = S^{\top}H S$ and $\tilde{A} = S^{\top}\tilde{H}S$ be two symmetric matrices of size $n\times n$. Let $\lambda_i, i \in [n]$ and $\tilde{\lambda}_i, i \in [n]$ be the eigenvalues of $A$ and $\tilde{A}$, respectively, ranked in descending order. Let $U_A, \tilde{U}_A \in \R^{n\times r}$ include the first $r$ eigenvectors of $A$ and $\tilde{A}$, respectively. Let $\|\cdot\|$ be any invariant norm. Define
\begin{eqnarray*}
\Delta_{\lambda} & = & \min\left(\sqrt{2}\left(1 - \frac{\lambda_{r+1}}{\lambda_r}\right), \frac{1}{\sqrt{2}}\right) \le \frac{1}{\sqrt{2}}\\
\Delta_H         & = & \frac{\|H^{-1}\| \|H - \tilde{H}\|}{\sqrt{1 - \|H^{-1}\|\|H - \tilde{H}\|}}
\end{eqnarray*}
If $\Delta_{\lambda} \geq \Delta_H/2$, we have
\[
\|\sin\Theta(U_A, \tilde{U}_A)\| \leq \frac{\Delta_H}{\Delta_{\lambda} - \Delta_H/2}\left(1 + \frac{\Delta_H\Delta_{\lambda}}{16} \right)
\]
where
\begin{eqnarray*}
\Theta(X,\tilde X)=\arccos ((X^*X)^{-1/2}X^*\tilde X(\tilde X^*\tilde X)^{-1}\tilde X^*X(X^*X)^{-1/2})^{1/2}
\end{eqnarray*}
defines the angle matrix between $X$ and $\tilde X$.
\end{thm}

Note that the above Theorem~\ref{thm:perturbation} follows directly from Theorem 4.4 and discussion in Section 5 from~\citep{li-1999-relative}.

\subsection{Proof of Theorem 2}
 We will first provide the key result for our analysis, and then bound each component of the key result, that is, first, we will show that $\|M - P_{\Uh} M P_{\Vh}\|_2^2$ is small; then, we will bound the strong convexity of the objective function.

The following theorem shows that the difference between $M$ and $\Mh$ is well bounded if both $\|M - P_{\Uh} M P_{\Vh}\|_2^2$ and the strong convexity of Eq.2 are well bounded,
\begin{thm} \label{thm:combine}
Assume (i) $\|M - P_{\Uh} M P_{\Vh}\|_2^2 \leq \Delta$, and (ii) the strong convexity of the objective function is no less than $|\Omega|\gamma$. Then
\[
\|M - \Mh\|^2_2 \leq 2\left(\Delta + \frac{\Delta}{\gamma }\right).
\]
\end{thm}
where strongly convexity is defined as,
\begin{definition}
A function $f: \mathcal{D}\rightarrow \mathbb{R}$ is $\xi$-strongly convex w.r.t.  norm $\|\cdot\|$ if $f$ is everywhere differentiable and
\begin{eqnarray*}
f(\w)\ge f(\w')+\nabla f(\w')(w-w')+\frac{\xi}{2}\|w-w'\|^2.
\end{eqnarray*}
Then $\xi$ is the strongly convexity of $f$.
\end{definition}

\begin{proof}
Set $Z = \Uh^{\top}M \Vh$. Since $\|M - P_{\Uh}M P_{\Vh}\|_2^2 \leq \Delta$, we have
\[
\|M - \Uh Z \Vh^{\top}\|_2^2 \leq \Delta,
\]
implying
\[
\|\Rt_{\Omega}(M) - \Rt_{\Omega}(\Uh Z \Vh^{\top})\|_F^2 \leq \Delta
\]
Let $Z_*$ be the optimal solution to Eq.2. Using the strongly convexity of Eq.2, we have
\[
\frac{1}{2}\gamma|\Omega|\|Z - Z_*\|_F^2 \leq \frac{1}{2}|\Omega|\Delta,
\]
i.e. $\|Z - Z_*\|_F^2 \leq \Delta/(\gamma)$.

This is because $f(Z)=\frac{1}{2}\|\Rt_{\Omega}(M) - \Rt_{\Omega}(\Uh Z \Vh^{\top})\|_F^2 $,
such that $\nabla f(Z)=\Uh^T[\Rt_\Omega(\Uh Z\Vh^T)-\Rt_\Omega(M)]\Vh$, and $\nabla f(Z_*)=0$
\begin{eqnarray*}
\frac{|\Omega|\gamma}{2}\|Z-Z_*\|^2_F&\leq& \frac{1}{2}\|\Rt_{\Omega}(M) - \Rt_{\Omega}(\Uh Z \Vh^{\top})\|_F^2 -\frac{1}{2}\|\Rt_{\Omega}(M) - \Rt_{\Omega}(\Uh Z_* \Vh^{\top})\|_F^2 \\
&\leq& \frac{1}{2}\|\Rt_{\Omega}(M) - \Rt_{\Omega}(\Uh Z \Vh^{\top})\|_F^2 \leq \frac{|\Omega|\Delta}{2}
\end{eqnarray*}

We thus have,
\begin{eqnarray*}
\|M - \Mh\|_2^2 & \leq & 2\|M - P_{\Uh}M P_{\Vh}\|_2^2 + 2\|P_{\Uh} M P_{\Vh} - \Uh Z_* \Vh^{\top}\|_2^2 \\
 & \leq & 2\|M - P_{\Uh}M P_{\Vh}\|_2^2 + 2\|P_{\Uh} M P_{\Vh} - \Uh Z_* \Vh^{\top}\|_F^2 \\
& \leq & 2\|M - P_{\Uh}M P_{\Vh}\|_2^2 + 2\|Z - Z_* \|_F^2 \leq 2\left(\Delta + \frac{\Delta}{\gamma|}\right)
\end{eqnarray*}
\end{proof}

In order to bound $\Delta$, we need the following theorem,
\begin{thm} \label{thm:2}
With a probability $1 - 2e^{-t}$, we have,
\[
\|M - MP_{\Vh}\|^2_2 \leq \sigma^2_{r+1}\left(1 + 2\frac{m}{d}\right)
\]
and
\[
\|M - P_{\Uh} M\|_2 \leq \sigma^2_{r+1}\left(1 + 2\frac{n}{d}\right)
\]
provided that $d \geq {7\mu(r) r (t+\ln r)}$.
\end{thm}
\begin{proof}
Let $i_1, \ldots, i_d$ are the $d$ selected columns. Define $\Psi = (\e_{i_1}, \ldots, \e_{i_d}) \in R^{m\times d}$, where $\e_i$ is the $i$th canonical basis. Such that we have $A=M\times\Psi$, that is, $A$ is composed of the $d$ selected columns of $M$.
To utilize Theorem~\ref{thm:1}, we need to bound the minimum eigenvalue of $\Psi_1\Psi_1^{\top}$, where $\Psi_1=V_1^T\Psi\in R^{r \times d}$ is full rank. We have
\[
\Psi_1\Psi_1^{\top} = V_1^{\top}\Psi\Psi^{\top} V_1
\]
Let $\vt_i^\top, i \in [d]$ be the $i$th row vector of $V_1$. We have,
\[
\Psi_1\Psi_1^{\top} = \sum_{j=1}^d \vt_{i_j}\vt_{i_j}^{\top}
\]
It is straightforward to show that
\[
\E\left[\Psi_1\Psi_1^{\top} \right] = \frac{d}{m}I_r
\]
and
\[
\E\left[ \vt_{i_j}\vt_{i_j}^{\top}\right] = \frac{1}{m}I_r.
\]
To bound the minimum eigenvalue of $\Psi_1\Psi_1^{\top}$, we need Theorem~\ref{lemm:1}, where we first need to bound the maximum eigen value of $\vt_{i_j}\vt_{i_j}^{\top}$, which is a rank-$1$ matrix, whose eigen value
\[
\max\limits_{1 \leq i \leq m}\lambda_{\max} (\vt_{i_j}\vt_{i_j}^{\top})= \max\limits_{1 \leq i \leq m} |\vt_i|^2 \leq \mu(r) \frac{r}{m},
\]
and
\[
\lambda_{\max}(\E\left[ \vt_{i_j}\vt_{i_j}^{\top}\right])=\lambda_{\min}(\E\left[ \vt_{i_j}\vt_{i_j}^{\top}\right])=\frac{1}{m}
\]
Thus, we have,
\begin{eqnarray*}
\Pr\left\{\lambda_{\min}(\Psi_1\Psi_1^{\top}) \leq (1 - \delta)\frac{d}{m}\right\}&\leq& r \exp\frac{-d/m}{r\mu(r)/m}\left[(1 - \rho )\ln (1 - \rho )+\rho \right]\\
&=&r \exp\frac{-d}{r\mu(r)}\left[(1 - \rho )\ln (1 - \rho )+\rho \right]
\end{eqnarray*}
By setting $\delta = 1/2$, we have,
\begin{eqnarray*}
\Pr\left\{\lambda_{\min}(\Psi_1\Psi_1^{\top}) \leq \frac{d}{2m}\right\}&\leq& r \exp\frac{-d}{7r\mu(r)}
= r e^{-d/[7\mu(r) r]}
\end{eqnarray*}
where with $d \geq {7\mu(r) r (t+\ln r)}$, we have $r\exp^{-d/[7\mu(r) r]}\le e^{-t}$, that is,
\begin{eqnarray*}
\Pr\left\{\lambda_{\min}(\Psi_1\Psi_1^{\top}) \geq \frac{d}{2m}\right\}&\geq& 1-e^{-t}
\end{eqnarray*}

With
\[
\lambda_{\min}(\Psi_1\Psi_1^{\top}) \geq \frac{d}{2m}
\]
according to Theorem~\ref{thm:1}, we have
\begin{eqnarray*}
\|M - MP_{\Vh}\|_2^2 &\leq& \|\Sigma_2\|^2_2 + \|\Sigma_2\Psi_2\Psi_1^{\dagger}\|_2^2\\
&\leq& \sigma_{r+1}^2 + \left\|\Sigma_2\Psi_2\Psi_1^{\dagger}\right\|_2^2 \\
&\leq&\sigma_{r+1}^2 + \|\Psi_1^{\dagger}\|_2^2\|\Sigma_2\Psi_2\|_2^2 \\
&\leq&\sigma_{r+1}^2 + \frac{2m}{d}\|\Sigma_2\Psi_2\|_2^2 \\
&\leq&\sigma_{r+1}^2 + \frac{2m}{d}\|\Sigma_2\|_2^2\|\Psi_2\|_2^2 \\
&\leq&\sigma_{r+1}^2 + \frac{2m}{d}\sigma_{r+1}^2 \\
&\leq& \sigma_{r+1}^2\left(1 + \frac{2 m}{d}\right)
\end{eqnarray*}

\begin{itemize}
\item The $1$st inequality is according to Theorem~\ref{thm:1}.
\item The $3$rd inequality is because the two facts, $\|M_1 M_2\|_2\le \|M_1\|_2 \times \|M_2\|_2$
\item The $4$th inequality is becuase
$\|\Psi_1^{\dagger}\|_2=1/\sigma_{\min}(\Psi_1)=\sqrt{1/\lambda_{\min}(\Psi_1\Psi_1^{\top})}\leq \sqrt{2m/d}$
\item The $6$th inequality is because $\|\Sigma_2\|_2=\sigma_{r+1}$ and $\|\Psi_2\|_2\le \|V_2\|_2\|\Psi\|_2=1$
\end{itemize}
\end{proof}

We then bound $\Delta$,
\begin{thm} \label{thm:Delta}
With a probability $1 - 2e^{-t}$, we have,
\[
\Delta := \|M - P_{\Uh} MP_{\Vh}\|^2_2 \leq 4\sigma^2_{r+1}\left(1 + \frac{m + n}{d}\right)
\]
if $d \geq {7\mu(r) r (t+\ln r)}$.
\end{thm}
\begin{proof}
Using Theorem~\ref{thm:2}, we have, with a probability $1 - 2e^{-t}$
\begin{eqnarray*}
\|M - P_{\Uh} M P_{\Vh}\|^2_2 &\leq& 2\|M - M P_{\Vh}\|^2_2 + 2\|(M - P_{\Uh}M)P_{\Vh}\|^2_2 \\
&\leq& 2\|M - M P_{\Vh}\|^2_2 + 2\|M - P_{\Uh}M\|^2_2 \\
&\leq& 4\sigma_{r+1}^2\left(1 + \frac{n + m}{d}\right)
\end{eqnarray*}
\end{proof}

We will then bound the strong convexity of the objective function,
\begin{thm} \label{thm:gamma}
With a probability $1 - e^{-t}$, we have that $\gamma|\Omega|$, the strongly convexity for the objective function in (2), is bounded from below by $|\Omega|/[2mn]$ (that is, $\gamma\ge 1/(2mn)$), provided that
\[
|\Omega| \geq 7 \muh^2(r) r^2(t+2\ln r)
\]
\end{thm}
\begin{proof}
To bound the strong convexity, we could instead bound the smallest eigen value of the Hessian matrix. The Hessian matrix is an $r^2\times r^2$ matrix. Assuming the second-order derivative of the $(i_1,j_1)$th and $(i_2, j_2)$th entry of $Z$ is the $(r(i_1-1)+j_1,r(i_2-1)+j_2)$th entry of the Hessian matrix, the Hessian matrix could be written as,
\begin{eqnarray*}
H=\sum_{(i,j)\in\Omega}[\text{vec}(\ut_i^\top\vt_j)][\text{vec}(\ut_i^\top\vt_j)]^T
\end{eqnarray*}

To bound the minimum eigenvalue of $H$, we will use Lemma~\ref{lemm:1}. Thus first we need to bound
\begin{eqnarray*}
\max\limits_{i,j} \lambda_{\max}([\text{vec}(\ut_i^\top \vt_j)][\text{vec}(\ut_i^\top \vt_j)]^T)=\max\limits_{i,j}|\text{vec}(\ut_i^\top \vt_j)|^2 &\leq& \max \|\ut_i^{\top}\vt_j\|_F^2\le \frac{\muh^2(r) r^2}{mn}
\end{eqnarray*}
and
\begin{eqnarray*}
  \lambda_{\min}\left(\E([\text{vec}(\ut_i^\top\vt_j)][\text{vec}(\ut_i^\top\vt_j)]^T)\right)
  &=& \frac{1}{mn}\lambda_{\min}\left((U\otimes V)^T\times(U\otimes V)\right)\\
  &=& \frac{1}{mn}
\end{eqnarray*}
where $\otimes$ is the Kronecker product.

Based on Theorem~\ref{lemm:1}, we have
\begin{eqnarray*}
\Pr\left\{\lambda_{\min}(H) \leq  \frac{|\Omega|}{2mn} \right\} &\leq& r^2
e ^\frac{-|\Omega|}{7\muh^2(r) r^2}
\end{eqnarray*}

Hence, with a probability $1 - e^{-t}$, we have
\[
\lambda_{\min}(H) \geq \frac{|\Omega|}{2mn}
\]
provided that
\[
|\Omega| \geq 7 \muh^2(r) r^2(t+2\ln r)
\]
\end{proof}

Theorem 2 can be easily proved combining Theorems~\ref{thm:combine},~\ref{thm:Delta} and~\ref{thm:gamma}.

\subsection{Proof of Theorem 1}
The following theorem allows us to replace $\muh(r)$ in Theorem~\ref{thm:gamma} with $\mu(r)$ when the rank of $M$ is less than or equal to $r$.
\begin{thm} \label{thm:mu-1}
With a probability $1 - 2e^{-t}$, we have $\muh(r) = \mu(r)$, if $d \geq {7\mu(r) r (t+\ln r)}$.
\end{thm}
\begin{proof}
According to Theorem~\ref{thm:Delta}, with a probability $1 - 2e^{-t}$, we have $M = P_{\Uh} M P_{\Vh}$, provided that $d \geq {7\mu(r) r (t+\ln r)}$. Hence $P_{U_1} = P_{\Uh}$ and $P_{V_1} =  P_{\Vh}$, which directly implies that $\mu(r) = \muh(r)$.
\end{proof}

Theorem 1 can be proved directly from Theorem 2 and Theorem~\ref{thm:mu-1}.

\subsection{Proof of Theorem 3}
Define
\[
H_A = \eta I + \frac{1}{mn}MM^{\top}, \quad \Hh_A = \eta I + \frac{1}{dn}AA^{\top}
\]
and
\[
H_B = \eta I + \frac{1}{mn}M^{\top}M, \quad \Hh_B = \eta I + \frac{1}{dm}BB^{\top}
\]

We can have the first $r$ eigen vector of would be $H_A$, because
\begin{eqnarray*}
H_A &=& \eta I + \frac{1}{mn}MM^{\top}\\
&=&\eta U U^T+\frac{1}{mn} U(\Sigma\Sigma^T)U^T\\
&=&U(\eta I+\frac{1}{mn}\Sigma\Sigma^T)U^T
\end{eqnarray*}
and
\begin{eqnarray*}
H_A^{-1/2}=U diag(\sqrt{\frac{mn}{\sigma_1^2+mn\eta}},\ldots,\sqrt{\frac{mn}{\sigma_m^2+mn\eta}})=\sqrt{mn}UTU^T
\end{eqnarray*}
where
\begin{eqnarray*}
T=diag(\sqrt{\frac{1}{\sigma_1^2+mn\eta}},\ldots,\sqrt{\frac{1}{\sigma_m^2+mn\eta}})
\end{eqnarray*}

\subsubsection{Proof of Lemma 1}
\begin{proof}
Just consider the maximization of the norm of rows of $U$, then we will have
\begin{eqnarray*}
\mu(\eta)&=&\max_{i=1,\ldots,n} \sum_{j=1}^m \frac{n}{r(M,\eta)}\frac{\sigma^2_j}{\sigma^2_j+mn\eta}U^2_{i,j}\\
&=&
\max_{i=1,\ldots,n} \frac{n}{r}\sum_{j=1}^m r\frac{\sigma^2_j}{r(M,\eta)(\sigma^2_j+mn\eta)}U^2_{i,j}\\
&\ge&\max_{i=1,\ldots,n} \frac{n}{r}\sum_{j=1}^m r\frac{a}{r}U^2_{i,j}\\
&=&a\max_{i=1,\ldots,n} \frac{n}{r}\sum_{j=1}^m U^2_{i,j}\\
&=& a\mu(r)
\end{eqnarray*}

when $\eta = \sigma_r^2/mn$, then $a\le r/2r(M,\eta)$, then
\begin{eqnarray*}
\mu(r)\le\frac{1}{a}\mu(\delta)\le \frac{2r(M,\eta)}{r}\mu(\eta)
\end{eqnarray*}
completes our proof.

\end{proof}

\subsubsection{Proof of Lemma 2}
To this end, we need the following theorem.
\begin{thm} \label{thm:3}
With a probability $1 - 4e^{-t}$, we have
\[
1 - \delta \leq \lambda_{k}(H_A^{-1/2}\Hh_A H_A^{-1/2}) \leq 1 + \delta, \quad 1 - \delta \leq \lambda_{k}(H_B^{-1/2}\Hh_B H_B^{-1/2}) \leq 1 + \delta, \; \forall k \in [n]
\]
if
\[
d \geq \frac{4}{\delta^2}(\mu(\eta) r(M, \eta) + 1)(t + \ln n)
\]
\end{thm}

\begin{proof}
It is sufficient to show the result for $\Hh_A$.

Define
\[
\X = \left\{M_i = (H_A^{-1/2})^T\left(\frac{1}{n}M_{*,i}M_{*,i}^{\top} + \eta I\right)H_A^{-1/2}, i=1, \ldots, m \right\}
\]
Note that if $\a_i$ is the $j$th column of matrix $M$, then,
\begin{eqnarray*}
M_{*,i}=U\Sigma (V_{i,*})^\top
\end{eqnarray*}

Thus we have
\begin{eqnarray*}
M_i & = & mnUTU^{\top}(\frac{1}{n} U\Sigma V_{i, *}^{\top}V_{i,*}\Sigma U^{\top} + \eta I)UTU^{\top} \\
& = & U\left(m T\Sigma V_{i, *}^{\top}V_{i,*}\Sigma T + mn\eta T^2\right)U^{\top}
\end{eqnarray*}

In this way
\begin{eqnarray*}
\lambda_{\max}(M_i) &\leq& \lambda_{\max}(mUT\Sigma V_{i, *}^{\top}V_{i,*}\Sigma TU^{\top}) +\lambda_{\max}(mn\eta UT^2U^{\top})\\
&=&m|UT\Sigma V^\top_{i, *}|^2_2+\frac{mn\eta}{\sigma_m^2+mn\eta}\\
&\le & \mu(\eta)r(M,\eta)+1
\end{eqnarray*}
(this is because $|Ax|_2^2\le \|A\|_2^2|x|_2^2\le \|A\|_F^2|x|_2^2$)
and
\begin{eqnarray*}
\lambda_{\max}(\E[M_i])&=&\lambda_{\max}(U\left( T\Sigma V^{\top}V\Sigma T + mn\eta T^2\right)U^{\top})\\
&=&\lambda_{\max}(U\left( T\Sigma \Sigma T + mn\eta T^2\right)U^{\top})\\
&=&\frac{\sigma_1^2}{mn\eta+\sigma_1^2}+\frac{mn\eta}{mn\eta+\sigma_1^2}\\
&=&1
\end{eqnarray*}

So
\[
\mu_{\max} = d \lambda_1(\E[M_i]) = d
\]
we have (using Lemma~\ref{lemm:1}),
\[
\Pr\left\{\lambda_{\max}\left(H_A^{-1/2}\Hh_A H_A^{-1/2} \right) \geq 1 + \delta \right\} \leq n\exp\left( -\frac{d}{\mu(\eta) r(M,\eta) + 1}\left[(1 + \delta)\ln(1 + \delta) - \delta\right]\right)
\]
Using the fact that (at $0$ they are the same, but the left increase faster than the right)
\[
(1 + \delta) \ln(1 + \delta) \geq \delta + \frac{1}{4}\delta^2, \forall \delta \in [0, 1],
\]
we have
\[
\Pr\left\{\lambda_{\max}\left(H_A^{-1/2}\Hh_A H_A^{-1/2} \right) \geq 1 + \delta \right\} \leq n\exp\left( -\frac{d \delta^2}{4(\mu r(M,\eta) + 1)}\right)
\]
We have the result by setting $d \ge 4(\mu(\eta) r(M, \eta) + 1) (\ln n + t)/\delta^2$. Similarly, for the lower bound, we have (using Lemma~\ref{lemm:1})
\[
\Pr\left\{\lambda_{\min}\left(H_A^{-1/2}\Hh_A H_A^{-1/2}\right) \leq 1 - \delta \right\} \leq n\exp\left( - \frac{d}{\mu(\eta) r(M,\eta) + 1}\left[(1 - \delta)\ln(1 - \delta) + \delta\right]\right)
\]
Using the fact that (by Taylor Expansion of $\ln(1-\delta)$)
\[
(1 - \delta)\ln(1 - \delta) \geq - \delta + \frac{\delta^2}{2}
\]
We have the result by setting $d\ge 2(\mu(\eta) r(M, \eta) + 1) (\ln n + t)/\delta^2$.
\end{proof}

Using Theorem~\ref{thm:3}, we will prove Lemma 2,
\begin{proof}
To utilize Theorem~\ref{thm:perturbation}, we rewrite $H_A$ and $\Hh_A$, as
\[
H_A = H_A^{1/2} I H_A, \quad \Hh_A = H_A^{1/2}DH_A^{1/2}
\]
where $D = H_A^{-1/2} \Hh_A H_A^{-1/2}$. According to Theorem~\ref{thm:3}, with a probability $1 - 2e^{-t}$, we have \textcolor{black}{$\|D - I\|_2 \leq \delta$}, provided that
\[
d = \frac{4}{\delta^2}(\mu(\eta) r(M, \eta) + 1)(t + \ln n)
\]
We then compute $\Delta_{H}$ defined in Theorem~\ref{thm:perturbation} as
\begin{eqnarray*}
\Delta_H \leq \frac{\delta}{\sqrt{1 - \delta}}
\end{eqnarray*}

Because $d \geq 16(\mu(\eta)r(M, \eta) + 1) (t + \ln n)$, we have
\begin{eqnarray*}
\frac{4}{\delta^2}(\mu(\eta) r(M, \eta) + 1)(t + \ln n)\ge 16(\mu(\eta)r(M, \eta) + 1) (t + \ln n)
\end{eqnarray*}
that is $\delta\leq 1/2$.

Because $\sigma_r \geq \sqrt{2}\sigma_{r+1}$, we have $1/2 \leq 1 - \sigma^2_{r+1}/\sigma^2_r$. Since \textcolor{black}{$\delta \leq 1/2 \leq 1 - \sigma^2_{r+1}/\sigma^2_r$}, we have $\Delta_H \leq \sqrt{2}\delta$.
%

Then according to Theorem~\ref{thm:perturbation}, we have,
\begin{eqnarray*}
\|\sin\Theta(U_1, \Uh)\|_2 &\leq& \frac{\sqrt{2}\delta}{\Delta_\lambda-\sqrt{2}\delta/2}(1+\frac{\sqrt{2}\delta\Delta_\lambda}{16})\\
&\leq& \frac{\sqrt{2}\delta}{\Delta_\lambda-\sqrt{2}\delta/2}(1+\frac{1}{32})< 3\sqrt{2}\delta
\end{eqnarray*}
Similarly, we have,
\[
\|\sin\Theta(V_1, \Vh)\|_2 < 3\sqrt{2}\delta
\]
Thus, with a probability $1 - 4e^{-t}$, we have
\[
\muh(r) \leq \frac{2r(M,\eta)}{r}\mu(\eta) + \frac{n}{r}\|\sin\Theta(V_1, \Vh)\|^2_2  \leq \frac{2r(M,\eta)}{r}\mu(\eta) + \frac{18n \delta^2}{r}
\]
\end{proof}

Theorem 3 can be proved by combining the results of Theorems~\ref{thm:combine},~\ref{thm:gamma}, Lemma 1 and Lemma 2.

\bibliography{cur}\bibliographystyle{alpha}

\end{document}